\documentclass{article}

\usepackage{microtype}
\usepackage{graphicx}
\usepackage{subcaption}
\usepackage{float}
\usepackage{booktabs} 

\usepackage{hyperref}

\usepackage{amsmath}
\usepackage{amssymb}
\usepackage{mathtools}
\usepackage{amsthm}

\usepackage[capitalize,noabbrev]{cleveref}

\usepackage[textsize=tiny]{todonotes}

\usepackage{Definitions}

\usepackage{todonotes}

\usepackage{color}

\usepackage[accepted]{icml2022}

\icmltitlerunning{Robustness Implies Generalization via Data-Dependent Generalization Bounds}

\begin{document}
        
        \twocolumn[
        \icmltitle{Robustness Implies Generalization via Data-Dependent Generalization Bounds}

        \begin{icmlauthorlist}
                \icmlauthor{Kenji Kawaguchi}{nus}
                \icmlauthor{Zhun Deng}{harvard}
                \icmlauthor{Kyle Luh}{boulder}
                \icmlauthor{Jiaoyang Huang}{newyork}
        \end{icmlauthorlist}
        
        \icmlaffiliation{nus}{National University of Singapore}
        \icmlaffiliation{harvard}{Harvard University}
        \icmlaffiliation{boulder}{University of Colorado Boulder}
        \icmlaffiliation{newyork}{New York University}
        
        \icmlcorrespondingauthor{Kenji Kawaguchi}{kenji@nus.edu.sg}

        \icmlkeywords{ICML, Deep Learning, Generalization Error, Robustness, Large Margin, Adversarial Attack, Adversarial Example}
        
        \vskip 0.3in
        ]

        \printAffiliationsAndNotice{}

    \begin{abstract}
                This paper proves that robustness implies generalization via data-dependent generalization bounds. As a result, robustness and generalization are shown to be connected closely in a data-dependent manner. Our bounds improve previous bounds in two directions, to solve an open problem that has seen little development since 2010. The first is to reduce the dependence on the covering number. The second is to remove the dependence on the hypothesis space. We present several examples, including ones for lasso and deep learning, in which our bounds are provably preferable. The experiments on real-world data and theoretical models demonstrate near-exponential improvements in various situations. To achieve these improvements, we do not require additional assumptions on the unknown distribution; instead, we only incorporate an observable and computable property of the training samples. A key technical innovation is an improved concentration bound for multinomial random variables that is of independent interest beyond robustness and generalization. 
        \end{abstract}

\section{Introduction}
   
        Robust optimization \citep{ben1998robust, bertsimas2011theory, gabrel2014recent} is an influential paradigm for dealing with noisy or uncertain data. Many optimization problems are sensitive to perturbations in their parameters. Using powerful concepts derived from convexity and duality, robust optimization aims to find a solution to these optimization problems that is feasible with respect to all possible realizations of noisy or unknown parameters. Essentially, this criterion solves the optimization problem for the worst-case choice of the possible parameters. Robust optimization has been successfully applied in a variety of fields, e.g., machine learning, to deal with inaccurately observed training samples and strengthen the robustness of deep learning \citep{bhattacharyya2004second, globerson2006nightmare,deng2021adversarial,rice2021robustness,robey2021adversarial,pedraza2022lyapunov,chen2022adversarial}. 
        
        Inspired by robust optimization, \citet{xu2010robustness,xu2012robustness}  showed that robust algorithms generalize  to unseen data well for various models including deep neural networks. Thus, the notion of robustness provides an alternative view in the topic of generalization  \citep{vapnik1998statistical,bartlett2002rademacher,bousquet2002stability,kawaguchi2017generalization,arora2019fine,kawaguchi2019gradient,deng2021toward,hu2021extended,pham2021combined,zhang2021understanding,zhang2021does}. 

A learning algorithm $\Acal$ is said to be robust if the loss $\ell$ of the hypothesis $\Acal_S$ (returned by the learning algorithm $\Acal$ under the training dataset $S)$ behaves similarly on two samples that are near each other:
        \vspace{-2pt}
        \begin{definition} \label{def:1}
                A learning algorithm $\Acal$ is\textit{ $(K,\epsilon(\cdot))$-robust}, for $K \in \NN$ and $\epsilon(\cdot):\Zcal^n\rightarrow \RR$, if $\Zcal$ can be partitioned into $K$ disjoint sets, denoted by $\{\Ccal_k\}_{k=1}^K$, such that the following holds for all $S \in \Zcal^n$:
                $
                \forall s\in S, \forall z \in \Zcal, \forall k\in [K], \text{ if } s,z \in \Ccal_k, \text{ then } |\ell(\Acal_S, s)-\ell(\Acal_S, z)| \le\epsilon(S). $
        \end{definition}     
        \vspace{-5pt}
  Here, a training dataset $S=(z_{i})_{i=1}^n$ consists of $n$ samples and $\ell:\Hcal \times\Zcal \rightarrow \RR_{\ge 0}$ is the per-sample loss, where $\Hcal$ is a hypothesis space and $z_{i} \in \Zcal $ is the $i$-th training data point. That is, a learning algorithm $\Acal$ is a mapping from $S \in \Zcal^n$ to $\Acal_S \in \Hcal$.
        
      Using Definition \ref{def:1}, \citet{xu2010robustness,xu2012robustness} proved that the generalization error of hypothesis $\Acal_S$ has an upper bound that scales proportionally to $\epsilon(S)+\sqrt{K/n}$. This bound is consequential in the theory of invariant classifiers \citep{sokolic2017generalization}, adversarial examples \citep{cisse2017parseval}, majority voting  \citep{xu2012robustness, devroye2013probabilistic}, support vector machines \citep{xu2012robustness, xu2009robustness, qi2013robust}, lasso \citep{xu2012robustness, hastie2019statistical}, principle component analysis \citep{xu2012robustness, jolliffe2016principal}, deep neural networks \cite{xu2012robustness,sokolic2017robust, cisse2017parseval, sener2017active, gouk2021regularisation, jia2019orthogonal}, metric learning \citep{bellet2015robustness, shi2014sparse}, facial recognition \citep{ding2015multi, tao2016person}, matrix completion \citep{luo2015multiview}, spectral clustering \citep{liu2017spectral}, domain adaption \citep{redko2020survey}, numerical analysis \citep{shen2020deep} and stochastic  algorithms \citep{zahavy2016ensemble}.

        The bound based on algorithmic robustness (Definition \ref{def:1}) has gained considerable interest in the community and has been discussed in much literature as listed above, partially because the dependence on the robustness term $\epsilon(S)$ is natural and intuitive. However, the square root dependence on the partition size (or covering number) $K$ is problematic, because $K$ can be prohibitively large in many applications, especially in high-dimensional data where the covering number can be exponential in the dimension \citep{van1996weak,vershynin2018high}. 
        
        Indeed, the $K$ dependence is one of the chief disadvantages of the robust algorithm framework and \citet{xu2010robustness,xu2012robustness} conjectured that it would be possible  to reduce the dependency on $K$  via adaptive partitioning but remarked that extending their proof to achieve this is complex, leaving this issue as an open problem. 
        
        The proof of the algorithmic robustness bound relies on the concentration results for multinomial random variables, in particular the $\ell_1$ deviations \citep{xu2012robustness,wellner2013weak}. Spurred by applications in learning theory, the concentration of multinomial random variables has been an active area of research by itself beyond algorithmic robustness \citep{weissman2003inequalities, devroye1983equivalence, agrawal2017posterior, qian2020concentration}, where a particular attention has been paid to the dependence of the bound on $K$ --- the number of possible outcomes in the definition of the multinomial random variable. In the robust algorithmic context, $K$ corresponds exactly to $K$ in Definition \ref{def:1}. In a paper previously published in NeurIPS \citep{agrawal2017posterior}, a significant improvement in the $\sqrt{K}$ dependence was claimed which was  later refuted by \citet{qian2020concentration} with the refutation  being  acknowledged by the authors \citep{agrawal2017correction}. Thus, to date, there has been no success in reducing the $\sqrt{K}$ term reported in the literature despite its importance and several previous attempts. 
        
        Importantly, \citet{qian2020concentration} established a lower bound that already scales as $\sqrt{K}$; that is, we have matching upper and lower bounds in terms of $K$. Thus, it may seem that the open question posed in \citep{xu2012robustness} has been settled negatively and any attempts to reduce the $\sqrt{K}$
        dependence is futile. However, similar to other lower bounds in machine learning, the lower bound given in \citep{qian2020concentration} only means that there exists a worst-case distribution for which the (lower) bound cannot be further improved. 
        
        It is plausible that this worst-case distribution is neither representative nor commonplace. Thus, by incorporating information from the training data, it may be possible to extract the properties of the underlying distribution, which may allow one to reduce the $\sqrt{K}$ dependence. In fact, by probing beyond the worst-case analysis, we show that \textit{non-uniform} and \textit{purely data-dependent}  bounds can greatly outperform these previous bounds (that are implicitly derived for the worst-case distributions). Here, a bound is said to be \textit{non-uniform} if the bound differs for different data-distributions. Unlike the standard data-dependence through the outcome of the learning algorithm $A_S$ (e.g., in the robustness term $\epsilon(S)$), a bound is said to be \textit{purely data-dependent} if the bound contains a term that is independent of the algorithm $\Acal$ and differs for different training data $S$. We summarize our main contributions below: 
        
        \vspace{-10pt}
        \begin{enumerate}
        \item In Section \ref{sec:main}, we address the open problem of reducing the $\sqrt{K}$  dependence without making any additional assumptions about the data distribution.  
        The key insight (and challenge) here is to prove an \textit{purely data-dependent} bound where the $\sqrt{K}$ dependence is replaced by an easily computable quantity that is a function of the training samples.         
        This  allows us to reduce the $\sqrt{K}$ dependence without presuming strong  prior knowledge of the probability space and the learning algorithm. 
        
        \vspace{-3pt}
        \item A crucial technical innovation is a series of \textit{non-uniform} and \textit{purely data-dependent} bounds for multinomial random variables that greatly improve the classical bounds under certain conditions.  A representative of our new bounds is stated in Section \ref{sec:main} (and others are presented in Appendices \ref{appendix:fixeda} and \ref{appendix:dependenta}).  These bounds are likely of independent interest in the broader literature beyond  robustness and generalization.  
        
        \vspace{-3pt}
        \item In addition to our main theorems, we  provide abundant numerical simulations and several theoretical examples in which our bounds are \emph{provably} superior in Sections \ref{sec:examples} and \ref{sec:experiments}. As a consequence of our improvements to algorithmic robustness, we can deduce immediate improvements to the many applications of algorithmic robustness listed above, ranging from invariant classifiers to numerical analysis  and stochastic learning algorithms. 
       
       \end{enumerate}

        \section{Preliminaries}
        This section introduces notation and  previous results. \vspace{-4pt}
        \subsection{Notation}  \vspace{-4pt}
        For an integer $n$, we use $[n]$ to denote the set of integers $1, \dots, n$. For a finite set $\mathcal{B}$, we let $|\mathcal{B}|$ represent the number of elements in $\mathcal{B}$. For a set $\Scal$ equipped with metric $\rho$, we define $\hat{\Scal}$ as an $\varepsilon$-cover of $\Scal$ if for all $s \in \Scal$, there exists $\hat{s} \in \hat{\Scal}$ such that $\rho(s, \hat{s}) \leq \varepsilon$. We then define the $\varepsilon$-covering number as
        $$\Ncal(\varepsilon, \Scal, \rho) = \min\{|\hat{\Scal}|: \hat{\Scal} \text{ is an } \varepsilon\text{-cover of } \Scal\}.$$
        We use $\one(\cdot)$ as an indicator function, and $\|\cdot\|_p$ is the standard $p$-norm for a vector.

        \subsection{Problem Setting and Background}
        In this study, we are interested in bounding the expected loss $\EE_{z}[\ell(\Acal_S,z)]$, where $\EE_z$ denotes the expectation with respect to the sampling distribution.  This is a quantity that cannot be computed or accessed. Accordingly, we obtain an upper bound by using the training loss $\frac{1}{n} \sum_{i=1}^n \ell(\Acal_S, z_i)$, which is a computable quantity, and by invoking other computable terms. 
        A previous study \citep{xu2012robustness} used algorithmic robustness (Definition \ref{def:1}) to achieve the following result:
        \begin{proposition} \label{prop:1}
                \citep{xu2012robustness} Assume that for all $h \in \Hcal$ and $z \in \Zcal$, the loss is upper bounded by $B$ i.e., $\ell(h,z)\le B$. If the learning algorithm $\Acal$ is $(K,\epsilon(\cdot))$-robust (with $\{\Ccal_k\}_{k=1}^K$),
                then for any $\delta>0$, with probability at least $1-\delta$ over an
                iid draw of $n$ samples $S=(z_{i})_{i=1}^n$, the following holds: 
                \begin{align} \label{eq:olderror} 
                        &\EE_{z}[\ell(\Acal_S,z)] 
                        \\ \nonumber &\le\frac{1}{n} \sum_{i=1}^n \ell(\Acal_S, z_i) +\epsilon(S)+B \sqrt{\frac{2K\ln2+2 \ln (1/\delta)}{n}}. 
                \end{align}
        \end{proposition}   
        For example, in the special case of supervised learning, the sample space can be decomposed as $\Zcal=\Xcal \times \Ycal$, where $\Xcal$ is the input space and $\Ycal$ is the label space. However, note that $\Zcal$ can differ from the original space of the data points. For example, if the original data point is $\tilde z$, we can use $z = g(\tilde z)$ for any fixed-function $g$.

        The previous paper \citep{xu2012robustness} also proves the same upper bound on
        $\frac{1}{n} \sum_{i=1}^n \ell(\Acal_S, z_i) -\EE_{z}[\ell(\Acal_S,z)]$, instead of $\EE_{z}[\ell(\Acal_S,z)]-\frac{1}{n} \sum_{i=1}^n \ell(\Acal_S, z_i)$. However, the empirical loss $\frac{1}{n} \sum_{i=1}^n \ell(\Acal_S, z_i)$ can be minimized during training; hence, we are typically interested in the upper bound on $\EE_{z}[\ell(\Acal_S,z)]-\frac{1}{n} \sum_{i=1}^n \ell(\Acal_S, z_i)$. The focus on this quantity.

        The relationship between algorithmic robustness and the multinomial distribution is apparent when we consider independent samples from the sample space of $\{\Ccal_{k}\}_{k=1}^K$. Then, the number of samples from each class, $\Ccal_{k}$, is multinomially distributed with $p_k = \PP(z \in \Ccal_{k})$. The actual values of $p_k$ are not available to us. Therefore, it is natural that the concentration of the multinomial values around these expectations is required in the argument. 
        
        The concentration of a multinomial random variable is of interest in theoretical probability and practical use in applied fields such as statistics and computer science \citep{van1996weak}. Consequently, several concentration bounds have been proposed in the literature \citep{weissman2003inequalities, devroye1983equivalence, agrawal2017posterior, qian2020concentration, van1996weak}, for example:
        
        \begin{proposition}[Bretagnolle-Huber-Carol inequality] \citep[Proposition A.6.6]{van1996weak} \label{prop:multinomialold}
                If $X_1, \dots, X_K$ are multinomially distributed with parameters $n$ and $p_1, \dots, p_K$, then for any $\delta > 0$,
with probability at least $1-\delta$, 
 \begin{equation} \label{eq:multinomialold}
                \sum_{k=1}^K \left| p_k- \frac{X_k}{n} \right| \leq  \sqrt{\frac{2K \ln 2+2 \ln(1/\delta)}{n}} 
        \end{equation}                
        \end{proposition}              
        
        Crucially, the bounds in the literature are uniform in the parameters $p_k$, meaning that the right-hand side of the inequality is true for any set of parameters. 
        A key step in our reduction of the $\sqrt{K}$ dependence in algorithmic robustness is the non-uniform (and purely data-dependent) enhancement of the above concentration bound, which may be of independent interest beyond algorithmic robustness.

        \section{Main Theorems}  \label{sec:main}

        In this section, we record our improvements to Proposition \ref{prop:1} along with our upgraded bounds for the multinomial distribution. We discuss our main contributions and relegate the complete proofs  of theoretical results to the appendices.
        
        \subsection{Algorithmic Robustness}
        One of the main contributions of this study is the following refinement of the algorithmic robustness bound:

        \begin{theorem} \label{thm:1}
                If the learning algorithm $\Acal$ is $(K,\epsilon(\cdot))$-robust (with $\{\Ccal_k\}_{k=1}^K$), 
                then for any $\delta>0$, with probability at least $1-\delta$ over an
                iid draw of $n$ samples $S=(z_{i})_{i=1}^n$, the following holds: 
                \begin{align} \label{eq:simplebound}
                        \EE_{z}[\ell(\Acal_S,z)] &\le \frac{1}{n} \sum_{i=1}^n \ell(\Acal_S, z_i)+\epsilon(S) \\
                        \nonumber &\quad +\zeta(\Acal_S) \Bigg((\sqrt{2}+1) \sqrt{\frac{|\Tcal_{S}|\ln(2K/\delta)}{n}} \\
                        \nonumber &\hspace{85pt} + \frac{2|\Tcal_{S}|\ln(2K/\delta)}{n}\Bigg), \end{align}
where $\Ical_{k}^S:=\{i\in[n]: z_{i}\in \Ccal_{k}\}$,
 \begin{flalign*}
 & \zeta(\Acal_S):=\max_{z \in \Zcal}\{\ell(\Acal_S,z)\}, \text{ and }
\\ & \Tcal_{S}:=\{k\in[K]: |\Ical_{k}^{S}|\ge 1\}. 
\end{flalign*}
\end{theorem}

        Theorem \ref{thm:1} is a significant improvement over the previous bound \eqref{eq:olderror} of Proposition \ref{prop:1} as \eqref{eq:simplebound} has a far better  dependence on $K$. In terms of $K$, we have reduced $\sqrt{K}$ to $\sqrt{\ln K}$. Overall, if we ignore the log factor, $K$ in the previous bound is replaced by $|\Tcal_{S}|$ in our bound. Here,  $|\Tcal_{S}|$ is the number of distinct classes, $\Ccal_k$, \textit{that actually appear in the single specific training  dataset $S$}; thus, $ |\Tcal_S|\le K$ by the definition and it is shown to have  $|\Tcal_{S}| \ll K$ in many general cases later in  Proposition \ref{p:pkdecay} and Sections \ref{sec:example:comparison} and \ref{sec:experiments}. For example, our experimental results in Section \ref{sec:experiments} indicate that $|\Tcal_{S}| \ll K$ in many natural settings, where we see an exponential discrepancy in a variety of real-world data sets and theoretical models. 
        
        Intuitively, $|\Tcal_{S}|$ is likely to be significantly less than $K$ when there are many sparsely populated classes $\Ccal_{k}$. Therefore, it is improbable that many of these classes are represented in the sample data. Our theoretical and experimental and   results demonstrate that such scenarios are prevalent in the field.

The following proposition shows that   $| \Tcal_{S}|$ is  indeed independent of $K$ and only scales as  $\ln n$ under a general mild condition on $p_k$, proving that we have $| \Tcal_{S}| \ll K$ and $|\Tcal_{S}| \ll n$ in  a  general case:

        \begin{proposition}\label{p:pkdecay}
        Under the assumptions of Theorem \ref{thm:1}, we denote $p_k=\mathbb P(z\in \Ccal_k)$ where $p_1\geq p_2\geq \cdots \geq p_K$. If there are some constants $\alpha, \beta, C>0$ such that  $p_k$ decays as $p_k\leq C e^{-(k/\beta)^\alpha}$, and $\ln n\geq \max\{1, 2/\alpha\}$ then with probability at least $1-\delta$, the following holds:
\begin{align*}
        |\Tcal_{S}|\leq \begin{cases} \beta (\ln n)^{\frac{1}{\alpha}}+ C(e-1)\frac{\beta}{\alpha} +\log(1/\delta) & \text{if } \alpha\geq 1, \\
(1+2C(e-1))\beta(\ln n)^{\frac{1}{\alpha}} +\log(1/\delta) & \text{if } \alpha< 1. \\
\end{cases}
\end{align*}
\end{proposition}
In Proposition \ref{p:pkdecay}, $\alpha$ controls the rate of decay for $p_k$. For real-world data sets, we expect the data distribution concentrates on a lower dimension manifold or around small number of modes. In such settings, we expect the probability $p_k=\mathbb P(z\in \Ccal_k)$ (arranging them in decaying order) exhibits fast decays. If $\alpha=\infty$, $p_k$ concentrates on \textit{unknown} $\beta$ bins and we have $|\Tcal_{S}|\leq \beta$. If $\alpha<\infty$, we have $p_k \neq 0$ for all $k \in [K]$, but $|\Tcal_{S}|$ is still upper bound by $\beta$ times the constant up to a logarithmic factor and is independent of $K$.

 Proposition~\ref{p:pkdecay} also demonstrates the fact that even with perfect prior knowledge of the data distribution, $|\mathcal{T}_S|$ can be much smaller than $K$   because $|\mathcal{T}_S|$ is more adaptive according to the training data while $K$ cells  need to cover all the possible parts that the distribution has positive mass on. Without the perfect  knowledge, $|\mathcal{T}_S|$ can be  more significantly smaller  than $K$.

        A crucial aspect of Theorem \ref{thm:1} is that $\Tcal_{S}$ depends exclusively on the training sample data and not the actual background distribution. Accordingly, our result is of practical value in statistical learning settings, where information about the actual distribution can only be obtained through the training sample data.

Although the main breakthrough in Theorem \ref{thm:1} is the reduced $K$ dependence, there is also a substantially refined dependence on the upper bound of the loss value --- the replacement of $B$ with $\zeta(\Acal_S)$, where $\zeta(\Acal_S)\le B $; i.e., we replace  a maximum over the entire hypothesis space with the single hypothesis returned by the algorithm. This can be a significant advantage for common loss functions, such as square loss and cross-entropy loss. 
        Note that $B$ in the previous bound is defined to be larger than $\ell(h,z)$ for all $h \in \Hcal$, meaning that $B$ is dependent on \textit{the entire hypothesis space} $\Hcal$. In contrast, $\zeta(\Acal_S)$ in our bound depends only on the single actual hypothesis, $\Acal_{S}$, returned by the specific algorithm for each data set $S$.

        With a more refined analysis, we also prove a stronger (yet more complicated) version of Theorem \ref{thm:1}:
        
        \begin{theorem} \label{thm:2}
                If the learning algorithm $\Acal$ is $(K,\epsilon(\cdot))$-robust (with $\{\Ccal_k\}_{k=1}^K$),
                then for any $\delta>0$, with probability at least $1-\delta$ over an
                iid draw of $n$ examples $S=(z_{i})_{i=1}^n$, the following holds: 
                \begin{align}
                        \EE_{z}[\ell(\Acal_S,z)] 
                        &\le \frac{1}{n} \sum_{i=1}^n \ell(\Acal_S, z_i)+\epsilon(S) \\
                        \nonumber &\qquad +\mathcal{\mathcal{Q}}_1 \sqrt{\frac{ \ln(2K/\delta)}{n}} + \frac{2\mathcal{\mathcal{Q}}_2\ln(2K/\delta)}{n} \end{align}
                where 
                $$\mathcal{\mathcal{Q}}_1:=\sum_{k \in\Tcal_S} (\alpha_{\Tcal_S^c}(\Acal_S)+\sqrt{2}\alpha_k (\Acal_S) ) \sqrt{\frac{|\Ical_{k}^{S}|}{n}},$$ 
                $$\mathcal{\mathcal{Q}}_2:=\alpha_{\Tcal_S^c}(\Acal_S)| \Tcal_{S}| + \sum_{k \in\Tcal_S} \alpha_k (\Acal_S),$$ 
                 with $\Tcal_{S}:=\{k\in[K]: |\Ical_{k}^{S}|\ge 1\},$
                 $\Ical_{k}^S:=\{i\in[n]: z_{i}\in \Ccal_{k}\}$, $\alpha_{k}(h):=\EE_{z}[\ell(h,z)|z\in \Ccal_{k}]$,  $\alpha_{\Tcal_S^c}(\Acal_S):=\max_{k\in \Tcal^c_S} \alpha_{k}(\Acal_S)$, and $\Tcal^c_S := [K]\setminus \Tcal_S$.
        \end{theorem}

        Note that $\sum_{k \in \Tcal_{S}} \alpha_k(\Acal_{S}) \leq |\Tcal_{S}| \zeta(\Acal_{S})$ and $\sum_{k \in \Tcal_{S}} \sqrt{|\Ical_{k}^{S}|/n} \leq \sqrt{|\Tcal_{S}|}$ by the Cauchy-Schwarz inequality. \ Thus,  Theorem \ref{thm:2} is always as strong as Theorem \ref{thm:1}. Furthermore, 
        Theorem \ref{thm:2} significantly upgrades Theorem \ref{thm:1} approximately when $\sum_{k \in \Tcal_{S}} \alpha_k(\Acal_{S}) \ll |\Tcal_{S}| \zeta(\Acal_{S})$ or $\sum_{k \in \Tcal_{S}} \sqrt{|\Ical_{k}^{S}|/n} \ll \sqrt{|\Tcal_{S}|}.$ Otherwise put, if the maximum expected loss of the classes is much larger than the typical expected loss or the distribution of samples in the classes is lopsided, Theorem \ref{thm:2} will be an even tighter bound.
        
        The complete proofs of Theorems \ref{thm:1} and  \ref{thm:2} are provided in Sections \ref{sec:3} and  \ref{appendix:mainresult2} of the Appendix. We remark that our proof of Theorem \ref{thm:1} proves a stronger theoretical statement, where $\zeta(\Acal_{S})$ is  replaced by $\max_{k \in [K]}\EE_{z}[\ell(\Acal_S,z)|z\in \Ccal_{k}]$ ($\le\zeta(\Acal_{S}) $). This formulation may have advantages over that in Theorem \ref{thm:1} if the problem context reveals more information about the conditional expectation.

        \subsection{Concentration Bounds for the Multinomial Distribution}
        Let the vector $X=(X_1,\dots,X_K)$ follow a multinomial distribution with parameters $n$ and $p=(p_1,\dots,p_K)$. As shown in the proof of Theorem \ref{thm:1}, the key technical hurdle is to  avoid an explicit $\sqrt{K}$ dependence as we upper bound a quantity of the following form:
        \begin{align} \label{eq:multidependent}
                \sum_{i=1}^K a_{i}(X) \left(p_i - \frac{X_i}{n}\right),
        \end{align}
        where $a_i$ is an arbitrary function with $a_i(X)\ge 0$ for all $i \in \{1,\dots, K\}$. 
        
        Importantly, $a_i(X)$ are functions of $X_1,\dots,X_K$, which makes this problem particularly challenging and further complicates the non-trivial correlations already present in $X_i$. 
        This difficulty is avoided in \citep{qian2020concentration, bellet2015robustness, xu2012robustness} by using the global maximum of the loss function with the  $\sqrt{K}$
dependence. Allowing $a_i(X)$ to depend on $X$ is critical in our analysis and underpins the improvement from $B$ in Proposition \ref{prop:1} to $\zeta(\Acal_{S})$ in Theorem \ref{thm:1}, in addition to the improvement from $\sqrt{K}$ to $\sqrt{\ln K}$.

        One example of our new multinomial bounds that are non-uniform is the following lemma.
        
        \begin{lemma} \label{lemma:multinomialnew}
                For any $\delta>0$, with probability at least $1-\delta$, 
                $$
                \sum_{i=1}^K a_i(X) \left(p_i- \frac{X_i}{n} \right)\le\left(\sum_{i=1}^K a_i(X) \sqrt{p_i}\right) \sqrt{\frac{2 \ln(K/\delta)}{n}}.
                $$  
        \end{lemma}        
        By comparing Lemma \ref{lemma:multinomialnew} to  Proposition \ref{prop:multinomialold}, in some range of parameters, we have essentially replaced $\sqrt{K}$ with $\sum_i \sqrt{p_i}$. If $p_i = 1/K$ for all $i$, then we recover \eqref{eq:multinomialold}. Conversely, in the other extreme case, if $p_1 \approx 1$ and the remaining $p_i$ are near zero, then $\sum \sqrt{p_i} \approx 1$. Thus, our result interpolates between these cases, and there is a wide range of possible distributions in which our bound is significantly better than Proposition \ref{prop:multinomialold}.
        
        While Lemma \ref{lemma:multinomialnew} is of independent interest, it is likely difficult to be used directly in machine learning because it depends on the probability distribution $p$, which is typically unknown. To overcome this issue, we further refine Lemma \ref{lemma:multinomialnew} and remove the dependence on $p$. To keep the notation consistent, we write $p_k=\PP(z\in \Ccal_{k})$ and $X_k=\sum_{i=1}^n \one\{z_i\in \Ccal_{k}\}$ with $\{\Ccal_{k}\}_{k=1}^K$ being arbitrary in this subsection. Since $\{\Ccal_{k}\}_{k=1}^K$ is arbitrary here, this still represents general multinomial distributions.

        One of the main contributions of this study is the following refinement of the concentration bound on general multinomial distributions: 
        
        \begin{theorem} \label{lem:weightedmultinomial}
                For any $\delta>0$, with probability at least $1-\delta$,                 \begin{align} \label{eq:new:1}
                        &\sum_{i=1}^K a_i(X) \left(p_i - \frac{X_i}{n}\right) \\ \nonumber &  \le \tilde \Qcal \sqrt{\frac{| \Tcal_S| \ln(2K/\delta)}{n}} + a_{\Tcal_S^c}(X) \frac{2| \Tcal_S|\ln(2K/\delta)}{n},
                \end{align}
        where ${\mathcal{\tilde Q}} =a_{\Tcal_S}(X)\sqrt{2} +a_{\Tcal_S^c}(X) $,  $a_{\Tcal_S}(X)=\max_{i\in \Tcal_S} a_i(X)$ and $a_{\Tcal_S^c}(X)=\max_{i\in \Tcal^c_S} a_i(X)$ with $\Tcal^c_S = [K]\setminus \Tcal_S$.         
        \end{theorem}

        Further results of this nature and their technical proofs can be found in Sections \ref{appendix:fixeda} and \ref{appendix:dependenta} of the Appendix. 

\subsection{Discussion and Extensions }
        \subsubsection{Proof Ideas and Challenges}
        The proof of Theorem \ref{thm:1} can be divided into three phases.  In the first, we prove (a stronger version of) Lemma \ref{lemma:multinomialnew}. Next, we invoke Lemma \ref{lemma:multinomialnew} to prove Theorem \ref{lem:weightedmultinomial}. Finally, we deduce Theorem \ref{thm:1} from Theorem \ref{lem:weightedmultinomial}. Thus, the first obstacle is to establish Lemma \ref{lemma:multinomialnew}, which supplants Proposition \ref{prop:multinomialold}. After this key step,  the next challenge lies in going from Lemma \ref{lemma:multinomialnew} to Theorem \ref{lem:weightedmultinomial}, which requires us to remove the $\sum_i \sqrt{p_i}$ term in Lemma \ref{lemma:multinomialnew} without incurring the  $\sqrt{K}$ dependence. For example,  if we naively  bound $\sum_i \sqrt{p_i}$ with the Cauchy--Schwarz inequality,
we have that 
\begin{align*}
\sum_{i=1}^K \sqrt{p_i}  \le \sqrt{\sum_{i=1}^K p_i}  \sqrt{\sum_{i=1}^K 1} =  \sqrt{K}. 
\end{align*}
Similarly, if we apply Jensen's inequality, which relies on the concavity of the square root function in this domain, we find that   
\vspace{-8pt}
\begin{align*} 
\frac{1}{K}\sum_{i=1}^K \sqrt{p_i}  \le \sqrt{\frac{1}{K}\sum_{i=1}^K  p_i}  =  \sqrt{\frac{1}{K}}, 
\end{align*}        
which again implies  that $\sum_{i=1}^K \sqrt{p_i} \le  \sqrt{K}$. Thus, both approaches reproduce the $\sqrt{K}$ dependence, illustrating the challenges of removing  the $\sum_i \sqrt{p_i}$ term without the  $\sqrt{K}$ dependence. Our novel observation is to first decompose the sum as
\vspace{-0pt}
\begin{align} \label{eq:idea:1}
\sum_{i=1}^K  \left(p_i - \frac{X_i}{n}\right) &=\sum_{i\in \Tcal_S}  \left(p_i - \frac{X_i}{n}\right)  + \sum_{i \notin \Tcal_S} p_i
\\ \nonumber & =\sum_{i\in \Tcal_S}  \left(p_i - \frac{X_i}{n}\right)  +\left(1- \sum_{i\in \Tcal_S} p_i\right), 
\end{align} 
and find an upper bound on the second term $1- \sum_{i\in \Tcal_S} p_{i}$ by using a \textit{lower} bound on $\sum_{i\in \Tcal_S}  \left(p_i - \frac{X_i}{n}\right)$. That is, if $ \sum_{i\in \Tcal_S}  \left(p_i - \frac{X_i}{n}\right)\ge  -c$ for some $c>0$, then $1- \sum_{i\in \Tcal_S} p_{i}\le1+c- \sum_{i\in \Tcal_S}\frac{X_i}{n}=c$. The second line of \eqref{eq:idea:1} is a conceptually crucial step where we convert the question of \textit{upper} bounding $ \sum_{i \notin \Tcal_S} p_i$  to that of \textit{lower} bounding $\sum_{i\in \Tcal_S} p_i$. If we directly upper bound $\sum_{i \notin \Tcal_S} p_i$, it will sustain a cost of $\sqrt{K-|\Tcal_S|}$, resulting in a $\sqrt{K}$ dependence. This step of decomposing  the sum and of converting the upper bound to a lower bound is designed to avoid the $\sqrt{K}$ dependence. Now, the problem has been reduced to finding tight lower and upper bounds on these quantities, which is still nontrivial and is described in Appendices  \ref{appendix:fixeda} and \ref{appendix:dependenta}.

        \subsubsection{Pseudo-Robustness}

        \citet{xu2012robustness} also introduced a more general notion of \emph{pseudo-robustness} which relaxes algorithmic robustness by only requiring the nearness of the loss functions to hold for a subset of the training samples:        \begin{definition}
                A learning algorithm $\Acal$ is\textit{ $(K,\epsilon(\cdot), \hn(\cdot))$ pseudo-robust }, for $K \in \NN,$ $\epsilon(\cdot):\Zcal^n \rightarrow \RR$, and $\hn(\cdot):\Zcal^n \rightarrow \{1,\dots,n\}$, where $\Zcal$ can be partitioned into $K$ disjoint sets, denoted by $\{\Ccal_k\}_{k=1}^K$, such that for all $S \in \Zcal^n$, there exists a subset of training samples $\hS$ with $|\hS|=\hn(S)$ and the following holds:
                $
                \forall s \in \hS, \forall z \in \Zcal, \forall k=1,\dots,K: \text{ if } s,z \in \Ccal_k, 
                \text{ then } |\ell(\Acal_S, s)-\ell(\Acal_S, z)| \le\epsilon(S). $
        \end{definition}         
For pseudo-robustness, we have proved the following analog of Theorem \ref{thm:1}:
        
        \begin{theorem} 
                If the learning algorithm $\Acal$ is $(K,\epsilon(\cdot), \hn(\cdot))$ is pseudo-robust (with $\{\Ccal_k\}_{k=1}^K$),
                then for any $\delta>0$, with probability at least $1-\delta$ over an
                iid draw of $n$ examples $S=(z_{i})_{i=1}^n$, the following holds: 
                \begin{align*}
                        \EE_{z}&[\ell(\Acal_S,z)] \le\frac{1}{n} \sum_{i=1}^n \ell(\Acal_S, z_i)+\frac{\hn(S)}{n}\epsilon(S) \\
                        &\qquad + \frac{n-\hn(S)}{n} \hzeta(\Acal,S) \\ 
                        &\qquad +\zeta(\Acal_S) \Bigg((\sqrt{2}+1) \sqrt{\frac{|\Tcal_{S}|\ln(2K/\delta)}{n}} \\
                        &\qquad \qquad \qquad \qquad \qquad + \frac{2|\Tcal_{S}|\ln(2K/\delta)}{n}\Bigg), \end{align*}
                where $\hzeta(\Acal,S):=\max_{(k, i) \in [K] \times [n] }\left|\alpha_{k}(\Acal_S)-\ell(\Acal_S ,z_{i}) \right|$, $$\zeta(\Acal_S):=\max_{k \in [K]}\EE_{z}[\ell(\Acal_S,z)|z\in \Ccal_{k}]$$ and $$\Tcal_{S}:=\{k\in[K]: |\Ical_{k}^{S}|\ge 1\}$$ with $\Ical_{k}^S:=\{i\in[n]: z_{i}\in \Ccal_{k}\}$.
        \end{theorem}
        
        Theorem \ref{thm:2} exhibits an analogous generalization. A precise statement and proof can be found in Appendix \ref{appendix:pseudo}. These theorems offer concomitant improvements to the pseudo-robustness bounds attained in \citep{xu2012robustness}.

        \section{Examples} \label{sec:examples}

        Our contributions augment the abstract framework of algorithmic robustness. We do not use application-dependent information nor do we append any restrictive assumptions, and, therefore, can deduce improvements to the many applications that employ the structure of algorithmic robustness. After presenting known examples in Section \ref{sec:example:robust}, we provide new theoretical comparisons via examples in Section \ref{sec:example:comparison}.

        \subsection{Robust Algorithms} \label{sec:example:robust}

        Although our Theorems \ref{thm:1} and \ref{thm:2} are applicable to a wide range of applications, this section provides only a few simple examples from \citep{xu2012robustness} to which our Theorems \ref{thm:1} and \ref{thm:2} can be applied. When we have the decomposition $z \in \Zcal = \Xcal \times \Ycal$ with $\Xcal$ as the input space and $\Ycal$ as the output space, we let $z^{(x)}\in \Xcal$ and $z^{(y)}\in \Ycal$ denote the $\Xcal$ and $\Ycal$ components of a sample point $z$, respectively, with  $z=(z^{(x)}, z^{(y)})$. 
We also write
        $S=(s_1,\dots,s_n)$.

        \begin{example} \citep{xu2012robustness}
                (Lipschitz continuous functions) Broad classes of learning problems are set in spaces with natural metrics. If the loss function is Lipschitz, which is a simple and natural condition, then the algorithm is robust. More precisely, if $\Zcal$ is compact with regarding a metric $\rho$ and $\ell(\Acal_S, \cdot)$ is Lipschitz continuous with Lipschitz constant $c(S)$, that is,
                $$
                |\ell(\Acal_S, z) - \ell(\Acal_S, z')| \le c(S)\rho(z,z'), \quad \forall z,z' \in \Zcal, 
                $$ 
                then $\Acal$ is $(\Ncal(\gamma/2,\Zcal, \rho), c(S)\gamma)$-robust for all $\gamma>0$.
        \end{example}
        
        \begin{example} \citep{xu2012robustness} (Lasso) \label{example:lasso}
                Lasso is a workhorse of modern machine learning \cite{hastie2019statistical}. We assume that $\Zcal$ is compact, and we use the loss function $\ell(\Acal_S, z) = |z^{(y)} - \Acal_{S}(z^{(x)})|$. Lasso can be formulated as an optimization problem.
                \[
                \mini_w: \frac{1}{n} \sum_{i=1}^n (s_i^{(y)} - w^\top s_i^{(x)})^2 + c \|w\|_1.
                \]
                This algorithm is $(\mathcal{N}(\nu/2, \Zcal, \|\cdot\|_{\infty}), \nu (\frac{1}{n} \sum_{i=1}^n (s_i^{(y)})^2)/c \allowbreak + \nu)$-robust for all $\nu > 0$. 
                
        \end{example}
        
        \begin{example} \citep{xu2012robustness} (Principal Component Analysis) 
        For $\Zcal \subset \RR^m$, a set with the maximum $\ell_2$ norm bounded by $B$. If we use the loss function  \vspace{-8pt}
        \[
        \ell((w_1, \dots, w_d), z) = \sum_{j=1}^d (w^\top_j z)^2, \vspace{-8pt}
        \] 
        then finding the first $d$ principal components via the optimization problem: \vspace{-8pt}
        \[
        \text{Maximize: } \sum_{i=1}^n \sum_{j=1}^d (w^\top_j s_i)^2 \vspace{-3pt}
        \]
        with the constraint that $\|w_j\|_2 = 1$ and $w_i^\top w_j = 0$ for $i \neq j$ is $(\Ncal (\gamma/2, \Zcal, \|\cdot \|_2), 2 d \gamma B)$-robust, for all $\gamma > 0$. 

        \end{example}

        Theorems \ref{thm:1} and \ref{thm:2} can be used as a black box mathematical tool  for many more of the existing applications cataloged in the introduction. 
        
       \subsection{Theoretical Comparisons} \label{sec:example:comparison}

       Here, we further  demonstrate the advantage of using our bounds in Section \ref{sec:main} over the bounds provided by \citet{xu2012robustness} and the bounds obtained via uniform stability \citep{bousquet2002stability}, using Lasso, least square regression, neural networks, and discrete-valued neural communication as examples.

   The first example demonstrates that when the data are embedded with high probability on a low-dimensional manifold in the data space,  and our bound is much stronger than that of \citep{xu2012robustness}: 
        
        \begin{example}[Lasso] \label{example:2} Recall that in Example \ref{example:lasso}, Lasso is $(\mathcal{N}(\nu/2, \Zcal, \|\cdot \|_{\infty}), \nu (\frac{1}{n} \sum_{i=1}^n (s_i^{(y)})^2)/c + \nu)$-robust for all $\nu > 0$. Consider $z^{(y)}\in\mathbb{R}$ and $z^{(x)}\in\mathbb{R}^d$. Given any $\nu>0$, let $z$ follow a distribution $\mathcal{D}_z$, such that $z^{(x)}=({x^{(1)}}^\top,{x^{(2)}}^\top)^\top$, where $x^{(1)}\sim N(0,I_p)|_{[-1,1]^p}$ (truncated Gaussian on $[-1,1]^p$), $x^{(2)}\sim N(\mu,\sigma^2I_r)|_{[-1,1]^{r}}$, and $r=d-p$, $z^{(y)}= w^{*\top}z^{(x)}$, where $\|w^*\|_1\le 1$. 
        For sufficiently small $\sigma$, we can check that the $\nu$-covering of the data space $[-1,1]^d$ satisfies Proposition \ref{p:pkdecay}, with $\beta=(2/\nu)^{p}$ and $\alpha=2$. As a consequence, we have that $|\Tcal_{\mathcal S}|=\Theta((2/\nu)^{p+1})$ with a probability of at least $1-\delta$. Since $\mathcal{N}(\nu/2, \Zcal, \|\cdot \|_{\infty})=\Theta((2/\nu)^{d+1})$, there exists $D, N$ such that for any $d>D$ and $n>N$, when $d\gg p$,  there exists $(\mu,\sigma)$ such that  the bound in Theorem 1 is much tighter than that in Proposition 1 as $|\Tcal_{\mathcal S}|=\Theta((2/\nu)^{p+1})\ll\Theta((2/\nu)^{d+1})=\mathcal{N}(\nu/2, \Zcal, \|\cdot \|_{\infty})$.  See Appendix \ref{sec:theory} for more details. \end{example}

   In the next example, we can see that our bound is much tighter than the bound obtained via uniform stability when there are outliers in the data: 
        \begin{example}[Regularized least square regression] \label{example:1}
        We refer to the example in \citep{bousquet2002stability} for regularized least squares regression. Specifically, $z^{(y)}\in [0,B]$ and $z^{(x)}\in[0,1]$. The regularized least squares regression is defined as 
        $\mathcal{A}_{S}=\argmin_w\frac{1}{n}\sum_{i=1}^n\ell(w,z_i)+\lambda|w|^2, $ 
        where $\ell(w,z)=(w\cdot z^{(x)} -z^{(y)})^2$ and $w\in \mathbb{R}$.
 For this example,       \citet{bousquet2002stability} observe that
        $0\le \ell(w,z)\le \sqrt{B/\lambda},$
        and establish the stability bound
        $\beta\le \frac{2B^2}{\lambda n}.$
        Now, consider the following distribution on $z$: $z^{(y)}=w^*\cdot z^{(x)}+\epsilon\bm1(|\epsilon|<B)$. In addition, $z^{(x)}$ follows a continuous distribution on $[0,1]$, for instance, a uniform distribution on $[0,1]$, and $\epsilon\sim N(0,\sigma^2)$. Without loss of generality, let $w^*=1$. In this example, we can possibly observe some large outlier with small probability. One can check that by suitably chosen $\sigma$, Proposition \ref{p:pkdecay} holds with $\beta=2/\nu$ and $\alpha=2$. Thus, there exists a threshold $N$ such that for any $n>N$, there exists $\nu>0,\sigma>0$ and $\delta>0$,  with a probability of at least $1-\delta$, $|\mathcal T_S|=\Theta(2\nu)$. Thus, if $B^2/\lambda\gg 2/\nu$, the bound in Theorem 1 is a far more precise bound than that obtained via uniform stability. For details of the proof, we refer the reader to Appendix \ref{sec:theory}.
     \end{example}  

\begin{figure*}[!ht]
\begin{subfigure}[b]{0.24\textwidth}
  \includegraphics[width=\textwidth, height=0.8\textwidth]{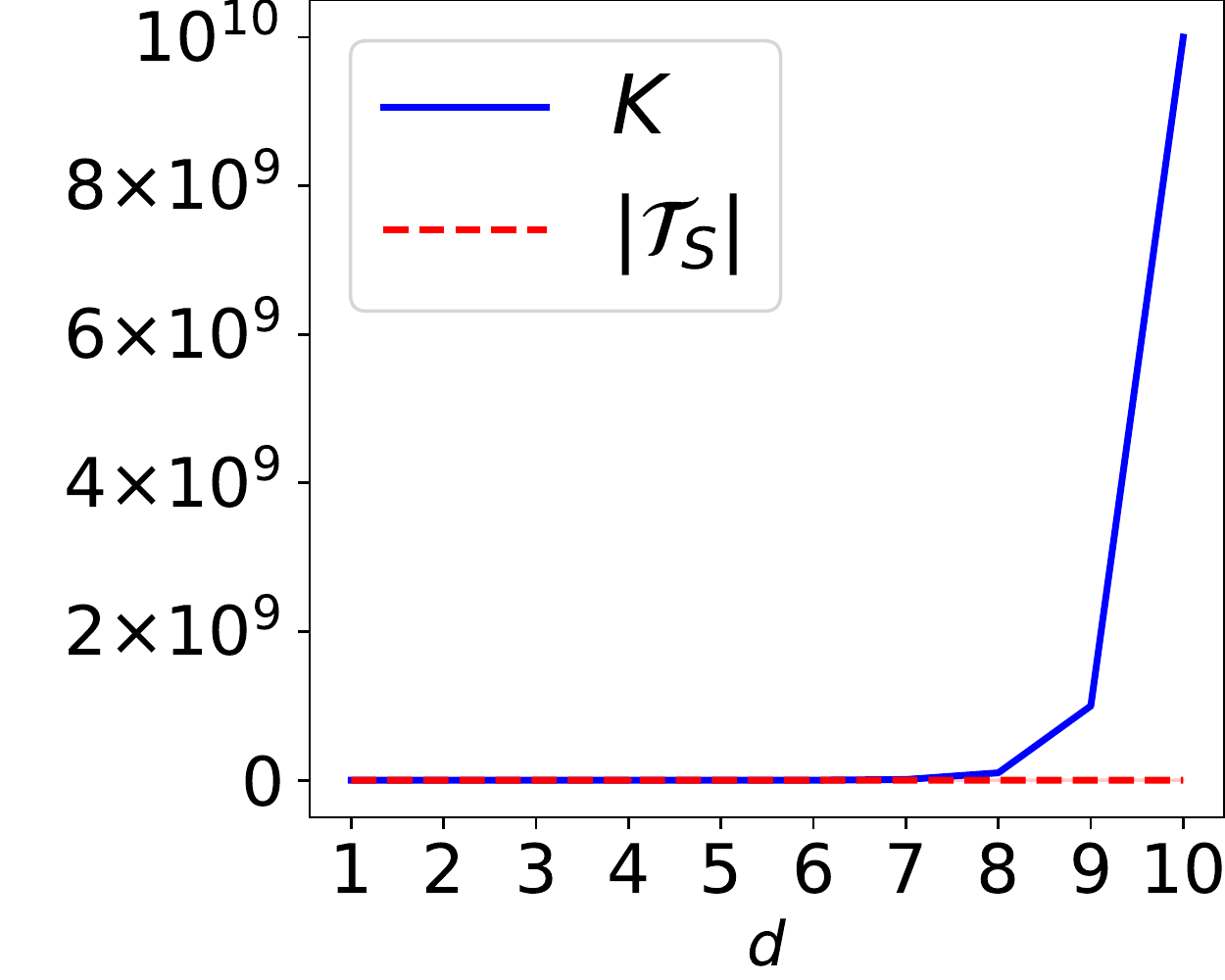}
  \vspace{-15pt}
  \caption{Beta(0.1, 0.1)} 
  \vspace{-8pt}
\end{subfigure}
\begin{subfigure}[b]{0.24\textwidth}
  \includegraphics[width=\textwidth, height=0.8\textwidth]{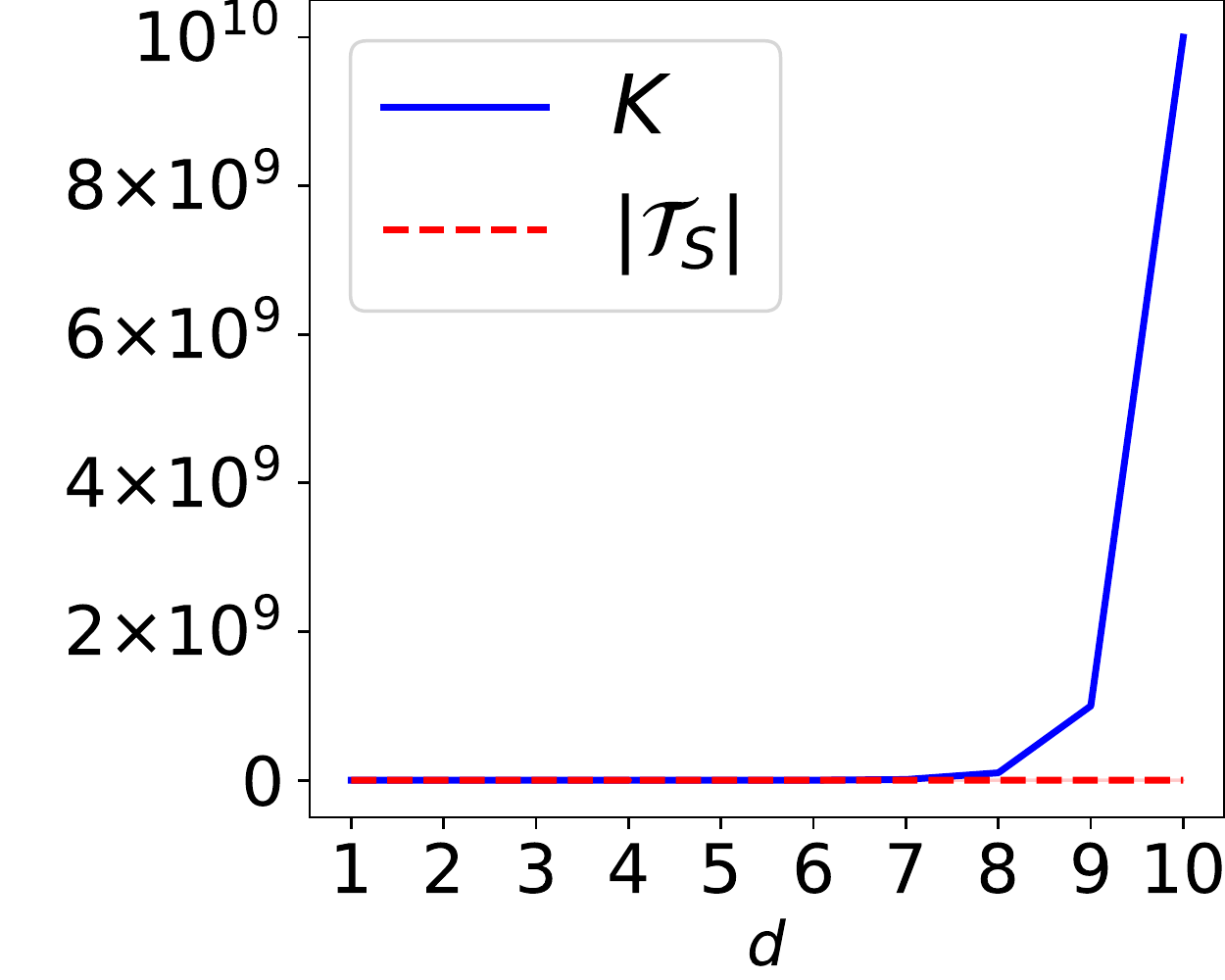}
  \vspace{-15pt}
  \caption{Beta(0.1, 10)} 
  \vspace{-8pt}
\end{subfigure}
\begin{subfigure}[b]{0.24\textwidth}
  \includegraphics[width=\textwidth, height=0.8\textwidth]{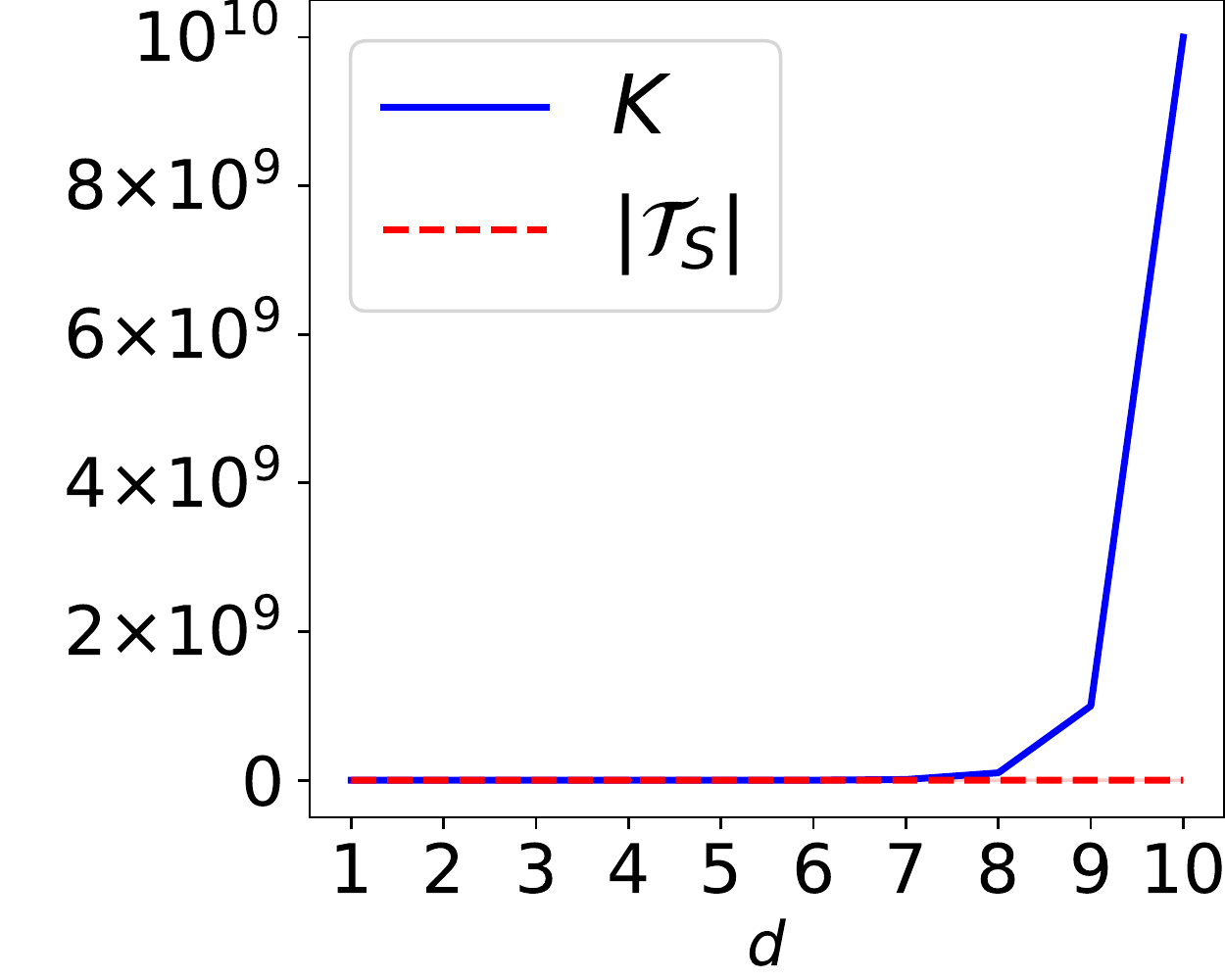}
  \vspace{-15pt}
  \caption{Gauss mix (0.01)} 
  \vspace{-8pt}
\end{subfigure}
\begin{subfigure}[b]{0.24\textwidth}
  \includegraphics[width=\textwidth, height=0.8\textwidth]{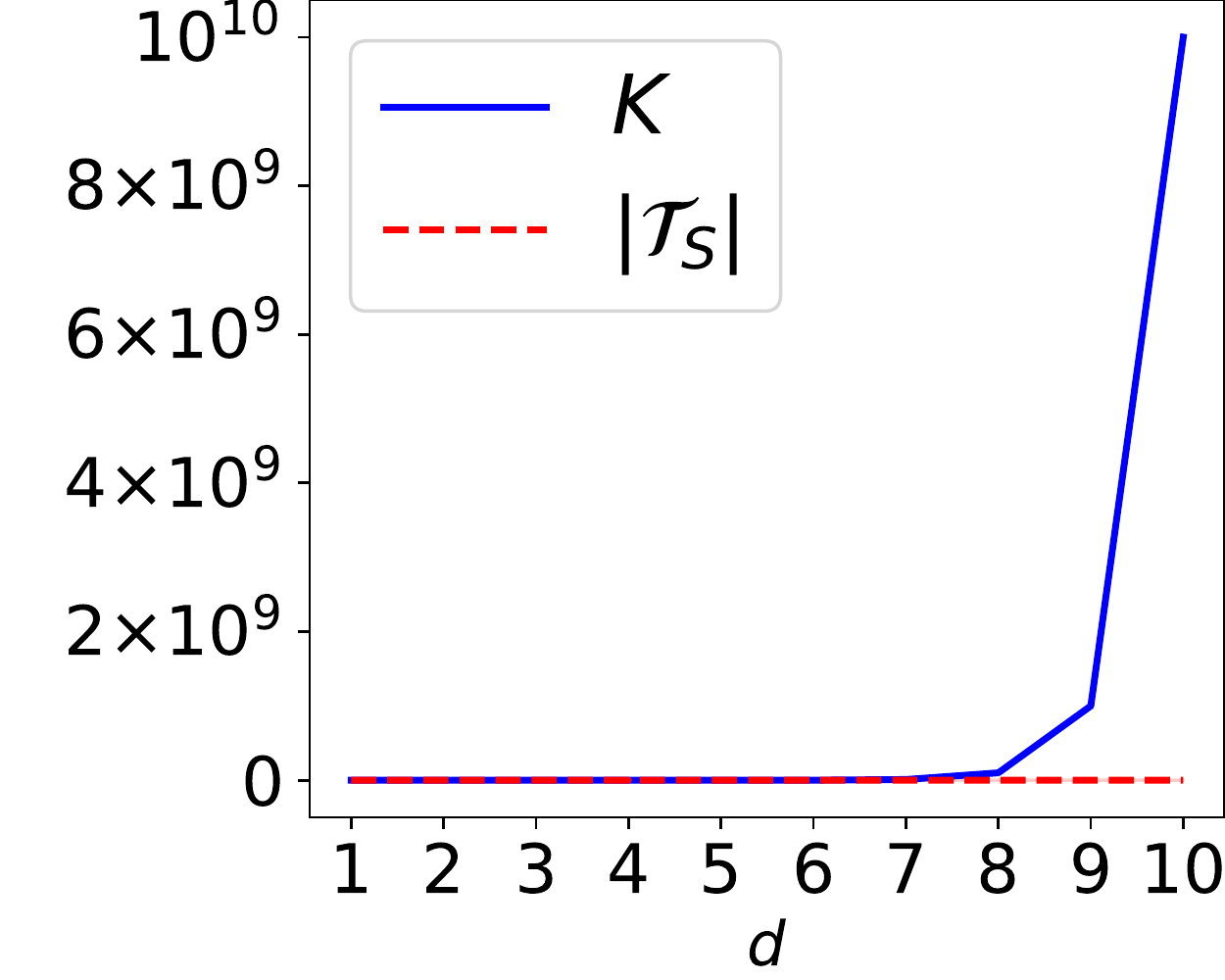}
  \vspace{-15pt}
  \caption{Gauss mix (1.0)} 
  \vspace{-8pt}
\end{subfigure}
\caption{The values of $K$ versus $|\Tcal_S|$ with synthetic data and the $\epsilon$-covering of the original space. The plot shows the mean of 10 random trials. The value of $K$ increases exponentially as $d$ increases linearly, whereas $|\Tcal_S|$ does not. } \label{fig:1} 
\end{figure*}        
\begin{figure*}[!ht]
\begin{subfigure}[b]{0.24\textwidth}
  \includegraphics[width=\textwidth, height=0.8\textwidth]{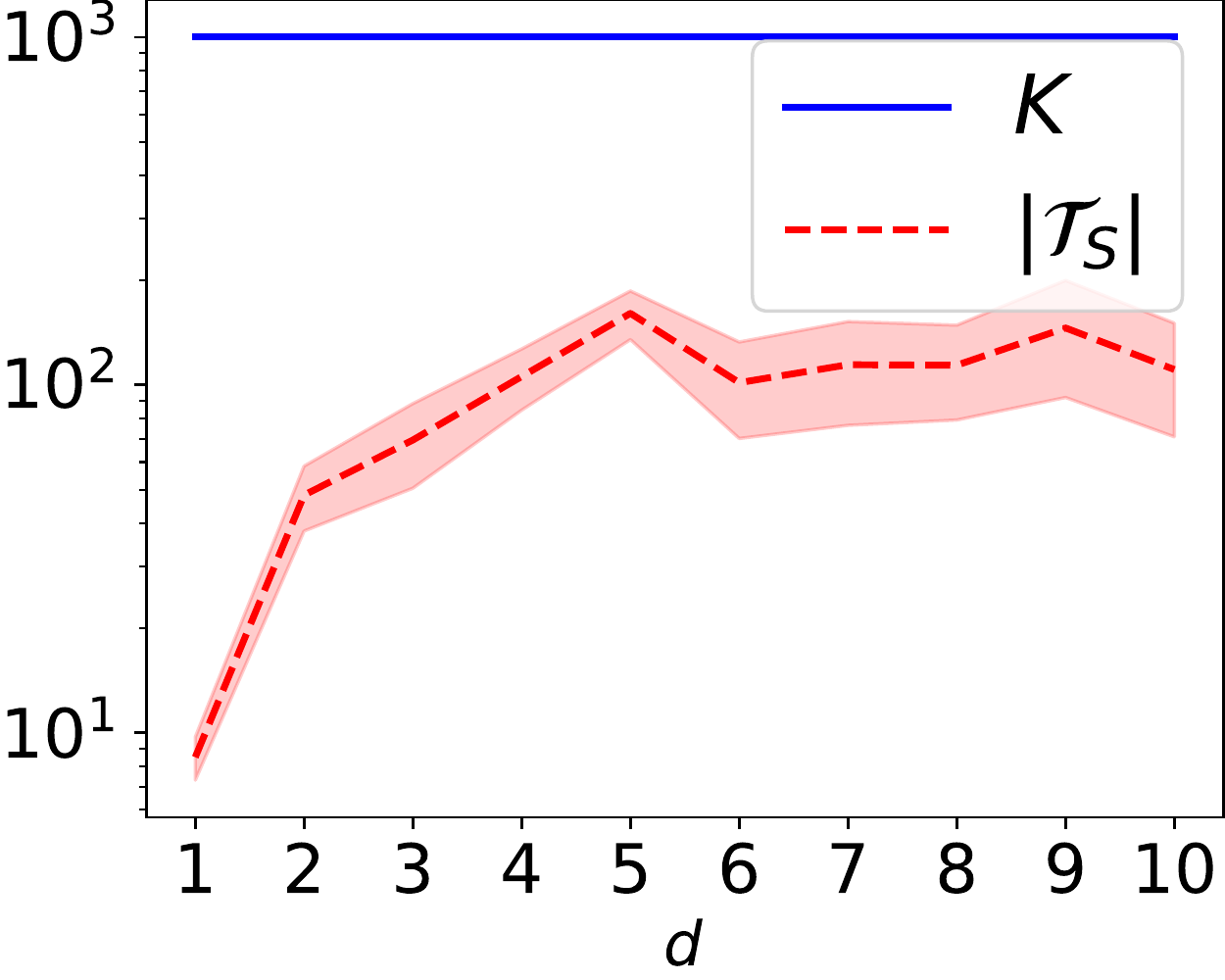}
  \vspace{-15pt}
  \caption{Beta(0.1, 0.1)} 
  \vspace{-8pt}
\end{subfigure}
\begin{subfigure}[b]{0.24\textwidth}
  \includegraphics[width=\textwidth, height=0.8\textwidth]{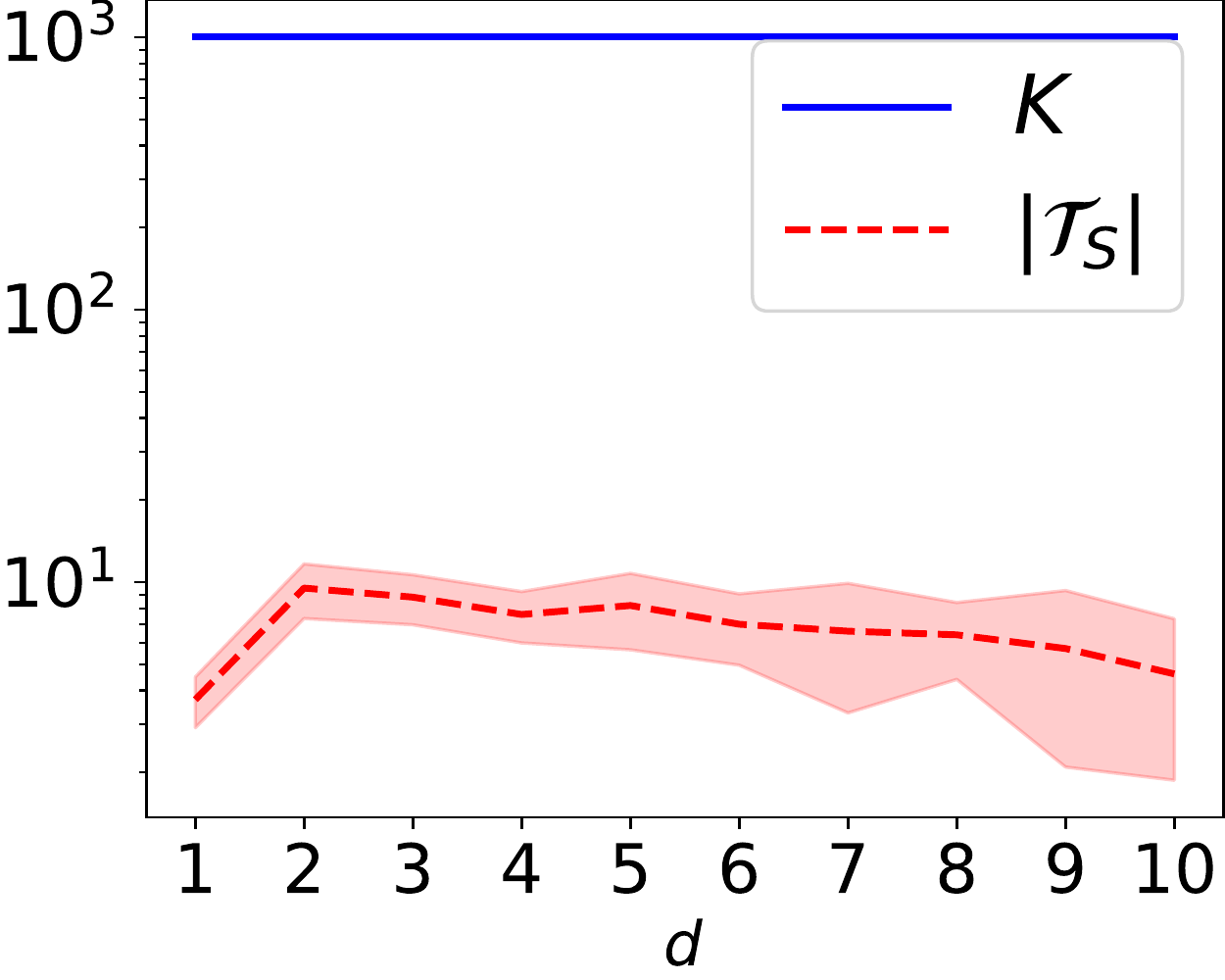}
  \vspace{-15pt}
  \caption{Beta(0.1, 10)} 
  \vspace{-8pt}
\end{subfigure}
\begin{subfigure}[b]{0.24\textwidth}
  \includegraphics[width=\textwidth, height=0.8\textwidth]{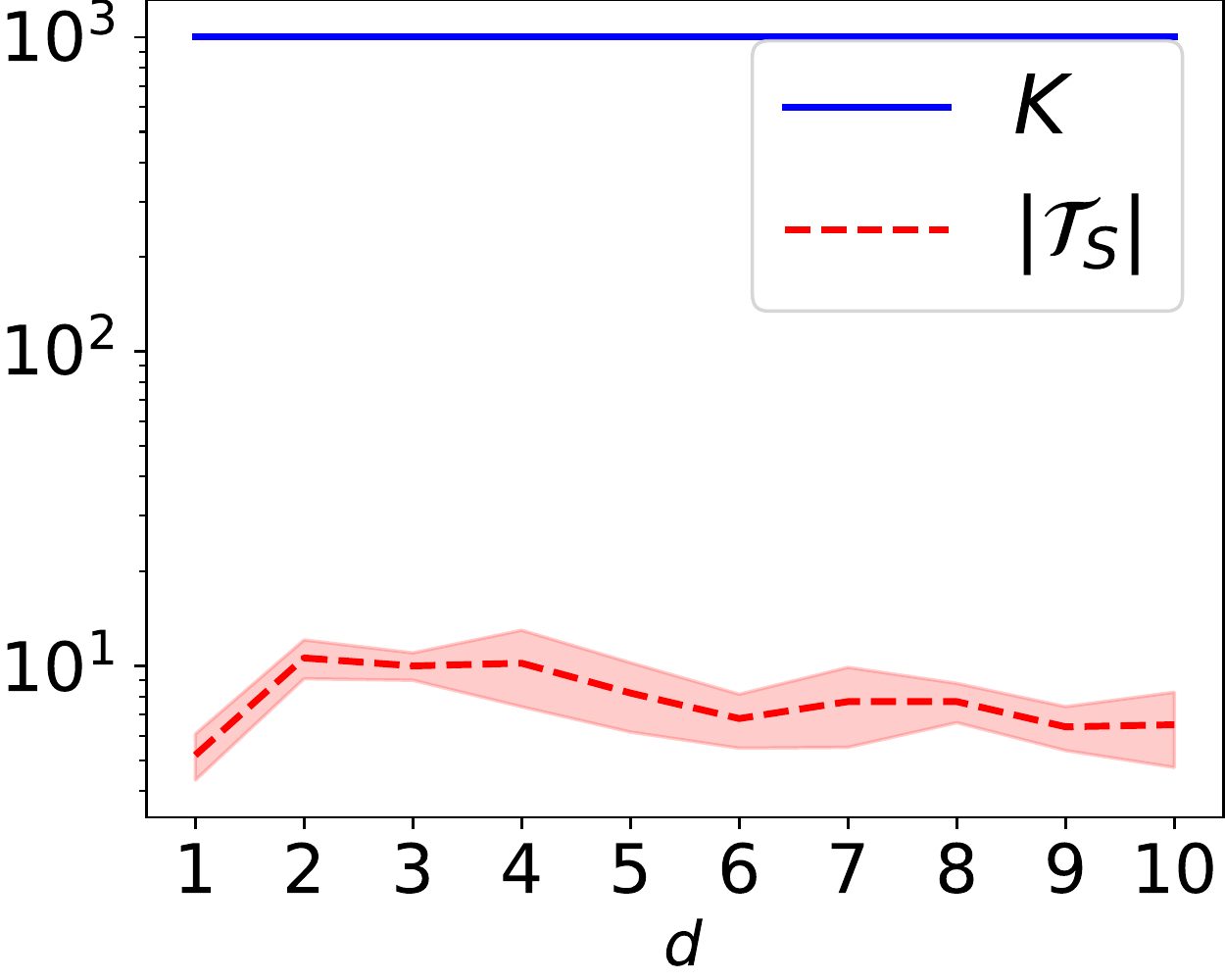}
  \vspace{-15pt}
  \caption{Gauss mix (0.01)} 
  \vspace{-8pt}
\end{subfigure}
\begin{subfigure}[b]{0.24\textwidth}
  \includegraphics[width=\textwidth, height=0.8\textwidth]{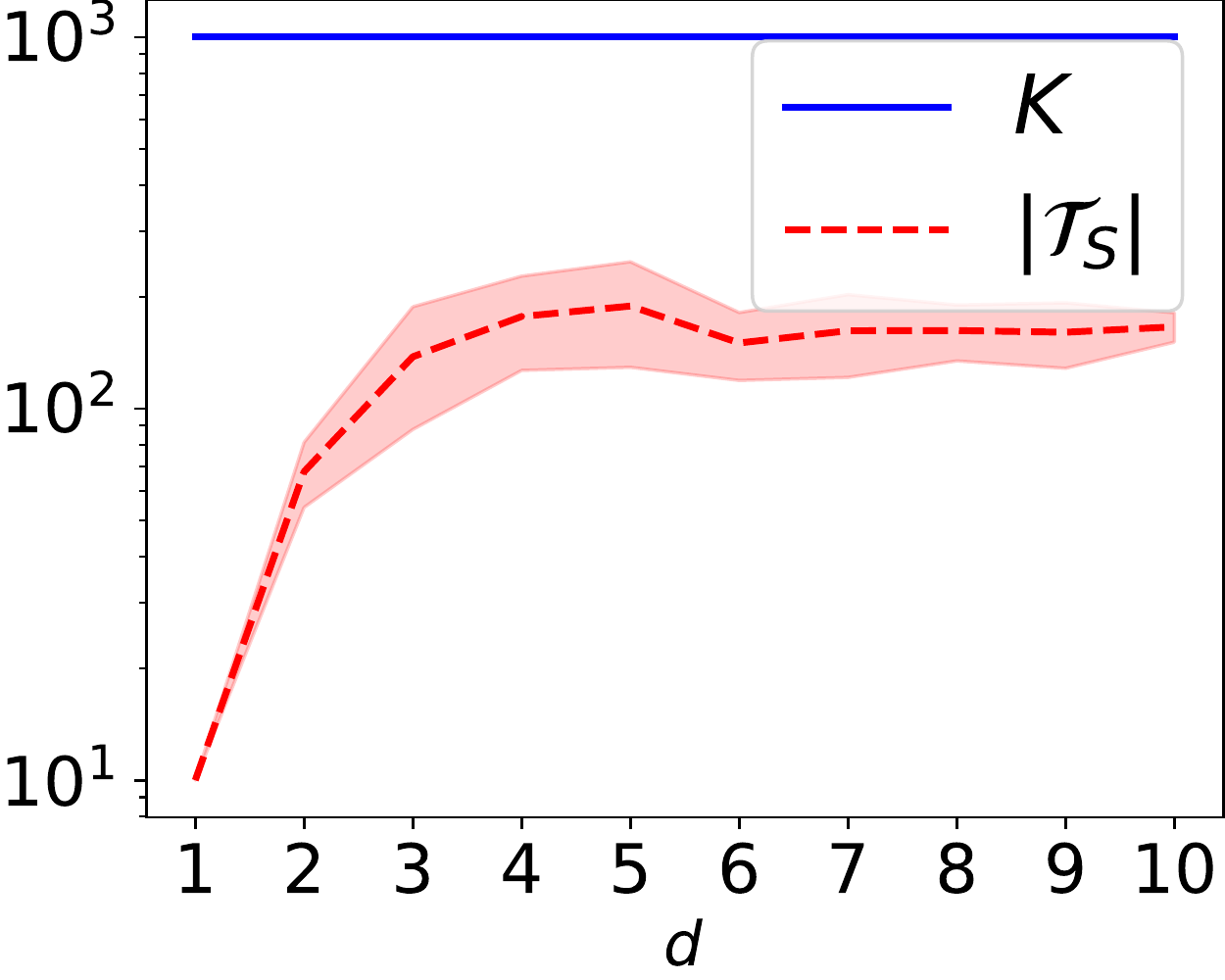}
  \vspace{-15pt}
  \caption{Gauss mix (1.0)} 
  \vspace{-8pt}
\end{subfigure}
\caption{The values of $K$ versus $|\Tcal_S|$ with synthetic data and the inverse image of the $\epsilon$-covering in randomly projected spaces. The plot shows the mean and one standard deviation of 10 random trials. We still have $|\Tcal_S|<K$ with random projections to reduce $K$.} \label{fig:2} 
\end{figure*}           

     Although Example \ref{example:1} compares our robustness bound  and the uniform stability bound, we emphasize that robustness and stability are very different properties and have distinct strengths and weaknesses: i.e., one setting prefers the robustness framework and another setting favors the stability approach. For example, a learning algorithm may be robust but not stable, e.g., lasso regression \citep{xu2009robustness}, and vice versa. Accordingly, this paper focuses on the fundamental advancements of the robustness framework and of the statistical bound for general multinomial distributions.

Furthermore, the following examples illustrate some of the immediate  improvements   for robust margin deep neural networks and discrete-valued neural communication:
\begin{example}[Robust  margin deep  neural networks] \label{example:new:1}
 The previous paper \citep{sokolic2017robust} uses Proposition \ref{prop:1} (in our paper) with the  $\epsilon$-covering of $\Xcal$ for robust  margin  neural networks. Our new theorem (Theorem \ref{thm:1} or Theorem \ref{thm:2})  immediately improve their bounds  by replacing  $K=\frac{2^{k+1} (C_M)^k}{\gamma^k_b}$ in their bounds with $|\Tcal_S|$. In our paper, the comparison of $K$ v.s.  $|\Tcal_S|$ for  the  $\epsilon$-covering of $\Xcal$ is shown in Figure \ref{fig:3}  for the real-life datasets. By plugging these values of  $K$ and  $|\Tcal_S|$  into the previous bounds  and our versions, we  yield  exponential improvements over the previous bounds for  robust  margin  deep neural networks.  
\end{example}

\begin{example}[Discrete-valued neural communication] \label{example:new:2}
The  bound in (previous) theorem 3 of the recent paper on \textit{discrete-valued neural communication} \citep{liu2021discrete} scales at the rate of $L^G$ (which is the size of the discrete bottleneck). Since their proof uses Proposition \ref{prop:multinomialold} (in our paper) to bound the left-hand-side of \eqref{eq:new:1} (in our Theorem \ref{lem:weightedmultinomial}), by applying our Theorem \ref{lem:weightedmultinomial}, we yield an improvement by replacing $L^G$ with $|\Tcal_S|$ for discrete-valued neural communication.    
\end{example}

        \section{Experiments} \label{sec:experiments}

This section establishes the advantage of our new bounds via experiments using both synthetic data and real-world data. We generated synthetic data by sampling from beta distributions and Gaussian mixture distributions with a variety of hyperparameters. For real-world data, we adopted the standard benchmark datasets: MNIST \citep{lecun1998gradient}, CIFAR-10 \allowbreak and CIFAR-100 \citep{krizhevsky2009learning},  SVHN \citep{netzer2011reading}, Fashion-MNIST (FMNIST) \citep{xiao2017fashion}, Kuzushiji-MNIST (KMNIST) \citep{clanuwat2019deep}, and Semeion \citep{srl1994semeion}. 

The value of $\epsilon(S)$ is exactly the same for the previous bound and our bound in all the experiments. To choose the partition $\{\Ccal_k\}_{k=1}^K$, in addition to other examples, we used the $\epsilon$-covering of the original input space $\Xcal$ as our primary example, as this is the default option in \citep{xu2012robustness}. The data space is normalized such that $\Xcal \subseteq [0,1]^d$ for the dimensionality $d$ of each input data. Accordingly, we used the infinity norm and a diameter of $0.1$ for the $\epsilon$-covering in all experiments. See Appendix \ref{app:exp} for more details on the experimental setup.

\subsection{Synthetic data} \label{sec:exp:synthetic}
 Figure \ref{fig:1} shows the values of $K$ and $|\Tcal_S|$ for the synthetic data with the partition $\{\Ccal_k\}_{k=1}^K$ being the $\epsilon$-covering of $\Xcal$. Here, Beta($\alpha$, $\beta$) indicates the Beta distribution with hyper-parameters $\alpha$ and $\beta$, and Gauss mix ($\sigma$) means the mixture of five Gaussian distributions with a standard deviation $\sigma$. 
Appendix \ref{app:exp} presents more results with different distributions, showing the same qualitative behavior in all cases.

While the $\epsilon$-covering of the original input space $\Xcal$ is the default example from the previous paper \citep{xu2012robustness}, in Figure \ref{fig:1} we see that $K$ grows rapidly as $d$ increases. Therefore, to reduce $K$ significantly, we also propose utilizing the inverse image of the $\epsilon$-covering in a randomly projected space. That is, given a random matrix $A$, we use the $\epsilon$-covering of the space of $u=Ax$ to define the pre-partition $\{\tilde \Ccal_k\}_{k=1}^K$. Then, the partition $\{\Ccal_k\}_{k=1}^K$ is defined by $\Ccal_k= \{x \in \Xcal : Ax \in \tilde \Ccal_{k}\}$. We randomly generated matrix $A \in \RR^{3 \times d}$ in each trial.

Figure \ref{fig:2} shows the values of $K$ versus $|\Tcal_S|$ for the synthetic data with the partition $\{\Ccal_k\}_{k=1}^K$ being the inverse image of the $\epsilon$-covering in randomly projected spaces. As can be seen, even in the case where $K$ is reduced via random projection, we have $|\Tcal_S| \ll K$. Thus, in both cases, our bounds are significantly tighter than the previous bound for these synthetics data.

\subsection{Real-world data}

Figure \ref{fig:3} shows the values of $K$ versus $|\Tcal_S|$ for the real-world data with the partition $\{\Ccal_k\}_{k=1}^K$ being the $\epsilon$-covering of $\Xcal$. All the training data points of each dataset were used. As can be seen, we  have $|\Tcal_S|\ll K$ for the real-world data.

To reduce the value of $K$, we additionally propose the following new method; i.e., as we have unlabeled data in many applications, we propose to use them to help define the partition $\{\Ccal_k\}_{k=1}^K$. The key idea here was that the choice of partition $\{\Ccal_k\}_{k=1}^K$ had to be independent of the labeled data used in the training loss in Theorems \ref{thm:1} and \ref{thm:2}, but it could depend on the unlabeled data. Otherwise expressed, given a set of unlabeled data points $\{\bar x_k\}_{k=1}^K$, the partition $\{\Ccal_k\}_{k=1}^K$ is defined by the clustering with the unlabeled data as $\Ccal_k= \{x \in \Xcal : k =\argmin_{k' \in [K]} \|x-\bar x_{k'} \|_2 \}$. Following the literature on semi-supervised learning, we split the training data points into labeled data points (500 for Semeion and 5000 for all other datasets) and unlabeled data points (the remainder of the training data). 

Figure \ref{fig:4} shows the values of $K$ versus $|\Tcal_S|$ for the real-world data with the partition $\{\Ccal_k\}_{k=1}^K$ being  the clustering with the unlabeled data. As can be seen, even in this case with  the significantly reduced $K$, we still have $|\Tcal_S| \ll K$.

Figure \ref{fig:5} shows the values of $K$ v.s. $|\Tcal_S|$ for  real-life datasets with the partition being the inverse image of the $\epsilon$-covering in randomly projected spaces. The random projection was conducted in the same manner without unlabeled data as in Figure \ref{fig:2}. The  projection reduced the value of $K$ significantly, and yet we still have $|\Tcal_S| \ll K$. Thus, in all  three cases, our bounds are significantly tighter than the previous bound for these real-world data.

\begin{figure}[t!]
\center
\begin{subfigure}[b]{0.4\textwidth}
  \includegraphics[width=\textwidth, height=0.52\textwidth]{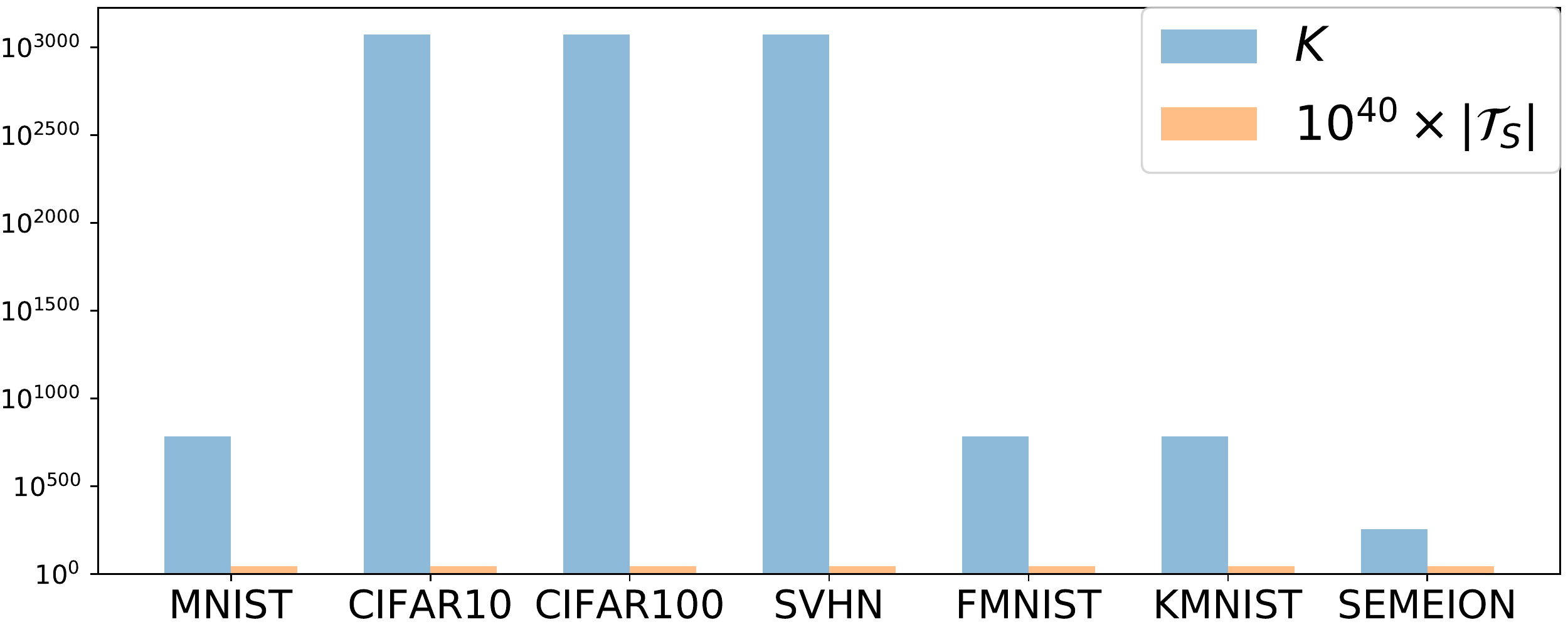}
  \vspace{-21pt}
\end{subfigure} 
\caption{The values of $K$ versus $|\Tcal_S|$ with real-world data and the $\epsilon$-covering. The values of $|\Tcal_S|$ are extremely small compared to those of $K$ in all datasets. } \label{fig:3} 
\end{figure} 

\begin{figure}[t!]
\center
\begin{subfigure}[b]{0.4\textwidth}
  \includegraphics[width=\columnwidth, height=0.6\textwidth]{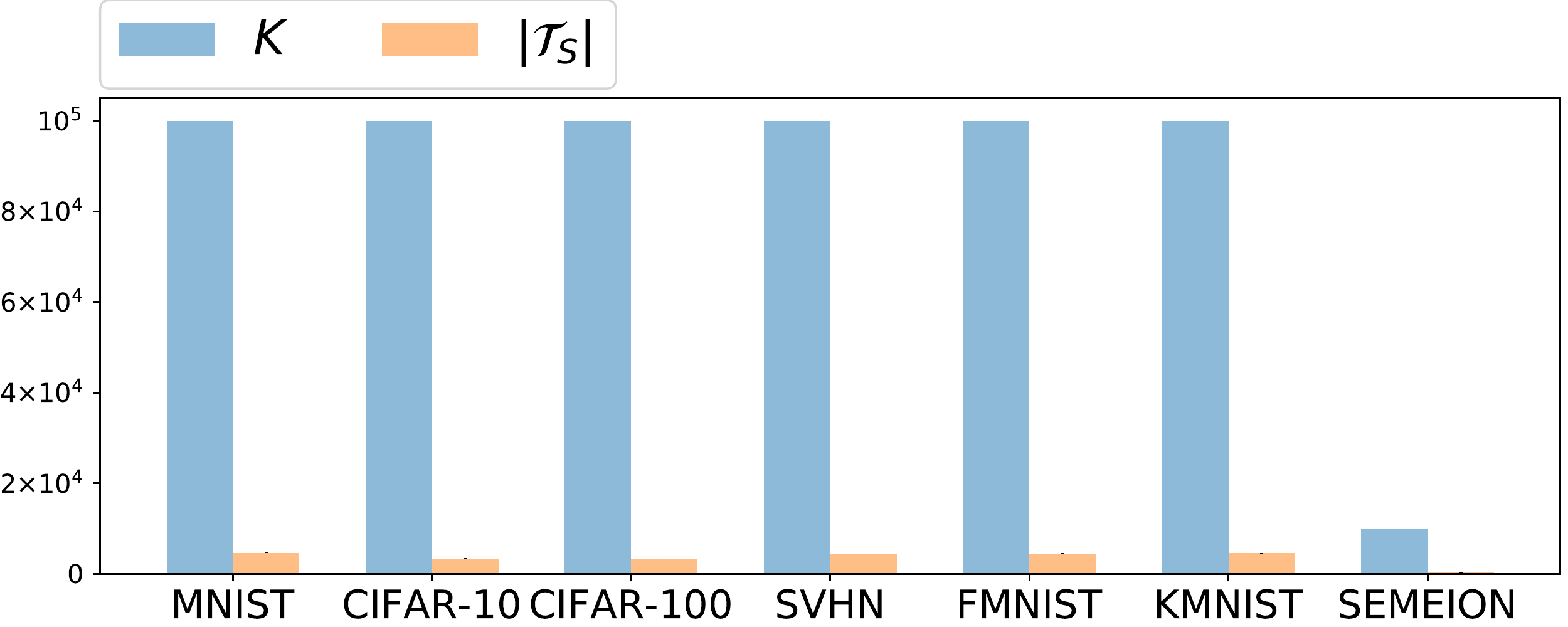}
  \vspace{-21pt}
\end{subfigure} 
    \caption{The values of $K$ versus $|\Tcal_S|$ with real-world data and the clustering using unlabeled data. With clustering to reduce $K$, we still have $|\Tcal_S|<K$. Here,  $|\Tcal_S|$ was close to zero for Semeion.}  \label{fig:4}  
\end{figure}

\begin{figure}[t!]
\center  
\begin{subfigure}[b]{0.4\textwidth}
  \includegraphics[width=\columnwidth, height=0.6\textwidth]{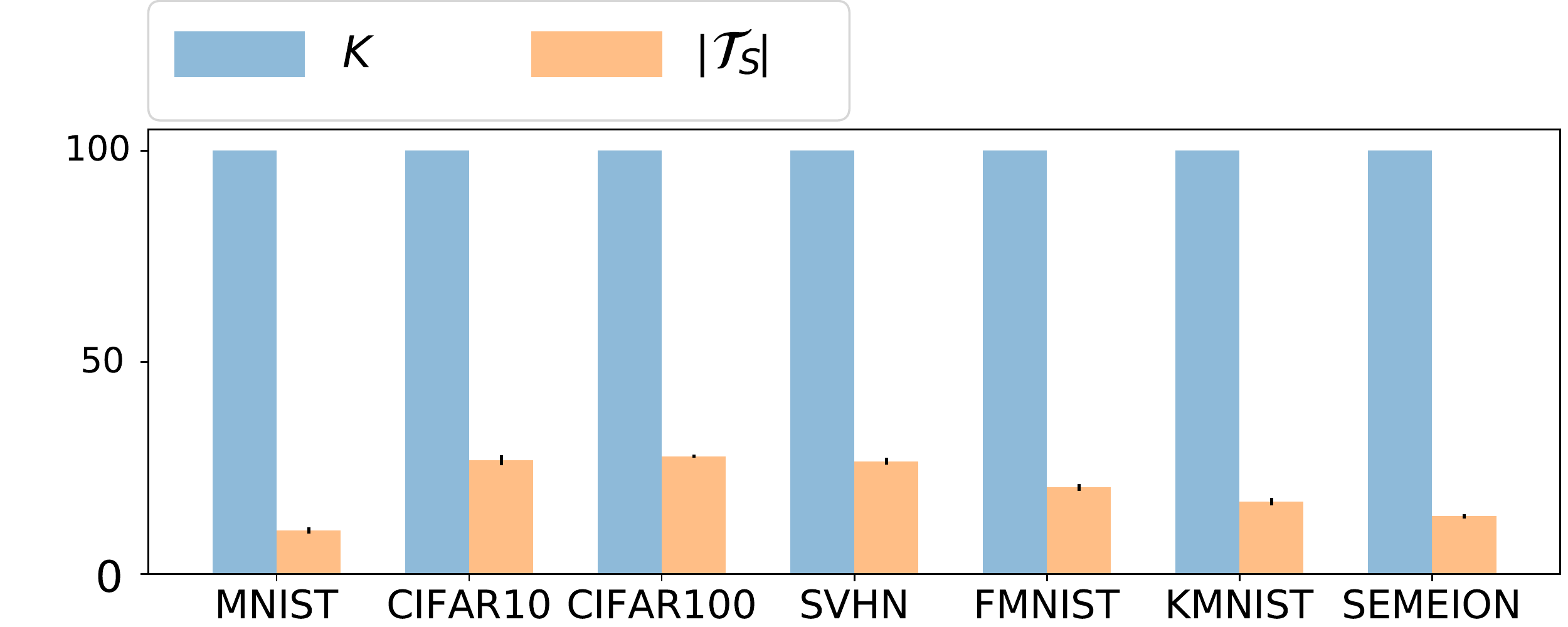}
  \vspace{-21pt}
\end{subfigure} 
    \caption{ The values of $K$ versus $|\Tcal_S|$ with real-world data and random projection. With random projection to reduce $K$, we still have $|\Tcal_S|<30< K=100< n\approx 60,000$ for the  real-life datasets. Here,  $n$ is the full train data size of each dataset: e.g., $n=60,000$ for MNIST.} \label{fig:5}
\end{figure}

\section{Conclusion} 
Since its introduction in 2010, algorithmic robustness has been a popular approach for analyzing learning algorithms \cite{xu2010robustness,xu2012robustness, bellet2015robustness, jolliffe2016principal}. In the original manuscript, which initiated the study of algorithmic robustness, \citet{xu2010robustness,xu2012robustness} pointed out that one disadvantage of their method is the dependence of the  bound on the covering number of the sample space. To the community, they posed the open problem of finding a mechanism to improve this dependence. Despite the popularity and several unsuccessful attempts, no significant progress has been made in this regard \citep{qian2020concentration, agrawal2017posterior, agrawal2017correction}.

        In this study, we provide tighter bounds for algorithmic robustness and general multinomial distributions. Our results establish natural and easily verified conditions in which the dependence of $K$ can be greatly reduced. Additionally, we demonstrate that the expected loss can be controlled by examining the single hypothesis returned by an algorithm, whereas in \citep{xu2012robustness}, the entire hypothesis space has to be analyzed. This is a considerable gain against several common loss functions. 
        
        Our bound is both practical and effective in the machine learning setting, as it depends only on the training samples. Furthermore, we provided theoretical and numerical examples in which our bounds proved superior to those of \citet{xu2012robustness} and those that follow from uniform stability \citep{bousquet2002stability}. Our experimental simulations show that on common datasets and popular theoretical models, this bound is exponentially better than the algorithmic robustness bound \citep{xu2012robustness}. These improvements to the foundations of algorithmic robustness have immediate impacts on applications ranging from metric learning to invariant classifiers. 

The main limitation of our approach is that we cannot know the values of the bounds until specifying training data; i.e.,  $|\Tcal_S|$ in our bound is data-dependent, whereas $K$ in the previous bound is data-independent. This data-dependence might not be preferable in some applications, where we may want to compute a bound before seeing training data.

        \section*{Acknowledgements}
        The research of J.H. is supported by the
Simons Foundation as a Junior Fellow at the Simons Society of Fellows, and NSF grant DMS-2054835. The research of Z.D. is supported by the Sloan Foundation grants, the NSF grant 1763665, and the Simons Foundation Collaboration on the Theory of Algorithmic Fairness.

        \bibliography{all}

\newcommand{\noopsort}[1]{} \newcommand{\printfirst}[2]{#1}
  \newcommand{\singleletter}[1]{#1} \newcommand{\switchargs}[2]{#2#1}
\begin{thebibliography}{60}
\providecommand{\natexlab}[1]{#1}
\providecommand{\url}[1]{\texttt{#1}}
\expandafter\ifx\csname urlstyle\endcsname\relax
  \providecommand{\doi}[1]{doi: #1}\else
  \providecommand{\doi}{doi: \begingroup \urlstyle{rm}\Url}\fi

\bibitem[Agrawal \& Jia(2017)Agrawal and Jia]{agrawal2017posterior}
Agrawal, S. and Jia, R.
\newblock Optimistic posterior sampling for reinforcement learning: worst-case
  regret bounds.
\newblock \emph{Advances in Neural Information Processing Systems (NeurIPS)},
  30, 2017.

\bibitem[Agrawal \& Jia(2020)Agrawal and Jia]{agrawal2017correction}
Agrawal, S. and Jia, R.
\newblock Posterior sampling for reinforcement learning: worst-case regret
  bounds.
\newblock \emph{arXiv update as a correction of the NeurIPS 2017 paper of the
  same authors}, 2020.

\bibitem[Arora et~al.(2019)Arora, Du, Hu, Li, and Wang]{arora2019fine}
Arora, S., Du, S., Hu, W., Li, Z., and Wang, R.
\newblock Fine-grained analysis of optimization and generalization for
  overparameterized two-layer neural networks.
\newblock In \emph{International Conference on Machine Learning}, pp.\
  322--332. PMLR, 2019.

\bibitem[Bartlett \& Mendelson(2002)Bartlett and
  Mendelson]{bartlett2002rademacher}
Bartlett, P.~L. and Mendelson, S.
\newblock Rademacher and gaussian complexities: Risk bounds and structural
  results.
\newblock \emph{Journal of Machine Learning Research}, 3\penalty0
  (Nov):\penalty0 463--482, 2002.

\bibitem[Bellet \& Habrard(2015)Bellet and Habrard]{bellet2015robustness}
Bellet, A. and Habrard, A.
\newblock Robustness and generalization for metric learning.
\newblock \emph{Neurocomputing}, 151:\penalty0 259--267, 2015.

\bibitem[Ben-Tal \& Nemirovski(1998)Ben-Tal and Nemirovski]{ben1998robust}
Ben-Tal, A. and Nemirovski, A.
\newblock Robust convex optimization.
\newblock \emph{Mathematics of operations research}, 23\penalty0 (4):\penalty0
  769--805, 1998.

\bibitem[Bertsimas et~al.(2011)Bertsimas, Brown, and
  Caramanis]{bertsimas2011theory}
Bertsimas, D., Brown, D.~B., and Caramanis, C.
\newblock Theory and applications of robust optimization.
\newblock \emph{SIAM review}, 53\penalty0 (3):\penalty0 464--501, 2011.

\bibitem[Bhattacharyya et~al.(2004)Bhattacharyya, Pannagadatta, and
  Smola]{bhattacharyya2004second}
Bhattacharyya, C., Pannagadatta, K., and Smola, A.~J.
\newblock A second order cone programming formulation for classifying missing
  data.
\newblock In \emph{Proceedings of the 17th International Conference on Neural
  Information Processing Systems}, pp.\  153--160, 2004.

\bibitem[Borwein \& Chan(2009)Borwein and Chan]{borwein2009uniform}
Borwein, J. and Chan, O.-Y.
\newblock Uniform bounds for the incomplete complementary gamma function.
\newblock \emph{Mathematical Inequalities and Applications}, 12:\penalty0
  115--121, 2009.

\bibitem[Bousquet \& Elisseeff(2002)Bousquet and
  Elisseeff]{bousquet2002stability}
Bousquet, O. and Elisseeff, A.
\newblock Stability and generalization.
\newblock \emph{Journal of Machine Learning Research}, 2\penalty0
  (Mar):\penalty0 499--526, 2002.

\bibitem[Chen et~al.(2022)Chen, Feng, Dai, Bai, Jiang, Xia, and
  Wang]{chen2022adversarial}
Chen, B., Feng, Y., Dai, T., Bai, J., Jiang, Y., Xia, S.-T., and Wang, X.
\newblock Adversarial examples generation for deep product quantization
  networks on image retrieval.
\newblock \emph{IEEE Transactions on Pattern Analysis and Machine
  Intelligence}, 2022.

\bibitem[Cisse et~al.(2017)Cisse, Bojanowski, Grave, Dauphin, and
  Usunier]{cisse2017parseval}
Cisse, M., Bojanowski, P., Grave, E., Dauphin, Y., and Usunier, N.
\newblock Parseval networks: Improving robustness to adversarial examples.
\newblock In \emph{International Conference on Machine Learning}, pp.\
  854--863. PMLR, 2017.

\bibitem[Clanuwat et~al.(2019)Clanuwat, Bober-Irizar, Kitamoto, Lamb, Yamamoto,
  and Ha]{clanuwat2019deep}
Clanuwat, T., Bober-Irizar, M., Kitamoto, A., Lamb, A., Yamamoto, K., and Ha,
  D.
\newblock Deep learning for classical japanese literature.
\newblock In \emph{NeurIPS Creativity Workshop 2019}, 2019.

\bibitem[Deng et~al.(2021{\natexlab{a}})Deng, He, and Su]{deng2021toward}
Deng, Z., He, H., and Su, W.
\newblock Toward better generalization bounds with locally elastic stability.
\newblock In \emph{International Conference on Machine Learning}, pp.\
  2590--2600. PMLR, 2021{\natexlab{a}}.

\bibitem[Deng et~al.(2021{\natexlab{b}})Deng, Zhang, Vodrahalli, Kawaguchi, and
  Zou]{deng2021adversarial}
Deng, Z., Zhang, L., Vodrahalli, K., Kawaguchi, K., and Zou, J.~Y.
\newblock Adversarial training helps transfer learning via better
  representations.
\newblock \emph{Advances in Neural Information Processing Systems}, 34,
  2021{\natexlab{b}}.

\bibitem[Devroye(1983)]{devroye1983equivalence}
Devroye, L.
\newblock The equivalence of weak, strong and complete convergence in l1 for
  kernel density estimates.
\newblock \emph{The Annals of Statistics}, pp.\  896--904, 1983.

\bibitem[Devroye et~al.(2013)Devroye, Gy{\"o}rfi, and
  Lugosi]{devroye2013probabilistic}
Devroye, L., Gy{\"o}rfi, L., and Lugosi, G.
\newblock \emph{A probabilistic theory of pattern recognition}, volume~31.
\newblock Springer Science \& Business Media, 2013.

\bibitem[Ding et~al.(2015)Ding, Xu, and Tao]{ding2015multi}
Ding, C., Xu, C., and Tao, D.
\newblock Multi-task pose-invariant face recognition.
\newblock \emph{IEEE Transactions on Image Processing}, 24\penalty0
  (3):\penalty0 980--993, 2015.

\bibitem[Gabrel et~al.(2014)Gabrel, Murat, and Thiele]{gabrel2014recent}
Gabrel, V., Murat, C., and Thiele, A.
\newblock Recent advances in robust optimization: An overview.
\newblock \emph{European journal of operational research}, 235\penalty0
  (3):\penalty0 471--483, 2014.

\bibitem[Globerson \& Roweis(2006)Globerson and Roweis]{globerson2006nightmare}
Globerson, A. and Roweis, S.
\newblock Nightmare at test time: robust learning by feature deletion.
\newblock In \emph{Proceedings of the 23rd international conference on Machine
  learning}, pp.\  353--360, 2006.

\bibitem[Gouk et~al.(2021)Gouk, Frank, Pfahringer, and
  Cree]{gouk2021regularisation}
Gouk, H., Frank, E., Pfahringer, B., and Cree, M.~J.
\newblock Regularisation of neural networks by enforcing lipschitz continuity.
\newblock \emph{Machine Learning}, 110\penalty0 (2):\penalty0 393--416, 2021.

\bibitem[Hastie et~al.(2019)Hastie, Tibshirani, and
  Wainwright]{hastie2019statistical}
Hastie, T., Tibshirani, R., and Wainwright, M.
\newblock \emph{Statistical learning with sparsity: the lasso and
  generalizations}.
\newblock Chapman and Hall/CRC, 2019.

\bibitem[Hu et~al.(2021)Hu, Jagtap, Karniadakis, and Kawaguchi]{hu2021extended}
Hu, Z., Jagtap, A.~D., Karniadakis, G.~E., and Kawaguchi, K.
\newblock When do extended physics-informed neural networks (xpinns) improve
  generalization?
\newblock \emph{arXiv preprint arXiv:2109.09444}, 2021.

\bibitem[Jia et~al.(2019)Jia, Li, Wen, Liu, and Tao]{jia2019orthogonal}
Jia, K., Li, S., Wen, Y., Liu, T., and Tao, D.
\newblock Orthogonal deep neural networks.
\newblock \emph{arXiv preprint arXiv:1905.05929}, 2019.

\bibitem[Jolliffe \& Cadima(2016)Jolliffe and Cadima]{jolliffe2016principal}
Jolliffe, I.~T. and Cadima, J.
\newblock Principal component analysis: a review and recent developments.
\newblock \emph{Philosophical Transactions of the Royal Society A:
  Mathematical, Physical and Engineering Sciences}, 374\penalty0
  (2065):\penalty0 20150202, 2016.

\bibitem[Kawaguchi \& Huang(2019)Kawaguchi and Huang]{kawaguchi2019gradient}
Kawaguchi, K. and Huang, J.
\newblock Gradient descent finds global minima for generalizable deep neural
  networks of practical sizes.
\newblock In \emph{2019 57th Annual Allerton Conference on Communication,
  Control, and Computing (Allerton)}, pp.\  92--99. IEEE, 2019.

\bibitem[Kawaguchi et~al.(2017)Kawaguchi, Kaelbling, and
  Bengio]{kawaguchi2017generalization}
Kawaguchi, K., Kaelbling, L.~P., and Bengio, Y.
\newblock Generalization in deep learning.
\newblock \emph{arXiv preprint arXiv:1710.05468}, 2017.

\bibitem[Krizhevsky \& Hinton(2009)Krizhevsky and
  Hinton]{krizhevsky2009learning}
Krizhevsky, A. and Hinton, G.
\newblock Learning multiple layers of features from tiny images.
\newblock Technical report, Citeseer, 2009.

\bibitem[LeCun et~al.(1998)LeCun, Bottou, Bengio, and
  Haffner]{lecun1998gradient}
LeCun, Y., Bottou, L., Bengio, Y., and Haffner, P.
\newblock Gradient-based learning applied to document recognition.
\newblock \emph{Proceedings of the IEEE}, 86\penalty0 (11):\penalty0
  2278--2324, 1998.

\bibitem[Liu et~al.(2021)Liu, Lamb, Kawaguchi, ALIAS PARTH~GOYAL, Sun, Mozer,
  and Bengio]{liu2021discrete}
Liu, D., Lamb, A.~M., Kawaguchi, K., ALIAS PARTH~GOYAL, A.~G., Sun, C., Mozer,
  M.~C., and Bengio, Y.
\newblock Discrete-valued neural communication.
\newblock \emph{Advances in Neural Information Processing Systems}, 34, 2021.

\bibitem[Liu et~al.(2017)Liu, Wu, Liu, Tao, and Fu]{liu2017spectral}
Liu, H., Wu, J., Liu, T., Tao, D., and Fu, Y.
\newblock Spectral ensemble clustering via weighted k-means: Theoretical and
  practical evidence.
\newblock \emph{IEEE transactions on knowledge and data engineering},
  29\penalty0 (5):\penalty0 1129--1143, 2017.

\bibitem[Luo et~al.(2015)Luo, Liu, Tao, and Xu]{luo2015multiview}
Luo, Y., Liu, T., Tao, D., and Xu, C.
\newblock Multiview matrix completion for multilabel image classification.
\newblock \emph{IEEE Transactions on Image Processing}, 24\penalty0
  (8):\penalty0 2355--2368, 2015.

\bibitem[Natalini \& Palumbo(2000)Natalini and
  Palumbo]{natalini2000inequalities}
Natalini, P. and Palumbo, B.
\newblock Inequalities for the incomplete gamma function.
\newblock \emph{Math. Inequal. Appl}, 3\penalty0 (1):\penalty0 69--77, 2000.

\bibitem[Netzer et~al.(2011)Netzer, Wang, Coates, Bissacco, Wu, and
  Ng]{netzer2011reading}
Netzer, Y., Wang, T., Coates, A., Bissacco, A., Wu, B., and Ng, A.~Y.
\newblock Reading digits in natural images with unsupervised feature learning.
\newblock In \emph{NIPS workshop on deep learning and unsupervised feature
  learning}, 2011.

\bibitem[Pedraza et~al.(2022)Pedraza, Deniz, and Bueno]{pedraza2022lyapunov}
Pedraza, A., Deniz, O., and Bueno, G.
\newblock Lyapunov stability for detecting adversarial image examples.
\newblock \emph{Chaos, Solitons \& Fractals}, 155:\penalty0 111745, 2022.

\bibitem[Pham et~al.(2021)Pham, Dai, Ghiasi, Kawaguchi, Liu, Yu, Yu, Chen,
  Luong, Wu, Tan, and Le]{pham2021combined}
Pham, H., Dai, Z., Ghiasi, G., Kawaguchi, K., Liu, H., Yu, A.~W., Yu, J., Chen,
  Y.-T., Luong, M.-T., Wu, Y., Tan, M., and Le, Q.~V.
\newblock Combined scaling for open-vocabulary image classification.
\newblock \emph{arXiv preprint arXiv:2111.10050}, 2021.
\newblock \doi{10.48550/arXiv.2111.10050}.
\newblock URL \url{https://arxiv.org/abs/2111.10050}.

\bibitem[Qi et~al.(2013)Qi, Tian, and Shi]{qi2013robust}
Qi, Z., Tian, Y., and Shi, Y.
\newblock Robust twin support vector machine for pattern classification.
\newblock \emph{Pattern Recognition}, 46\penalty0 (1):\penalty0 305--316, 2013.

\bibitem[Qian et~al.(2020)Qian, Fruit, Pirotta, and
  Lazaric]{qian2020concentration}
Qian, J., Fruit, R., Pirotta, M., and Lazaric, A.
\newblock Concentration inequalities for multinoulli random variables.
\newblock \emph{arXiv preprint arXiv:2001.11595}, 2020.

\bibitem[Redko et~al.(2020)Redko, Morvant, Habrard, Sebban, and
  Bennani]{redko2020survey}
Redko, I., Morvant, E., Habrard, A., Sebban, M., and Bennani, Y.
\newblock A survey on domain adaptation theory: learning bounds and theoretical
  guarantees.
\newblock \emph{arXiv preprint arXiv:2004.11829}, 2020.

\bibitem[Rice et~al.(2021)Rice, Bair, Zhang, and Kolter]{rice2021robustness}
Rice, L., Bair, A., Zhang, H., and Kolter, J.~Z.
\newblock Robustness between the worst and average case.
\newblock \emph{Advances in Neural Information Processing Systems}, 34, 2021.

\bibitem[Robey et~al.(2021)Robey, Chamon, Pappas, Hassani, and
  Ribeiro]{robey2021adversarial}
Robey, A., Chamon, L., Pappas, G.~J., Hassani, H., and Ribeiro, A.
\newblock Adversarial robustness with semi-infinite constrained learning.
\newblock \emph{Advances in Neural Information Processing Systems},
  34:\penalty0 6198--6215, 2021.

\bibitem[Sener \& Savarese(2017)Sener and Savarese]{sener2017active}
Sener, O. and Savarese, S.
\newblock Active learning for convolutional neural networks: A core-set
  approach.
\newblock \emph{arXiv preprint arXiv:1708.00489}, 2017.

\bibitem[Shen et~al.(2020)Shen, Cheng, and Liang]{shen2020deep}
Shen, X., Cheng, X., and Liang, K.
\newblock Deep euler method: solving odes by approximating the local truncation
  error of the euler method.
\newblock \emph{arXiv preprint arXiv:2003.09573}, 2020.

\bibitem[Shi et~al.(2014)Shi, Bellet, and Sha]{shi2014sparse}
Shi, Y., Bellet, A., and Sha, F.
\newblock Sparse compositional metric learning.
\newblock In \emph{Proceedings of the AAAI Conference on Artificial
  Intelligence}, volume~28, 2014.

\bibitem[Sokolic et~al.(2017{\natexlab{a}})Sokolic, Giryes, Sapiro, and
  Rodrigues]{sokolic2017generalization}
Sokolic, J., Giryes, R., Sapiro, G., and Rodrigues, M.
\newblock Generalization error of invariant classifiers.
\newblock In \emph{Artificial Intelligence and Statistics}, pp.\  1094--1103,
  2017{\natexlab{a}}.

\bibitem[Sokolic et~al.(2017{\natexlab{b}})Sokolic, Giryes, Sapiro, and
  Rodrigues]{sokolic2017robust}
Sokolic, J., Giryes, R., Sapiro, G., and Rodrigues, M.~R.
\newblock Robust large margin deep neural networks.
\newblock \emph{IEEE Transactions on Signal Processing}, 2017{\natexlab{b}}.

\bibitem[Srl \& Brescia(1994)Srl and Brescia]{srl1994semeion}
Srl, B.~T. and Brescia, I.
\newblock Semeion handwritten digit data set.
\newblock \emph{Semeion Research Center of Sciences of Communication, Rome,
  Italy}, 1994.

\bibitem[Tao et~al.(2016)Tao, Guo, Song, Li, Yu, and Tang]{tao2016person}
Tao, D., Guo, Y., Song, M., Li, Y., Yu, Z., and Tang, Y.~Y.
\newblock Person re-identification by dual-regularized kiss metric learning.
\newblock \emph{IEEE Transactions on Image Processing}, 25\penalty0
  (6):\penalty0 2726--2738, 2016.

\bibitem[Van Der~Vaart et~al.(1996)Van Der~Vaart, van~der Vaart, van~der Vaart,
  and Wellner]{van1996weak}
Van Der~Vaart, A.~W., van~der Vaart, A.~W., van~der Vaart, A., and Wellner, J.
\newblock \emph{Weak convergence and empirical processes: with applications to
  statistics}.
\newblock Springer Science \& Business Media, 1996.

\bibitem[Vapnik(1998)]{vapnik1998statistical}
Vapnik, V.
\newblock \emph{Statistical learning theory}, volume~1.
\newblock Wiley New York, 1998.

\bibitem[Vershynin(2018)]{vershynin2018high}
Vershynin, R.
\newblock \emph{High-dimensional probability: An introduction with applications
  in data science}, volume~47.
\newblock Cambridge university press, 2018.

\bibitem[Weissman et~al.(2003)Weissman, Ordentlich, Seroussi, Verdu, and
  Weinberger]{weissman2003inequalities}
Weissman, T., Ordentlich, E., Seroussi, G., Verdu, S., and Weinberger, M.~J.
\newblock {Inequalities for the L1 deviation of the empirical distribution}.
\newblock \emph{Hewlett-Packard Labs, Tech. Rep}, 2003.

\bibitem[Wellner et~al.(2013)]{wellner2013weak}
Wellner, J. et~al.
\newblock \emph{Weak convergence and empirical processes: with applications to
  statistics}.
\newblock Springer Science \& Business Media, 2013.

\bibitem[Xiao et~al.(2017)Xiao, Rasul, and Vollgraf]{xiao2017fashion}
Xiao, H., Rasul, K., and Vollgraf, R.
\newblock Fashion-mnist: a novel image dataset for benchmarking machine
  learning algorithms.
\newblock \emph{arXiv preprint arXiv:1708.07747}, 2017.

\bibitem[Xu \& Mannor(2010)Xu and Mannor]{xu2010robustness}
Xu, H. and Mannor, S.
\newblock Robustness and generalization.
\newblock In \emph{Conference on Learning Theory (COLT)}, 2010.

\bibitem[Xu \& Mannor(2012)Xu and Mannor]{xu2012robustness}
Xu, H. and Mannor, S.
\newblock Robustness and generalization.
\newblock \emph{Machine learning}, 86\penalty0 (3):\penalty0 391--423, 2012.

\bibitem[Xu et~al.(2009)Xu, Caramanis, and Mannor]{xu2009robustness}
Xu, H., Caramanis, C., and Mannor, S.
\newblock Robustness and regularization of support vector machines.
\newblock \emph{Journal of machine learning research}, 10\penalty0 (7), 2009.

\bibitem[Zahavy et~al.(2016)Zahavy, Kang, Sivak, Feng, Xu, and
  Mannor]{zahavy2016ensemble}
Zahavy, T., Kang, B., Sivak, A., Feng, J., Xu, H., and Mannor, S.
\newblock Ensemble robustness and generalization of stochastic deep learning
  algorithms.
\newblock \emph{arXiv preprint arXiv:1602.02389}, 2016.

\bibitem[Zhang et~al.(2021{\natexlab{a}})Zhang, Bengio, Hardt, Recht, and
  Vinyals]{zhang2021understanding}
Zhang, C., Bengio, S., Hardt, M., Recht, B., and Vinyals, O.
\newblock Understanding deep learning (still) requires rethinking
  generalization.
\newblock \emph{Communications of the ACM}, 64\penalty0 (3):\penalty0 107--115,
  2021{\natexlab{a}}.

\bibitem[Zhang et~al.(2021{\natexlab{b}})Zhang, Deng, and
  Kawaguchi]{zhang2021does}
Zhang, L., Deng, Z., and Kawaguchi, K.
\newblock How does mixup help with robustness and generalization?
\newblock In \emph{International Conference on Learning Representations
  (ICLR)}, 2021{\natexlab{b}}.

\end{thebibliography}
        \bibliographystyle{icml2022}

\newpage
\appendix
\onecolumn
        \allowdisplaybreaks

        \section{Multinomial Concentration Bounds with Probability Distribution Dependence} \label{appendix:fixeda}
                Let the vector $X=(X_1,\dots,X_K)$ follow the multinomial distribution with parameters $m$ and $p=(p_1,\dots,p_K)$. We wish to upper bound the following quantity:
        \begin{align}
                \sum_{i=1}^K a_{i}(X) \left(p_i - \frac{X_i}{m}\right)  ,
        \end{align}
        where $a_i(X)\ge 0$ for all $i \in \{1,\dots, K\}$. Recall that $a_i(X)$ depend on $X_1,\dots,X_K$, which is the heart of the problem.
        
        In this section, we first establish concentration results for scaled multinomial random variables, in other words with $a_i(X)$ fixed independent of $X$.  This is the first step towards dealing with the dependent coefficients in the above equation.
After this first step, we will analyze the quantity with the dependent $a_i$.        
        \subsection{Multinomial Distribution with fixed $a$}

        In this setting, we obtain the following sharp lower and upper bounds.
        
        \begin{lemma} \label{lemma:3}
                Let $\ba_1,\dots,\ba_K \ge 0$ be fixed such that $ \sum_{i=1}^K \ba_i p_i\neq 0$. Then, for any $M>0$,
                $$
                \PP\left(\sum_{i=1}^K \ba_{i} \left(p_i - \frac{X_i}{m}\right) <-M   \right) \le \exp\left(-\frac{mM}{2 \ba} \min\left\{1, \frac{\ba M}{\beta}\right\}\right)
                $$ 
                where $\ba:= \max_{i \in [K]} \ba_i$ and $\beta:= 2\sum_{i=1}^K \ba^2_i p_i$.
        \end{lemma}
        \begin{proof} 
                Note that  $\PP\left(\sum_{i=1}^K \ba_{i} \left(p_i - \frac{X_i}{m}\right) <-M \right)=\PP\left(\sum_{i=1}^K \ba_{i} \left( \frac{X_i}{m}-p_i \right)>M \right)$. By Markov's inequality, the following holds for any $\nu>0$: 
                \begin{align} \label{eq:5}
                        \nonumber \PP\left(\sum_{i=1}^K \ba_{i} \left( \frac{X_i}{m}-p_i\right) >M   \right) &=\PP\left(\nu\sum_{i=1}^K \ba_{i}\frac{X_i}{m}  >\nu\left(M +\sum_{i=1}^K \ba _{i}p_{i}  \right)  \right) 
                        \\ \nonumber & =\PP\left(e^{\nu\sum_{i=1}^K \ba_{i}\frac{X_i}{m}}  >e^{\nu\left(M +\sum_{i=1}^K \ba _{i}p_{i}  \right)}  \right)
                        \\ \nonumber & \le\PP\left(e^{\nu\sum_{i=1}^K \frac{\ba_{i}X_i}{m}}  \ge e^{\nu\left(M +\sum_{i=1}^K \ba _{i}p_{i}\right)}  \right)
                        \\ & \le e^{-\nu\left(M +\sum_{i=1}^K \ba _{i}p_{i}  \right)} \EE[e^{\nu\sum_{i=1}^K \frac{\ba_{i}X_i}{m}}].  
                \end{align}
                To evaluate the moment generating function, we use the probability mass function of the multinomial distribution and the multinomial theorem to find that 
                \begin{align} \label{eq:10}
                        \nonumber \EE[e^{\nu\sum_{i=1}^K \frac{\ba_{i}X_i}{m}}]&=\sum \frac{m!}{X_1!\dots X_K!} p_1^{X_1}p_2^{X_2}\dots p_K^{X_K} e^{\nu\sum_{i=1}^K \frac{\ba_{i}X_i}{m}}
                        \\ \nonumber & =\sum \frac{m!}{X_1!\dots X_K!} (p_1^{}e^{ \frac{\nu\ba_{1}}{m}})^{X_1}(p_2^{}e^{ \frac{\nu\ba_{2}}{m}})^{X_2}\dots (p_K^{}e^{ \frac{\nu\ba_{K}}{m}})^{X_K} \\ & =(\sum_{i=1}^K p_i e^{ \frac{\nu\ba_{i}}{m}})^m, 
                \end{align}
                where the sum in the first two lines is taken over all possible values of the random variable $(X_1,\dots,X_K)$ (i.e., this is  the sum in the definition of the expectation on the left-hand side of the equation). Pugging this into \eqref{eq:5}, we conclude that 
                \begin{align} \label{eq:6}
                        \PP\left(\sum_{i=1}^K \ba_{i} \left( \frac{X_i}{m}-p_i\right) >M   \right) & \le e^{-\nu\left(M +\sum_{i=1}^K \ba _{i}p_{i}  \right)} \left(\sum_{i=1}^K p_i e^{ \frac{\nu\ba_{i}}{m}}\right)^m.
                \end{align}
                We recall the following simple bounds for exponential functions:
                \begin{align}
                        \label{eq:7} & e^x \ge 1+x \qquad \ \ \  \text{for } x \ge 0
                        \\ \label{eq:8} & e^x \le 1+x+x^2 \ \ \text{for } 0 \le x \le 1 .
                \end{align}
                Then, for $0<\nu \le \frac{m}{\ba} \le \frac{m}{\ba_i}$, using \eqref{eq:7}--\eqref{eq:8}, we have that
                \begin{align*}
                        \sum_{i=1}^K p_i e^{ \frac{\nu\ba_{i}}{m}} \le \sum_{i=1}^K p_i \left(1+\frac{\nu\ba_{i}}{m} +\frac{\nu^{2}\ba_{i}^2}{m^{2}} \right) \le e^{ \sum_{i=1}^K p_i \left(\frac{\nu\ba_{i}}{m} +\frac{\nu^{2}\ba_{i}^2}{m^{2}} \right)}. 
                \end{align*}  
                Pugging this into \eqref{eq:6}, we deduce that
                \begin{align} \label{eq:9}
                        \nonumber \PP\left(\sum_{i=1}^K \ba_{i} \left( \frac{X_i}{m}-p_i\right) >M   \right) & \le e^{-\nu\left(M +\sum_{i=1}^K \ba _{i}p_{i}  \right)} \left(e^{ \sum_{i=1}^K p_i \left(\frac{\nu\ba_{i}}{m} +\frac{\nu^{2}\ba_{i}^2}{m^{2}} \right)}\right)^m
                        \\ \nonumber & = e^{-\nu\left(M +\sum_{i=1}^K \ba _{i}p_{i}  \right)}   e^{ \sum_{i=1}^K \nu p_i \ba_{i}  + \sum_{i=1}^K p_i\frac{\nu^{2}\ba_{i}^2}{m} }
                        \\ \nonumber & = e^{-\nu M+\sum_{i=1}^K p_i\frac{\nu^{2}\ba_{i}^2}{m}}   \\ & = e^{-\nu M+ \frac{\nu^2\beta}{2m}} 
                \end{align} 
                Here, we have $\beta:= 2\sum_{i=1}^K \ba^2_i p_i>0$ since $\sum_{i=1}^K \ba_i p_i\neq 0$ and $\ba_i,p_i\ge 0$ by assumption.
                
                We consider two possible cases. If $M \le \frac{\beta}{\ba}$, we take $0<\nu =\frac{mM}{\beta} \le \frac{m}{\ba}$, which yields in \eqref{eq:9}:
                $$
                \PP\left(\sum_{i=1}^K \ba_{i} \left( \frac{X_i}{m}-p_i\right) >M   \right) \le e^{- \frac{mM^{2}}{2\beta}}. 
                $$
                If $M \ge \frac{\beta}{\ba}$, we take $0<\nu = \frac{m}{\ba}$, which yields in \eqref{eq:9}:
                $$
                \PP\left(\sum_{i=1}^K \ba_{i} \left( \frac{X_i}{m}-p_i\right) >M   \right) \le e^{- \frac{m}{\ba}(M-\frac{\beta}{2\ba})}\le e^{-\frac{mM}{2\ba}}. 
                $$
                Combining these conclusions yields 
                $$
                \PP\left(\sum_{i=1}^K \ba_{i} \left( \frac{X_i}{m}-p_i\right) >M   \right) \le \exp\left(-\frac{mM}{2 \ba} \min\left\{1, \frac{\ba M}{\beta}\right\}\right).
                $$
        \end{proof}
        
        For the other tail, we establish the following estimate.
        
        \begin{lemma} \label{lemma:4}
                Let $\ba_1,\dots,\ba_K \ge 0$ be fixed such that $ \sum_{i=1}^K \ba_i p_i\neq 0$. Then, for any $M>0$,
                $$
                \PP\left(\sum_{i=1}^K \ba_{i} \left(p_i - \frac{X_i}{m}\right) >M   \right) \le \exp\left(-\frac{mM^{2}}{\beta}\right)
                $$ 
                where $\beta:= 2\sum_{i=1}^K \ba^2_i p_i$.
        \end{lemma}
        \begin{proof}
                By Markov's inequality, the following holds \textit{for any $\nu<0$:} 
                \begin{align*} 
                        \nonumber \PP\left(\sum_{i=1}^K \ba_{i} \left( p_i-\frac{X_i}{m}\right) >M   \right) &=\PP\left(\nu\sum_{i=1}^K \ba_{i}\frac{X_i}{m}  >\nu\left(\sum_{i=1}^K \ba _{i}p_{i} -M   \right)  \right) 
                        \\ \nonumber & =\PP\left(e^{\nu\sum_{i=1}^K \ba_{i}\frac{X_i}{m}}  >e^{\nu\left(\sum_{i=1}^K \ba _{i}p_{i}  -M \right)}  \right)
                        \\ \nonumber & \le\PP\left(e^{\nu\sum_{i=1}^K \frac{\ba_{i}X_i}{m}}  \ge e^{\nu\left(\sum_{i=1}^K \ba _{i}p_{i}-M\right)}  \right)
                        \\ & \le e^{-\nu\left(\sum_{i=1}^K \ba _{i}p_{i}  -     M\right)} \EE[e^{\nu\sum_{i=1}^K \frac{\ba_{i}X_i}{m}}]  
                \end{align*}
                Using \eqref{eq:10} from the proof of Lemma \ref{lemma:3}, we have  $\EE[e^{\nu\sum_{i=1}^K \frac{\ba_{i}X_i}{m}}]=(\sum_{i=1}^K p_i e^{ \frac{\nu\ba_{i}}{m}})^m$, yielding that
                $$
                \PP\left(\sum_{i=1}^K \ba_{i} \left( p_i-\frac{X_i}{m}\right) >M   \right) \le e^{-\nu\left(\sum_{i=1}^K \ba _{i}p_{i}  -     M\right)} (\sum_{i=1}^K p_i e^{ \frac{\nu\ba_{i}}{m}})^m.
                $$
                We recall the following bounds for exponential functions:
                \begin{align*}
                        & e^x \le 1+x+\frac{x^2}{2} \ \ \text{for any } x \le 0 
                        \\  &  e^x \ge 1+x \ \ \text{for any } x \in \RR  
                \end{align*}
                Using these bounds, 
                $$
                \sum_{i=1}^K p_i e^{ \frac{\nu\ba_{i}}{m}} \le \sum_{i=1}^K p_i (1+\frac{\nu \ba_i}{m} + \frac{\nu^2 \ba_i^2}{2m^{2}}) \le e^{\sum_{i=1}^K p_i (\frac{\nu \ba_i}{m}+\frac{\nu^2 \ba_i^2}{2m^{2}})}. 
                $$
                Combining what we have so far,
                \begin{align*}
                        \PP\left(\sum_{i=1}^K \ba_{i} \left( p_i-\frac{X_i}{m}\right) >M   \right) & \le e^{-\nu\left(\sum_{i=1}^K \ba _{i}p_{i}  -     M\right)}  (e^{\sum_{i=1}^K p_i (\frac{\nu \ba_i}{m}+\frac{\nu^2 \ba_i^2}{2m^{2}})})^m
                        \\ & =e^{\nu M+\sum_{i=1}^K p_i \frac{\nu^2 \ba_i^2}{2m} }
                        \\ & =e^{\nu M+\frac{\nu^2 \beta}{4m}}   
                \end{align*}
                Here, we have $\beta:= 2\sum_{i=1}^K \ba^2_i p_i>0$ since $\sum_{i=1}^K \ba_i p_i\neq 0$ and $\ba_i,p_i\ge 0$ by assumption.
                Finally, we set $\nu=-\frac{2m\nu}{\beta}(<0)$ to conclude that
                $$
                \PP\left(\sum_{i=1}^K \ba_{i} \left( p_i-\frac{X_i}{m}\right) >M   \right) \le e^{-\frac{mM^{2}}{\beta}}.
                $$
        \end{proof}

        \subsection{Multinomial Distribution without $a$}
        
        In this section, we establish some bounds for $a_i = 1$ which is of special interest in applications.
        
        \begin{lemma} \label{lemma:5}
                For any $\delta>0$, with probability at least $1-\delta$, the following holds for all $i \in \{1,\dots,K\}$:
                
                $$
                p_i \ge  \begin{cases} \frac{X_i}{m} - \sqrt{\frac{p_i\ln(K/\delta)}{m}} & \text{if $p_i >  \frac{\ln(K/\delta)}{4m}$}\\
                        \frac{X_i}{m} -
                        \frac{2\ln(K/\delta)}{m} & \text{if $p_i \le  \frac{\ln(K/\delta)}{4m}$} \\
                \end{cases}
                $$  
        \end{lemma}
        \begin{proof}  
                
                For each $i\in[K]$,  if $p_i=0$, then the desired statement holds vacuously, because $0=p_i\ge\frac{X_i}{m} -
                \frac{2\ln(K/\delta)}{m}$ where $\frac{X_i}{m}=0$ (since $p_i=0$) and $\frac{2\ln(K/\delta)}{m}\ge 0$.
                Thus, 
                for the remainder of the proof, we consider the case where $p_i\neq 0$. For each $i\in[K]$, we use Lemma \ref{lemma:3} with $\ba_i=1$ and $\ba_j=0$ for all $j \neq i$ (which satisfies $\sum_{i=1}^K \ba_i p_i\neq 0$ since $p_{i}\neq 0$), yielding  that for any $M>0$,
                $$
                \PP\left( p_i - \frac{X_i}{m} <-M   \right) \le \exp\left(-\frac{mM}{2} \min\left\{1, \frac{ M}{2{p_{i}}}\right\}\right).
                $$ 
                In other words, for any $M>0$, 
                $$
                \forall M\le2p_{i}, \ \PP\left(\frac{X_i}{m} - p_i > M\right) \le e^{-\frac{m M^{2}}{4p_i} } 
                $$
                and
                $$
                \forall M\ge 2p_{i}, \ \PP\left(\frac{X_i}{m} - p_i > M\right) \le e^{-\frac{m M}{2} } .
                $$
                We now consider the two cases on the value of  $p_i$ for an arbitrary $\delta>0$.

                \textbf{Case $p_i \ge  \frac{\ln(K/\delta)}{4m}$:}  in this case, we set $M=\sqrt{\frac{p_i\ln(K/\delta)}{m}}$. Then, we have that $M\le2p_{i}$ since the condition $p_i \ge  \frac{\ln(K/\delta)}{4m}$ implies that $4p_i ^{2}\ge  \frac{p_{i}\ln(K/\delta)}{m}$ and hence $2p_i ^{}\ge \sqrt{\frac{p_{i}\ln(K/\delta)}{m}}=M$. Therefore, if $p_i \ge  \frac{\ln(K/\delta)}{4m}$,
                $$
                 \ \PP\left(\frac{X_i}{m} - p_i > \sqrt{\frac{p_i\ln(K/\delta)}{m}}\right) \le \frac{\delta}{K}.
                $$ 
                \textbf{Case $p_i \le  \frac{\ln(K/\delta)}{4m}$:}  here, we set $M=\frac{2\ln(K/\delta)}{m}$. Then, we have that $M\ge2p_{i}$ since the condition $p_i \le  \frac{\ln(K/\delta)}{4m}$ implies that $2p_i ^{} \le  \frac{\ln(K/\delta)}{2m}\le\frac{2\ln(K/\delta)}{m}=M$. Thus, we have that if $p_i \le  \frac{\ln(K/\delta)}{4m}$,
                $$
                \PP\left(\frac{X_i}{m} - p_i >\frac{2\ln(K/\delta)}{m}\right) \le \frac{\delta}{K}.
                $$ 
                
                Taking union bounds over all $i \in \{1,\dots,K\}$, we have that for any $\delta>0$, with probability at least $1-\delta$, the following holds for all $i \in \{1,\dots,K\}$:
                $$
                \frac{X_i}{m} - p_i \le \begin{cases} \sqrt{\frac{p_i\ln(K/\delta)}{m}} & \text{if $p_i >  \frac{\ln(K/\delta)}{4m}$}\\
                        \frac{2\ln(K/\delta)}{m} & \text{if $p_i \le  \frac{\ln(K/\delta)}{4m}$} .\\
                \end{cases}
                $$
                In other words, we have that for any $\delta>0$, with probability at least $1-\delta$, the following holds for all $i \in \{1,\dots,K\}$:
                
                $$
                p_i \ge  \begin{cases} \frac{X_i}{m} - \sqrt{\frac{p_i\ln(K/\delta)}{m}} & \text{if $p_i >  \frac{\ln(K/\delta)}{4m}$}\\
                        \frac{X_i}{m} -
                        \frac{2\ln(K/\delta)}{m} & \text{if $p_i \le  \frac{\ln(K/\delta)}{4m}$} .\\
                \end{cases}
                $$  
        \end{proof}

        We have a simpler statement for the other tail.
        \begin{lemma} \label{lemma:6}
                For any $\delta>0$, with probability at least $1-\delta$, the following holds for all $i \in \{1,\dots,K\}$: 
                \begin{align} \label{eq:new:new:1}
                    p_i- \frac{X_i}{m} \le \sqrt{\frac{2p_i \ln(K/\delta)}{m}}.
                \end{align}
        \end{lemma}
        \begin{proof}  For each $i\in[K]$,  if $p_i=0$, then the desired statement holds trivially because $p_i- \frac{X_i}{m}=- \frac{X_i}{m} \le \sqrt{\frac{2p_i \ln(K/\delta)}{m}}$ where $\frac{X_i}{m}=0$ and $\sqrt{\frac{2p_i \ln(K/\delta)}{m}}=0 $.
                Thus, for
                the rest of the proof, we consider the case where $p_i\neq 0$. For each $i\in[K]$, we use Lemma \ref{lemma:4}  with $\ba_i=1$ and $\ba_j=0$ for all $j \neq i$ (which satisfies $\sum_{i=1}^K \ba_i p_i\neq 0$ since $p_{i}\neq 0$), yielding  that for any $M>0$,
                $$
                \PP\left( p_i - \frac{X_i}{m} >M   \right) \le \exp\left(-\frac{mM^{2}}{2p_{i}} \right).
                $$ 
                By setting $M= \sqrt{\frac{2 p_i\ln(K/\delta)}{m}}$, 
                $$
                \PP\left( p_i - \frac{X_i}{m} > \sqrt{\frac{2 p_i\ln(K/\delta)}{m}}   \right) \le \frac{\delta}{K}.
                $$
                We conclude the proof by taking a union bound over all $i \in [K]$.
        \end{proof}

        \subsection{Proof of Lemma \ref{lemma:multinomialnew}}
        
        The proof of Lemma \ref{lemma:multinomialnew} is obtained by combining the results thus far:
        
        \begin{proof}[Proof of Lemma \ref{lemma:multinomialnew}]  Lemma \ref{lemma:6} implies Lemma \ref{lemma:multinomialnew} by summing up both sides of \eqref{eq:new:new:1} with the coefficients $a_i(X)\ge 0$. That is, since $a_i(X)\ge 0$, Lemma \ref{lemma:6} implies that for any $\delta>0$, with probability at least $1-\delta$, the following holds for all $i \in \{1,\dots,K\}$: 
                \begin{align*}
                    a_i(X) \left(p_i- \frac{X_i}{m}\right) \le a_i(X) \sqrt{\frac{2p_i \ln(K/\delta)}{m}}.
                \end{align*}
        This then implies that for any $\delta>0$, with probability at least $1-\delta$, 
                \begin{align*}
                    \sum_{k=1}^K a_i(X) \left(p_i- \frac{X_i}{m}\right) \le \left(\sum_{k=1}^K a_i(X) \sqrt{p_i} \right) \sqrt{\frac{2 \ln(K/\delta)}{m}}.
                \end{align*}
        \end{proof}

        \section{Multinomial Concentration Bounds with Data Dependence} \label{appendix:dependenta}
        Recall that  $\Tcal_{S}:=\{k\in[K]: |\Ical_{k}^{S}|\ge 1\}$ with $\Ical_{k}^S:=\{i\in[n]: z_{i}\in \Ccal_{k}\}$. Additionally, we defined  $a_{\Tcal_S}(X)=\max_{i\in \Tcal_S} a_i(X)$ and $a_{\Tcal_S^c}(X)=\max_{i\in \Tcal^c_S} a_i(X)$ where $\Tcal^c_S = [K]\setminus \Tcal_S$. 
        
        Let the vector $X=(X_1,\dots,X_K)$ follows the multinomial distribution with parameter $n$ and $p=(p_1,\dots,p_K)$, where $p_k=\PP(z\in \Ccal_{k})$ and $X_k=\sum_{i=1}^n \one\{z_i\in \Ccal_{k}\}$. We want to upper bound the following quantity:
        \begin{align}
                \sum_{i=1}^K a_{i}(X) \left(p_i - \frac{X_i}{n}\right)  ,
        \end{align}
        where $a_i(X)\ge 0$ for all $i \in \{1,\dots, K\}$. Here, the $a_i(X)$ depend on $X_1,\dots,X_K$.

        \subsection{Basic Version}
        In this subsection, using our new results from  Appendix \ref{appendix:fixeda}, we prove Theorem \ref{lem:weightedmultinomial}, which is restated as Lemma \ref{lemma:7} in the following and further refined in Appendix \ref{app:1}:
        \begin{lemma} \label{lemma:7}
                For any $\delta>0$, with probability at least $1-\delta$, the following holds:
                $$
                \sum_{i=1}^K a_i(X) \left(p_i - \frac{X_i}{n}\right) \le (a_{\Tcal_S}(X)\sqrt{2} +a_{\Tcal_S^c}(X) ) \sqrt{\frac{| \Tcal_S| \ln(2K/\delta)}{n}} + a_{\Tcal_S^c}(X) \frac{2| \Tcal_S|\ln(2K/\delta)}{n}.
                $$
        \end{lemma}
        \begin{proof}
                By using Lemma \ref{lemma:5} and Lemma \ref{lemma:6} and takeing union bounds for the two,   we have that for any $\delta>0$, with probability at least $1-\delta$, the following holds for all $i \in [K]$, 
                \begin{align} \label{eq:11}
                        p_i- \frac{X_i}{n} \le \sqrt{\frac{2p_i \ln(2K/\delta)}{n}}
                \end{align}
                \textit{and}  
                \begin{align} \label{eq:12}
                        p_i \ge  \begin{cases} \frac{X_i}{n} - \sqrt{\frac{p_i\ln(2K/\delta)}{n}} & \text{if $p_i >  \frac{\ln(2K/\delta)}{4n}$}\\
                                \frac{X_i}{n} -
                                \frac{2\ln(2K/\delta)}{n} & \text{if $p_i \le  \frac{\ln(2K/\delta)}{4n}$} \\
                        \end{cases}
                \end{align}
                Recall that  $\Tcal_{S}:=\{k\in[K]: |\Ical_{k}^{S}|\ge 1\}$ with $\Ical_{k}^S:=\{i\in[n]: z_{i}\in \Ccal_{k}\}$. Summing up both sides of \eqref{eq:12} over $i\in\Tcal_{S}$ and using $\sum_{i \in\Tcal_{S}}\frac{X_i}{n}=1$,
                $$
                \sum_{i \in \Tcal_S} p_i \ge  1- \sum_{i \in I_{1}}   \sqrt{\frac{p_i\ln(2K/\delta)}{n}} -\sum_{i \in I_{2}} \frac{2\ln(2K/\delta)}{n}=1- \sum_{i \in I_{1}}   \sqrt{\frac{p_i\ln(2K/\delta)}{n}} -(|  \Tcal_{S}|-|I_{1}|) \frac{2\ln(2K/\delta)}{n}
                $$ 
                where $I_1 = \{i\in\Tcal_{S} :p_i >  \frac{\ln(K/\delta)}{4n} \}$ and $I_2= \{i\in  \Tcal_{S} :p_i \le   \frac{\ln(K/\delta)}{4n} \}$. This implies that
                $$
                1-\sum_{i\in\Tcal_S} p_i  \le  \sum_{i \in I_{1}}   \sqrt{\frac{p_i\ln(2K/\delta)}{n}} +(|  \Tcal_{S}|-| I_{1}|) \frac{2\ln(2K/\delta)}{n}.
                $$
                Since $\sum_{i=1}^K p_i=1$ and hence $\sum_{i\notin \Tcal_S} p_i=1-\sum_{i\in \Tcal_S} p_i $,
                this implies that 
                \begin{align} \label{eq:13}
                        \sum_{i\notin \Tcal_S} p_i\le\left(\sum_{i \in I_{1}} \sqrt{p_i}\right)   \sqrt{\frac{\ln(2K/\delta)}{n}} +(|  \Tcal_{S}|-|I_{1}|) \frac{2\ln(2K/\delta)}{n}.
                \end{align}
                Using $a_{\Tcal_S^c}(X)=\max_{i\in \Tcal^c_S} a_i(X)$ with $a_i(X) \ge 0$, 
                \begin{align*}
                        \nonumber \sum_{i=1}^K a_{i}(X) \left(p_i - \frac{X_i}{n}\right) &=\sum_{i\in \Tcal_S} a_{i}(X) \left(p_i - \frac{X_i}{n}\right)  + \sum_{i\notin \Tcal_S} a_i (X)p_i 
                        \\ & \le \sum_{i\in \Tcal_S} a_i(X) \left(p_i - \frac{X_i}{n}\right)  + a_{\Tcal_S^c}(X) \sum_{i\notin \Tcal_S} p_i  
                \end{align*}
                Plugging \eqref{eq:11} in the first term and \eqref{eq:13} in the second term, 
                \begin{align} \label{eq:14}
                        & \sum_{i=1}^K a_{i}(X) \left(p_i - \frac{X_i}{n}\right) 
                        \\ \nonumber & \le  \sum_{i\in \Tcal_S} a_i(X)  \sqrt{\frac{2p_i \ln(2K/\delta)}{n}}+ a_{\Tcal_S^c}(X) \left(\sum_{i \in I_{1}} \sqrt{p_i}\right)   \sqrt{\frac{\ln(2K/\delta)}{n}} +a_{\Tcal_S^c}(X)(|  \Tcal_{S}|-|I_{1}|) \frac{2\ln(2K/\delta)}{n}
                        \\ \nonumber & \le \left(\sum_{i\in \Tcal_S} a_i(X) \sqrt{p_i}\right) \sqrt{\frac{2 \ln(2K/\delta)}{n}}+ a_{\Tcal_S^c}(X) \left(\sum_{i \in I_{1}} \sqrt{p_i}\right)   \sqrt{\frac{\ln(2K/\delta)}{n}} +a_{\Tcal_S^c}(X)(|  \Tcal_{S}|-|I_{1}|) \frac{2\ln(2K/\delta)}{n}
                \end{align}
                Using $a_{\Tcal_S}(X)=\max_{i\in \Tcal_S} a_i(X)$, since $\sum_{i\in \Tcal_S} a_i(X) \sqrt{p_i} \le a_{\Tcal_S}(X) \sqrt{| \Tcal_S|}$ and $\sum_{i \in I_{1}} \sqrt{p_i}\le \sum_{i\in \Tcal_S} \sqrt{p_i} \le \sqrt{| \Tcal_S|}$ by using the Cauchy-Schwarz inequality ($\sum_{i=1}^r b_{i} \sqrt{a_i} \le \sqrt{\sum_{i=1}^r a_i}  \sqrt{\sum_{i=1}^r b_{i}^2} $), 
                \begin{align*}
                        & \sum_{i=1}^K \alpha_i(X) \left(p_i - \frac{X_i}{n}\right) 
                        \\  &\le a_{\Tcal_S}(X) \sqrt{\frac{2| \Tcal_S| \ln(2K/\delta)}{n}}  +a_{\Tcal_S^c}(X)  \sqrt{\frac{| \Tcal_S|\ln(2K/\delta)}{n}} +a_{\Tcal_S^c}(X) \frac{2| \Tcal_S|\ln(2K/\delta)}{n}
                        \\ & \le(a_{\Tcal_S}(X)\sqrt{2} +a_{\Tcal_S^c}(X) ) \sqrt{\frac{| \Tcal_S| \ln(2K/\delta)}{n}} + a_{\Tcal_S^c}(X) \frac{2| \Tcal_S|\ln(2K/\delta)}{n}.
                \end{align*}
        \end{proof}
        
        \subsection{Tighter version} \label{app:1}
        
        \begin{lemma} \label{lemma:8}
                For any $\delta>0$, with probability at least $1-\delta$, the following holds:
                \begin{align*}
                        &\sum_{i=1}^K \alpha_i(X) \left(p_i - \frac{X_i}{n}\right) 
                        \\ & \le \sqrt{\frac{ \ln(2K/\delta)}{n}} \left(  \sum_{i \in\Tcal_S}   (a_{\Tcal_S^c}(X)+\sqrt{2}a_i(X)   ) \sqrt{\frac{X_i}{n}}  \right) + \frac{2\ln(2K/\delta)}{n}    \left( a_{\Tcal_S^c}(X)|  \Tcal_{S}| +   \sum_{i \in\Tcal_S} a_i(X) \right).
                \end{align*}
        \end{lemma}
        \begin{proof} We start with \eqref{eq:14} from the proof of Lemma \ref{lemma:7}: i.e.,  for any $\delta>0$, with probability at least $1-\delta$, 
                
                \begin{align} \label{eq:15}
                        p_i- \frac{X_i}{n} \le \sqrt{\frac{2p_i \ln(2K/\delta)}{n}} \qquad \forall i\in [K],
                \end{align}
                \textit{and}
                \begin{align} \label{eq:16}
                        &\sum_{i=1}^K a_{i}(X) \left(p_i - \frac{X_i}{n}\right)
                        \\ \nonumber & \le \left(\sum_{i\in \Tcal_S} a_i(X) \sqrt{p_i}\right) \sqrt{\frac{2 \ln(2K/\delta)}{n}}+ a_{\Tcal_S^c}(X) \left(\sum_{i \in I_{1}} \sqrt{p_i}\right)   \sqrt{\frac{\ln(2K/\delta)}{n}} +a_{\Tcal_S^c}(X)(|  \Tcal_{S}|-|I_{1}|) \frac{2\ln(2K/\delta)}{n}.
                \end{align}
                Instead of using the Cauchy-Schwarz inequality
                to bound  the two terms $\sum_{i\in \Tcal_S} a_i(X) \sqrt{p_i}$ and $\sum_{i \in I_{1}} \sqrt{p_i}$ (which is done in the proof of Lemma \ref{lemma:7}), we now derive and use another probabilistic bound to bound these terms. For any $I \subseteq [K]$,
                \begin{align*}
                        \sum_{i \in I} a_i(X) \sqrt{p_i} &= \sum_{i \in I} a_i(X)\left(\sqrt{p_i} - \sqrt{\frac{X_i}{n}}\right)+ \sum_{i \in I}  a_i(X) \sqrt{\frac{X_i}{n}} \\ & = \sum_{i \in I} a_i(X)\frac{p_i-\frac{X_i}{n}}{\sqrt{p_{i}}+\sqrt{\frac{X_i}{n}}}+ \sum_{i \in I}  a_i(X) \sqrt{\frac{X_i}{n}}
                \end{align*}
                Using \eqref{eq:16} for $p_i-\frac{X_i}{n}$, 
                \begin{align*}
                        \sum_{i \in I} a_i(X) \sqrt{p_i} &\le \sum_{i \in I} a_i(X)\frac{1}{\sqrt{p_{i}}+\sqrt{\frac{X_i}{n}}} \sqrt{\frac{2p_i \ln(2K/\delta)}{n}}  + \sum_{i \in I}  a_i(X) \sqrt{\frac{X_i}{n}}
                        \\ & \le \sum_{i \in I} a_i(X) \sqrt{\frac{2p_i \ln(2K/\delta)}{np_{i}}}  + \sum_{i \in I}  a_i(X) \sqrt{\frac{X_i}{n}} 
                        \\ & = \sqrt{\frac{2 \ln(2K/\delta)}{n}}  \sum_{i \in I} a_i(X) + \sum_{i \in I}  a_i(X) \sqrt{\frac{X_i}{n}} 
                \end{align*}
                Thus, by setting  $I= \Tcal_S$ as well as $I= I_{1}$ and $a_i(X)=1$,
                \begin{align*} 
                        & \nonumber \sum_{i \in \Tcal_S} a_i(X) \sqrt{p_i} \le \sqrt{\frac{2 \ln(2K/\delta)}{n}}  \sum_{i \in\Tcal_S} a_i(X) + \sum_{i \in\Tcal_S}  a_i(X) \sqrt{\frac{X_i}{n}} \qquad \text{ and, } 
                        \\ & \sum_{i \in I_{1}}  \sqrt{p_i} \le  \sqrt{\frac{2 \ln(2K/\delta)}{n}}  |I_{1}| + \sum_{i \in I_{1}}   \sqrt{\frac{X_i}{n}}
                \end{align*}
                Plugging these into \eqref{eq:16},
                \begin{align*}
                        &\sum_{i=1}^K a_{i}(X) \left(p_i - \frac{X_i}{n}\right)
                        \\ \nonumber & \le \left( \sqrt{\frac{2 \ln(2K/\delta)}{n}}  \sum_{i \in\Tcal_S} a_i(X) +  \sum_{i \in\Tcal_S}  a_i(X) \sqrt{\frac{X_i}{n}}\right) \sqrt{\frac{2 \ln(2K/\delta)}{n}} 
                        \\ & \qquad+ a_{\Tcal_S^c}(X) \left(  \sqrt{\frac{2 \ln(2K/\delta)}{n}}  |I_{1}| + \sum_{i \in I_{1}}   \sqrt{\frac{X_i}{n}}\right)   \sqrt{\frac{\ln(2K/\delta)}{n}} +a_{\Tcal_S^c}(X)(|  \Tcal_{S}|-|I_{1}|) \frac{2\ln(2K/\delta)}{n}
                        \\ & = \frac{2 \ln(2K/\delta)}{n}  \sum_{i \in\Tcal_S} a_i(X)+ \sqrt{\frac{2 \ln(2K/\delta)}{n}}   \sum_{i \in\Tcal_S}  a_i(X) \sqrt{\frac{X_i}{n}}
                        \\ & \qquad+ a_{\Tcal_S^c}(X) \left(  \frac{ \ln(2K/\delta)}{n}   \sqrt{2}|I_{1}| +\sqrt{\frac{\ln(2K/\delta)}{n}} \sum_{i \in I_{1}}   \sqrt{\frac{X_i}{n}}\right)    +a_{\Tcal_S^c}(X)(|  \Tcal_{S}|-|I_{1}|) \frac{2\ln(2K/\delta)}{n}
                        \\ & \le \frac{2 \ln(2K/\delta)}{n}  \sum_{i \in\Tcal_S} a_i(X)+ \sqrt{\frac{2 \ln(2K/\delta)}{n}}   \sum_{i \in\Tcal_S}  a_i(X) \sqrt{\frac{X_i}{n}}
                        \\ & \qquad+   a_{\Tcal_S^c}(X)\sqrt{\frac{\ln(2K/\delta)}{n}} \sum_{i \in\Tcal_S}   \sqrt{\frac{X_i}{n}}    +a_{\Tcal_S^c}(X)|  \Tcal_{S}| \frac{2\ln(2K/\delta)}{n} \\ & = \sqrt{\frac{ \ln(2K/\delta)}{n}} \left(  \sum_{i \in\Tcal_S}   (a_{\Tcal_S^c}(X)+\sqrt{2}a_i(X)   ) \sqrt{\frac{X_i}{n}}  \right) + \frac{2\ln(2K/\delta)}{n}    \left( a_{\Tcal_S^c}(X)|  \Tcal_{S}| +   \sum_{i \in\Tcal_S} a_i(X) \right) 
                \end{align*}
                
        \end{proof}

        \section{Proof of Theorem \ref{thm:1}} \label{sec:3}
        This section culminates in the proof of Theorem \ref{thm:1} using our new results on multinomial distributions from Appendices \ref{appendix:fixeda} and \ref{appendix:dependenta}.         
        We define
        $$
        \Ical_{k}:=\Ical_{k}^S:=\{i\in[n]: z_{i}\in \Ccal_{k}\},
        $$
        and
        $$
        \alpha_{k}(h):=\EE_{z}[\ell(h ,z)|z\in  \Ccal_{k}].
        $$
        We start with the proof of the following lemma that relate the gap to the concentration of the  multinomial distributions: 
        \begin{lemma} \label{lemma:1}
                For any $h \in \Hcal$ and $z_i \in \Zcal$ for all $i \in [n]$,   
                \begin{align*} 
                        \EE_{z}[\ell(h ,z)]   - \frac{1}{n} \sum_{i=1}^n \ell(h, z_i) 
                        =\sum_{k=1}^K \alpha_{k}(h)\left(\Pr(z_{}\in  \Ccal_{k})- \frac{|\Ical_{k}|}{n}\right) +\frac{1}{n}\sum_{k=1}^K  |\Ical_{k}|\left(\alpha_{k}(h)-\frac{1}{|\Ical_{k}|}\sum_{i \in \Ical_{k}}\ell(h ,z_{i}) \right). 
                \end{align*}
        \end{lemma}
        \begin{proof}
                We first write the expected error as the sum of the conditional expected error:
                \begin{align*}
                        \EE_{z}[\ell(h ,z)] &=\sum_{k=1}^K \EE_{z}[\ell(h ,z)|z\in  \Ccal_{k}]\Pr(z_{}\in  \Ccal_{k}) =\sum_{k=1}^K \EE_{z_{k}}[\ell(h ,z_{k})]\Pr(z_{}\in  \Ccal_{k}),
                \end{align*}
                where $z_k$ is the random variable $z$ conditioned on the event $z\in  \Ccal_{k}$. 
                Using this, we  decompose the generalization error into two terms:\begin{align} \label{eq:1}
                        &\EE_{z}[\ell(h ,z)]  - \frac{1}{n} \sum_{i=1}^n \ell(h, z_i) 
                        \\ \nonumber & =\sum_{k=1}^K \EE_{z_{k}}[\ell(h ,z_{k})]\left(\Pr(z_{}\in  \Ccal_{k})- \frac{|\Ical_{k}|}{n}\right)
                        +\left(\sum_{k=1}^K \EE_{z_{k}}[\ell(h ,z_{k})]\frac{|\Ical_{k}|}{n}- \frac{1}{n} \sum_{i=1}^n \ell(h, z_i)\right). 
                \end{align}
                The second term in the right-hand side  of \eqref{eq:1} is further simplified by using 
                $$
                \frac{1}{n} \sum_{i=1}^n \ell(h, z_i)=\frac{1}{n}\sum_{k=1}^K  \sum_{i \in \Ical_{k}}\ell(h ,z_{i}), 
                $$
                as 
                \begin{align*}
                        \sum_{k=1}^K \EE_{z_{k}}[\ell(h ,z_{k})]\frac{|\Ical_{k}|}{n}- \frac{1}{n} \sum_{i=1}^n \ell(h, z_i)
                        & =\frac{1}{n}\sum_{k=1}^K  |\Ical_{k}|\left(\EE_{z_{k}}[\ell(h ,z_{k})]-\frac{1}{|\Ical_{k}|}\sum_{i \in \Ical_{k}}\ell(h ,z_{i}) \right).
                        \
                \end{align*}
                Substituting these into equation \eqref{eq:1} yields
                \begin{align*} 
                        &\EE_{z}[\ell(h ,z)]   - \frac{1}{n} \sum_{i=1}^n \ell(h, z_i) 
                        \\ \nonumber & =\sum_{k=1}^K \EE_{z_{k}}[\ell(h ,z_{k})]\left(\Pr(z_{}\in  \Ccal_{k})- \frac{|\Ical_{k}|}{n}\right) +\frac{1}{n}\sum_{k=1}^K  |\Ical_{k}|\left(\EE_{z_{k}}[\ell(h ,z_{k})]-\frac{1}{|\Ical_{k}|}\sum_{i \in \Ical_{k}}\ell(h ,z_{i}) \right). \end{align*}
        \end{proof}
        
        The second term in the previous lemma is bounded by the following lemma using the robustness: 
        
        \begin{lemma} \label{lemma:2}
                If a learning algorithm $\Acal$ is  \textit{$(K,\epsilon(\cdot))$-robust},
                then the following holds for any $S \in \Zcal^n$:
                $$
                \left|\frac{1}{n}\sum_{k=1}^K  |\Ical_{k}|\left(\EE_{z}[\ell(\Acal_S, z)|z\in  \Ccal_{k}]-\frac{1}{|\Ical_{k}|}\sum_{i \in \Ical_{k}}\ell(\Acal_{S} ,z_{i}) \right)\right| \le\epsilon(S). 
                $$
        \end{lemma}
        \begin{proof}
                By the triangle inequality,
                \begin{align*}
                        &\left|\frac{1}{n}\sum_{k=1}^K  |\Ical_{k}|\left(\EE_{z}[\ell(\Acal_S, z)|z\in  \Ccal_{k}]-\frac{1}{|\Ical_{k}|}\sum_{i \in \Ical_{k}}\ell(\Acal_{S} ,z_{i}) \right)\right| 
                        \\ & \le\frac{1}{n}\sum_{k=1}^K  |\Ical_{k}|\left|\EE_{z}[\ell(\Acal_S, z)|z\in  \Ccal_{k}]-\frac{1}{|\Ical_{k}|}\sum_{i \in \Ical_{k}}\ell(\Acal_{S} ,z_{i})\right|. 
                \end{align*}
                Furthermore, again by the triangle inequality,\begin{align*}
                        \left|\EE_{z}[\ell(\Acal_S, z)|z\in  \Ccal_{k}]-\frac{1}{|\Ical_k|}\sum_{S_{i} \in \Ical_k } \ell(\Acal_S, S_{i})\right| & =\left|\frac{1}{|\Ical_k|}\sum_{S_{i} \in \Ical_k }\EE_{z}[\ell(\Acal_S, z)|z\in  \Ccal_{k}]-\frac{1}{|\Ical_k|}\sum_{S_{i} \in \Ical_k } \ell(\Acal_S, S_{i})\right|  
                        \\ & \le\frac{1}{|\Ical_k|}\sum_{S_{i} \in \Ical_k }\left|\EE_{z}[\ell(\Acal_S, z)|z\in  \Ccal_{k}]- \ell(\Acal_S, S_{i})\right|
                        \\ & \le \sup_{ z \in \Zcal\cap\Ccal_k, s\in S\cap  \Ccal_k} \frac{1}{|\Ical_k|}\sum_{S_{i} \in \Ical_k }\left|\ell(\Acal_S, z)- \ell(\Acal_S,s)\right|
                        \\ & = \sup_{ z \in \Zcal\cap\Ccal_k, s\in S\cap  \Ccal_k} \left|\ell(\Acal_S, z)- \ell(\Acal_S,s)\right|.
                \end{align*}
                We now suppose that a learning algorithm $\Acal$ is $(K,\epsilon(\cdot))$-robust. Then, $\sup_{ z \in \Zcal\cap\Ccal_k, s\in S\cap  \Ccal_k} \left|\ell(\Acal_S, z)- \ell(\Acal_S,s)\right| \le\epsilon(S)$ for all $k=1,\dots, K$ by the definition of a learning algorithm $\Acal$ being $(K,\epsilon(\cdot))$-robust. Thus, we have that 
                $$
                \frac{1}{n}\sum_{k=1}^K  |\Ical_{k}|\left|\EE_{z}[\ell(\Acal_S, z)|z\in  \Ccal_{k}]-\frac{1}{|\Ical_{k}|}\sum_{i \in \Ical_{k}}\ell(\Acal_{S} ,z_{i})\right|\le\epsilon(S) \left(\frac{1}{n}\sum_{k=1}^K  |\Ical_{k}| \right)=\epsilon(S), 
                $$
                since $\sum_{k=1}^K  |\Ical_{k}|=n$.  
        \end{proof}

        Using these lemmas and our new concentration bounds on multinomial distributions   from Appendices \ref{appendix:fixeda} and \ref{appendix:dependenta}, we can complete the proof of Theorem \ref{thm:1} as follows:

        \begin{proof}[Proof of Theorem \ref{thm:1}]  
                From Lemma \ref{lemma:1}, 
                \begin{align} \label{eq:2}
                        & \EE_{z}[\ell(\Acal_S ,z)]   - \frac{1}{n} \sum_{i=1}^n \ell(\Acal_S, z_i)
                        \\ \nonumber & =\frac{1}{n}\sum_{k=1}^K  |\Ical_{k}|\left(\alpha_{k}(\Acal_S)-\frac{1}{|\Ical_{k}|}\sum_{i \in \Ical_{k}}\ell(\Acal_S ,z_{i}) \right)+\sum_{k=1}^K \alpha_{k}(\Acal_S)\left(\Pr(z_{}\in  \Ccal_{k})- \frac{|\Ical_{k}|}{n}\right) .  
                \end{align}
                By using Lemma \ref{lemma:7} with
                $a_{k}(X)=\alpha_k (\Acal_S)$  and noticing that $a_{\Tcal_S}(X),a_{\Tcal_S^c}(X)\le \zeta(\Acal_S)$, we have that for any $\delta>0$, with probability at least $1-\delta$, 
                \begin{align} \label{eq:3}
                        \sum_{k=1}^K \alpha_{k}(\Acal_S)\left(\Pr(z_{}\in  \Ccal_{k})- \frac{|\Ical_{k}|}{n}\right) \le\zeta(\Acal_S) \left((\sqrt{2}+1) \sqrt{\frac{|\Tcal_{S}| \ln(2K/\delta)}{n}}  + \frac{2|\Tcal_{S}|\ln(2K/\delta)}{n}\right). \end{align}
                Invoking Lemma \ref{lemma:2}, 
                \begin{align} \label{eq:4}
                        \frac{1}{n}\sum_{k=1}^K  |\Ical_{k}|\left(\alpha_{k}(\Acal_S)-\frac{1}{|\Ical_{k}|}\sum_{i \in \Ical_{k}}\ell(\Acal_S ,z_{i}) \right) \le \left|\frac{1}{n}\sum_{k=1}^K  |\Ical_{k}|\left(\alpha_{k}(\Acal_S)-\frac{1}{|\Ical_{k}|}\sum_{i \in \Ical_{k}}\ell(\Acal_S ,z_{i}) \right)\right|\le\epsilon(S).
                \end{align}
                By combining \eqref{eq:2}--\eqref{eq:4},  we obtain the desired statement.
        \end{proof}
        
        \section{Proof of Proposition \ref{p:pkdecay}}
        
        \begin{proof}[Proof of Proposition \ref{p:pkdecay}]
Take $m=\beta (\ln n)^{1/\alpha}$, then we have 
\begin{equation}
\Tcal_{S}\leq m+\sum_{k>m} {\bm 1}(|\Ical_{k}^{S}|\ge 1).
\end{equation}
Since $z_i$ are i.i.d, similarly to the Chernoff bound, we have
\begin{align}\label{e:che}\begin{split}
\mathbb P\left(\sum_{k>m} {\bm 1}(|\Ical_{k}^{S}|\ge 1)\geq \lambda \right)
&\leq \mathbb P\left(\sum_{i=1}^n {\bm 1}(z_i\in  \cup_{k>m}\Ccal_k) \geq\lambda\right)
\leq 
\prod_{i=1}^n \mathbb E[e^{{\bm 1}(z_i\in  \cup_{k>m}\Ccal_k)}]e^{-\lambda}\\
&=
\prod_{i=1}^n (1+\mathbb P(z_i\in  \cup_{k>m}\Ccal_k)(e-1))e^{-\lambda}
\leq e^{n\mathbb P(z\in  \cup_{k>m}\Ccal_k)(e-1)-\lambda}
\end{split}\end{align}

Under our assumption that $p_k\leq C e^{-(k/\beta)^\alpha}$, we have 
\begin{align}\begin{split}\label{e:bb}
&\phantom{{}={}}\mathbb P(z\in \cup_{k>m}\Ccal_k)=\sum_{k>m}p_k\leq \sum_{k>m}C e^{-(k/\beta)^\alpha}\\
&\leq \int_m^\infty Ce^{-(x/\beta)^\alpha}{\rm d} x
=\frac{C\beta}{\alpha}\int_{(m/\beta)^\alpha}^\infty y^{\frac{1}{\alpha}-1}e^{-y}{\rm d} y.
\end{split}\end{align}
 There are two cases i) $\alpha\geq1$ and ii) $\alpha<1$. In the first case, if $\alpha\geq 1$, we recall from our choice that $(m/\beta)^\alpha=\ln n\geq 1$. So the integrand $y^{1/\alpha-1}\leq 1$ and 
 \begin{align}\label{e:bb1}
  \mathbb P(z\in \cup_{k>m}\Ccal_k)
  \leq \frac{C\beta}{\alpha}\int_{(m/\beta)^\alpha}^\infty e^{-y}{\rm d} y
  =\frac{C\beta}{\alpha}e^{-(m/\beta)^\alpha}=\frac{C\beta}{n\alpha}.
 \end{align}
By plugging \eqref{e:bb1} into \eqref{e:che}, and take $\lambda=C(e-1)\beta/\alpha +\log(1/\delta)$, we can conclude that 
\begin{equation}
\mathbb P\left(\sum_{k>m} {\bm 1}(|\Ical_{k}^{S}|\ge 1)\geq \lambda \right)\leq \delta,
\end{equation}
and the first inequality follows. 
 
For the second case when $\alpha<1$, the integral in \eqref{e:bb} is the incomplete Gamma function
\begin{align}
    \mathbb P(z\in \cup_{k>m}\Ccal_k)\leq \frac{C\beta}{\alpha}\int_{(m/\beta)^\alpha}^\infty y^{\frac{1}{\alpha}-1}e^{-y}{\rm d} y
    =\frac{C\beta}{\alpha}\Gamma(1/\alpha, \ln n).
\end{align}
For the incomplete Gamma function, we have the following bound from \cite{natalini2000inequalities,borwein2009uniform}:
for $a>1$, $B>1$
\begin{align}\label{e:gamma}
    \Gamma(a,x)\leq B x^{a-1}e^{-x}, \text{ when } x\geq \frac{B}{B-1}(a-1).
\end{align}
By our assumption $\ln n \geq 2/\alpha$, we can take $B=2$, $a=1/\alpha$ and $x=\ln n$ in \eqref{e:gamma} to conclude 
\begin{align}\label{e:bb2}
    \mathbb P(z\in \cup_{k>m}\Ccal_k)\leq \frac{C\beta}{\alpha}\Gamma(1/\alpha, \ln n)
    \leq \frac{2C\beta }{n\alpha}(\ln n)^{\frac{1}{\alpha}-1}.
\end{align}

By plugging \eqref{e:bb2} into \eqref{e:che}, and take $\lambda=2C(e-1)\beta (\ln n)^{\frac{1}{\alpha}-1}/\alpha +\log(1/\delta)$, we can conclude that 
\begin{equation}
\mathbb P\left(\sum_{k>m} {\bm 1}(|\Ical_{k}^{S}|\ge 1)\geq \lambda \right)\leq \delta.
\end{equation}
Noticing that by our assumption $\ln n \geq 1/\alpha$, it gives that
\begin{align*}
    | \Tcal_{S}|&\leq
        \beta (\ln n)^{\frac{1}{\alpha}}+ \frac{2C(e-1)\beta}{\alpha \ln n}(\ln n)^{\frac{1}{\alpha}} +\log(1/\delta)\\
        &\leq 
         (1+2C(e-1))\beta(\ln n)^{\frac{1}{\alpha}} +\log(1/\delta),
\end{align*}
the second inequality follows.

\end{proof}
        
        \section{Proof of Theorem \ref{thm:2}}  \label{appendix:mainresult2}
        In this section, we refine the proof of Theorem \ref{thm:1} to obtain a tighter bound by using a tighter version of our new concentration bounds on multinomial distributions   from Appendices \ref{appendix:fixeda} and \ref{appendix:dependenta}.         \begin{proof}[Proof of Theorem \ref{thm:2}]
                The proof begins in the same manner as in the previous section, but uses the sharper multinomial bound.
                From Lemma \ref{lemma:1}, 
                \begin{align} 
                        & \EE_{z}[\ell(\Acal_S ,z)]   - \frac{1}{n} \sum_{i=1}^n \ell(\Acal_S, z_i)
                        \\ \nonumber & =\frac{1}{n}\sum_{k=1}^K  |\Ical_{k}|\left(\alpha_{k}(\Acal_S)-\frac{1}{|\Ical_{k}|}\sum_{i \in \Ical_{k}}\ell(\Acal_S ,z_{i}) \right)+\sum_{k=1}^K \alpha_{k}(\Acal_S)\left(\Pr(z_{}\in  \Ccal_{k})- \frac{|\Ical_{k}|}{n}\right) .  
                \end{align}
                By using Lemma \ref{lemma:8} with
                $a_{k}(X)=\alpha_k (\Acal_S), a_{\Tcal_S}(X)=\alpha_{\Tcal_S} (\Acal_S)$, and $a_{\Tcal_S^c}(X)=\alpha_{\Tcal_S^c}(\Acal_S)$, we have that for any $\delta>0$, with probability at least $1-\delta$, 
                \begin{align} 
                        &\sum_{k=1}^K \alpha_{k}(\Acal_S)\left(\Pr(z_{}\in  \Ccal_{k})- \frac{|\Ical_{k}|}{n}\right) 
                        \\ & \le \sqrt{\frac{ \ln(2K/\delta)}{n}} \left(  \sum_{k \in\Tcal_S}   (\alpha_{\Tcal_S^c}(\Acal_S)+\sqrt{2}\alpha_k (\Acal_S)   ) \sqrt{\frac{|\Ical_{k}|}{n}}  \right) + \frac{2\ln(2K/\delta)}{n}    \left( \alpha_{\Tcal_S^c}(\Acal_S)|  \Tcal_{S}| +   \sum_{k \in\Tcal_S} \alpha_k (\Acal_S) \right) \end{align}
                Applying Lemma \ref{lemma:2}, 
                \begin{align} 
                        \frac{1}{n}\sum_{k=1}^K  |\Ical_{k}|\left(\alpha_{k}(\Acal_S)-\frac{1}{|\Ical_{k}|}\sum_{i \in \Ical_{k}}\ell(\Acal_S ,z_{i}) \right) \le \left|\frac{1}{n}\sum_{k=1}^K  |\Ical_{k}|\left(\alpha_{k}(\Acal_S)-\frac{1}{|\Ical_{k}|}\sum_{i \in \Ical_{k}}\ell(\Acal_S ,z_{i}) \right)\right|\le\epsilon(S).
                \end{align}
                By combining these, we have that for any $\delta>0$, with probability at least $1-\delta$, 
                \begin{align*}
                        \EE_{z}[\ell(\Acal_S ,z)]   - \frac{1}{n} \sum_{i=1}^n \ell(\Acal_S, z_i)
                        \le\epsilon(S)+\mathcal{\mathcal{Q}}_1 \sqrt{\frac{ \ln(2K/\delta)}{n}} + \frac{2\mathcal{\mathcal{Q}}_2\ln(2K/\delta)}{n},    
                \end{align*}
                where $\mathcal{\mathcal{Q}}_1=\sum_{k \in\Tcal_S}   (\alpha_{\Tcal_S^c}(\Acal_S)+\sqrt{2}\alpha_k (\Acal_S)   ) \sqrt{\frac{|\Ical_{k}|}{n}}$ and $\mathcal{\mathcal{Q}}_2=\alpha_{\Tcal_S^c}(\Acal_S)|  \Tcal_{S}| +   \sum_{k \in\Tcal_S} \alpha_k (\Acal_S)$. 
        \end{proof}

                \section{Pseudo-robustness} \label{appendix:pseudo}
                This section is devoted to the generalizations of Theorems \ref{thm:1} and \ref{thm:2} in the pseudo-robust context.
        \begin{definition}
                A learning algorithm $\Acal$ is\textit{ $(K,\epsilon(\cdot), \hn(\cdot))$ pseudo robust}, for $K \in \NN,$  $\epsilon(\cdot):\Zcal^n \rightarrow \RR$, and $\hn(\cdot):\Zcal^n \rightarrow \{1,\dots,n\}$, if $\Zcal$ can be partitioned into $K$ disjoint sets, denoted by $\{\Ccal_k\}_{k=1}^K$, such that for all  $S \in \Zcal^n$, there exists a subset of training samples $\hS$ with $|\hS|=\hn(S)$ and the following holds:
                $$
                \forall s \in \hS, \forall z \in \Zcal, \forall k=1,\dots,K: \text{ if } s,z \in \Ccal_k, \text{ then } |\ell(\Acal_S, s)-\ell(\Acal_S, z)| \le\epsilon(S). $$
        \end{definition}
        
        Define   $\hzeta(\Acal,S)=\max_{(k, i) \in [K] \times [n] }\left|\alpha_{k}(\Acal_S)-\ell(\Acal_S ,z_{i}) \right|$ where $S=(z_{1},\dots,z_n)$ and $\alpha_{k}(h)=\EE_{z}[\ell(h ,z)|z\in  \Ccal_{k}]$.
        
        \subsection{Simple Version}
        Our first theorem is the analogue of Theorem \ref{thm:1}.
        \begin{theorem} \label{thm:5}
                If a learning algorithm $\Acal$ is $(K,\epsilon(\cdot), \hn(\cdot))$ pseudo robust (with $\{\Ccal_k\}_{k=1}^K$),
                then  for any $\delta>0$,  with  probability at least $1-\delta$ over an
                iid draw of $n$ examples $S=(z_{i})_{i=1}^n$, the following holds: 
                \begin{align*}
                        & \EE_{z}[\ell(\Acal_S ,z)]   
                        \\ & \le\frac{1}{n} \sum_{i=1}^n \ell(\Acal_S, z_i)+\frac{\hn(S)}{n}\epsilon(S)+ \frac{n-\hn(S)}{n}  \hzeta(\Acal,S)+\zeta(\Acal_S) \left((\sqrt{2}+1) \sqrt{\frac{|\Tcal_{S}|\ln(2K/\delta)}{n}}  + \frac{2|\Tcal_{S}|\ln(2K/\delta)}{n}\right), \end{align*}
                where $\zeta(\Acal_S):=\max_{k \in [K]}\EE_{z}[\ell(\Acal_S ,z)|z\in  \Ccal_{k}]$,  $\hzeta(\Acal,S):=\max_{(k, i) \in [K] \times [n] }\left|\alpha_{k}(\Acal_S)-\ell(\Acal_S ,z_{i}) \right|$,  and $\Tcal_{S}:=\{k\in[K]: |\Ical_{k}^{S}|\ge 1\}$ with $\Ical_{k}^S:=\{i\in[n]: z_{i}\in \Ccal_{k}\}$.
        \end{theorem}

        \subsection{Stronger Version}
        The following statement is the analogue of Theorem \ref{thm:2} and is a strengthening of Theorem \ref{thm:5}.
        \begin{theorem} \label{thm:6}
                If a learning algorithm $\Acal$ is $(K,\epsilon(\cdot), \hn(\cdot))$ pseudo robust (with $\{\Ccal_k\}_{k=1}^K$),
                then  for any $\delta>0$,  with  probability at least $1-\delta$ over an
                iid draw of $n$ examples $S=(z_{i})_{i=1}^n$, the following holds: 
                \begin{align*}
                        & \EE_{z}[\ell(\Acal_S ,z)]\le \frac{1}{n} \sum_{i=1}^n \ell(\Acal_S, z_i)+\frac{\hn(S)}{n}\epsilon(S)+ \frac{n-\hn(S)}{n}  \hzeta(\Acal,S)+\mathcal{\mathcal{Q}}_1 \sqrt{\frac{ \ln(2K/\delta)}{n}} + \frac{2\mathcal{\mathcal{Q}}_2\ln(2K/\delta)}{n} \end{align*}
                where $\mathcal{\mathcal{Q}}_1:=\sum_{k \in\Tcal_S}   (\alpha_{\Tcal_S^c}(\Acal_S)+\sqrt{2}\alpha_k (\Acal_S)   ) \sqrt{\frac{|\Ical_{k}^{S}|}{n}}$, $\mathcal{\mathcal{Q}}_2:=\alpha_{\Tcal_S^c}(\Acal_S)|  \Tcal_{S}| +   \sum_{k \in\Tcal_S} \alpha_k (\Acal_S)$,  $\Tcal_{S}:=\{k\in[K]: |\Ical_{k}^{S}|\ge 1\}$ with $\Ical_{k}^S:=\{i\in[n]: z_{i}\in \Ccal_{k}\}$, $\alpha_{k}(h):=\EE_{z}[\ell(h ,z)|z\in  \Ccal_{k}]$, and $\alpha_{\Tcal_S^c}(\Acal_S):=\max_{k\in \Tcal^c_S} \alpha_{k}(\Acal_S)$ with $\Tcal^c_S = [K]\setminus \Tcal_S$.\end{theorem}

        \subsection{Proof for pseudo robustness}
        We now prove Theorems \ref{thm:5} and \ref{thm:6}.
        
        \begin{lemma} \label{lemma:11}
                If a learning algorithm $\Acal$ is  \textit{$(K,\epsilon(\cdot), \hn(\cdot))$ pseudo robust},
                then the following holds for any $S \in \Zcal^n$:
                
                $$
                \EE_{z}[\ell(\Acal_S ,z)]   - \frac{1}{n} \sum_{i=1}^n \ell(\Acal_S, z_i)
                \le \frac{\hn(S)}{n}\epsilon(S)+ \frac{n-\hn(S)}{n}  \hzeta(\Acal,S)+\sum_{k=1}^K \alpha_{k}(\Acal_S)\left(\Pr(z_{}\in  \Ccal_{k})- \frac{|\Ical_{k}|}{n}\right). $$
                where  $\hzeta(\Acal,S):=\max_{(k, i) \in [K] \times [n] }\left|\alpha_{k}(\Acal_S)-\ell(\Acal_S ,z_{i}) \right|$. 
        \end{lemma}
        \begin{proof}
                Define $\hIcal_{k} :=\{i\in[n]: z_{i}\in \hS,z_{i}\in \Ccal_{k}\}$.
                Then
                \begin{align*}
                        &\left|\frac{1}{n}\sum_{k=1}^K  |\Ical_{k}|\left(\alpha_{k}(\Acal_S)-\frac{1}{|\Ical_{k}|}\sum_{i \in \Ical_{k}}\ell(\Acal_S ,z_{i}) \right)\right|
                        \\ & =\left|\frac{1}{n}\sum_{k=1}^K  \left(|\Ical_{k}|\alpha_{k}(\Acal_S)-\sum_{i \in\hIcal_{k} }\ell(\Acal_S ,z_{i})-\sum_{i \in \Ical_{k}\wedge i \notin\hIcal_{k}  }\ell(\Acal_S ,z_{i}) \right)\right| 
                        \\ & = \left|\frac{1}{n}\sum_{k=1}^K  \sum_{i \in\hIcal_{k} }(\alpha_{k}(\Acal_S)-\ell(\Acal_S ,z_{i}))+\frac{1}{n}\sum_{k=1}^K  \sum_{i \in \Ical_{k}\wedge i \notin\hIcal_{k} }(\alpha_{k}(\Acal_S)-\ell(\Acal_S ,z_{i}) )\right| 
                        \\ & \le \frac{1}{n}\sum_{k=1}^K  \sum_{i \in\hIcal_{k} } \left|\alpha_{k}(\Acal_S)-\ell(\Acal_S ,z_{i}) \right| + \frac{1}{n}\sum_{k=1}^K  \sum_{i \in \Ical_{k}\wedge i \notin\hIcal_{k} } \left|\alpha_{k}(\Acal_S)-\ell(\Acal_S ,z_{i}) \right| 
                        \\ & \le  \frac{\hn(S)}{n}\epsilon(S)+ \frac{n-\hn(S)}{n}  \hzeta(\Acal,S),
                \end{align*}
                where $\hzeta(\Acal,S)=\max_{(k, i) \in [K] \times [n] }\left|\alpha_{k}(\Acal_S)-\ell(\Acal_S ,z_{i}) \right|$. Combining this with Lemma \ref{lemma:1} gives 
                \begin{align*}
                        &\EE_{z}[\ell(\Acal_S ,z)]   - \frac{1}{n} \sum_{i=1}^n \ell(\Acal_S, z_i)
                        \\ & =\frac{1}{n}\sum_{k=1}^K  |\Ical_{k}|\left(\alpha_{k}(\Acal_S)-\frac{1}{|\Ical_{k}|}\sum_{i \in \Ical_{k}}\ell(\Acal_S ,z_{i}) \right)+\sum_{k=1}^K \alpha_{k}(\Acal_S)\left(\Pr(z_{}\in  \Ccal_{k})- \frac{|\Ical_{k}|}{n}\right)  
                        \\ \nonumber & \le \frac{\hn(S)}{n}\epsilon(S)+ \frac{n-\hn(S)}{n}  \hzeta(\Acal,S)+\sum_{k=1}^K \alpha_{k}(\Acal_S)\left(\Pr(z_{}\in  \Ccal_{k})- \frac{|\Ical_{k}|}{n}\right).
                \end{align*}
        \end{proof}

        We now have all the tools necessary to complete the proofs of the main theorems.
        
        \begin{proof}[Proof of Theorem \ref{thm:5}]
                By  Lemma \ref{lemma:11}, 
                $$
                \EE_{z}[\ell(\Acal_S ,z)]   - \frac{1}{n} \sum_{i=1}^n \ell(\Acal_S, z_i)
                \le \frac{\hn(S)}{n}\epsilon(S)+ \frac{n-\hn(S)}{n}  \hzeta(\Acal,S)+\sum_{k=1}^K \alpha_{k}(\Acal_S)\left(\Pr(z_{}\in  \Ccal_{k})- \frac{|\Ical_{k}|}{n}\right). $$
                By using Lemma \ref{lemma:7} with
                $a_{k}(X)=\alpha_k (\Acal_S)$  and noticing that $a_{\Tcal_S}(X),a_{\Tcal_S^c}(X)\le \zeta(\Acal_S)$, we have that for any $\delta>0$, with probability at least $1-\delta$, 
                \begin{align*} 
                        \sum_{k=1}^K \alpha_{k}(\Acal_S)\left(\Pr(z_{}\in  \Ccal_{k})- \frac{|\Ical_{k}|}{n}\right) \le\zeta(\Acal_S) \left((\sqrt{2}+1) \sqrt{\frac{|\Tcal_{S}| \ln(2K/\delta)}{n}}  + \frac{2|\Tcal_{S}|\ln(2K/\delta)}{n}\right). 
                \end{align*}
        \end{proof}

        \begin{proof}[Proof of Theorem \ref{thm:6}]
                By  Lemma \ref{lemma:11}, 
                $$
                \EE_{z}[\ell(\Acal_S ,z)]   - \frac{1}{n} \sum_{i=1}^n \ell(\Acal_S, z_i)
                \le \frac{\hn(S)}{n}\epsilon(S)+ \frac{n-\hn(S)}{n}  \hzeta(\Acal,S)+\sum_{k=1}^K \alpha_{k}(\Acal_S)\left(\Pr(z_{}\in  \Ccal_{k})- \frac{|\Ical_{k}|}{n}\right). $$
                By using Lemma \ref{lemma:8} with
                $a_{k}(X)=\alpha_k (\Acal_S), a_{\Tcal_S}(X)=\alpha_{\Tcal_S} (\Acal_S)$, and $a_{\Tcal_S^c}(X)=\alpha_{\Tcal_S^c}(\Acal_S)$, we have that for any $\delta>0$, with probability at least $1-\delta$, 
                \begin{align*} 
                        &\sum_{k=1}^K \alpha_{k}(\Acal_S)\left(\Pr(z_{}\in  \Ccal_{k})- \frac{|\Ical_{k}|}{n}\right) 
                        \\ & \le \sqrt{\frac{ \ln(2K/\delta)}{n}} \left(  \sum_{k \in\Tcal_S}   (\alpha_{\Tcal_S^c}(\Acal_S)+\sqrt{2}\alpha_k (\Acal_S)   ) \sqrt{\frac{|\Ical_{k}|}{n}}  \right) + \frac{2\ln(2K/\delta)}{n}    \left( \alpha_{\Tcal_S^c}(\Acal_S)|  \Tcal_{S}| +   \sum_{k \in\Tcal_S} \alpha_k (\Acal_S) \right). 
                \end{align*}
        \end{proof}
        
\section{Proof of Theoretical Comparisons} \label{sec:theory}

\paragraph{Proof of Example \ref{example:2}.} First, in order to show the bound in Theorem 1 is tighter than the that of Proposition \ref{prop:1}, we must show that
\begin{align} 
(\sqrt{2}+1) \sqrt{\frac{|\Tcal_{S}|\ln(2K/\delta)}{n}}+ \frac{2|\Tcal_{S}|\ln(2K/\delta)}{n}\le \sqrt{\frac{2K\ln2+2 \ln (1/\delta)}{n}}. 
\end{align}
It is not hard to see for any given $\delta,$ when $n>2|\mathcal{T}_S|\ln(2K/\delta)$, $|\mathcal {T}_S|\ll K$, and $2K> 1/\delta$, the above inequality holds.

We can now divide each coordinate of $z^{(x)}$ into equal sized $\nu$-length intervals (possibly excluding the last interval):
$$[-1,-1+\nu),[-1+\nu,-1+2\nu),\cdots [-1+i\nu,-1+(i+1)\nu),\cdots$$
Then, $\{\mathcal C_i\}$ is the Cartesian product of the intervals of each coordinate. Notice that by standard concentration for the maximum of a sequence of sub-gaussians, for any $\delta>0$, there exists small enough $\sigma>0$, such that $\|x^{(2)}-\mu\|_\infty<1/\nu$ with probabilty at least $1-\delta$. Let us choose $\mu=c\cdot (1,1,\cdots,1)^\top$, where $c\in [-1,1]$ is the center of one of the intervals constructed above. Then, with probability at least $1-\delta$, $|\mathcal{T}_S|=\Theta((2/\nu)^{p+1})$. As a result, when $d\gg p$, we have the desired inequality.

\paragraph{Proof of Example \ref{example:1}.} Recall the bound in Theorem 1 implies that
        \begin{align}
\EE_{z}[\ell(\Acal_S ,z)] \le\frac{1}{n} \sum_{i=1}^n \ell(\Acal_S, z_i) +\epsilon(S)+\sqrt{\frac{B}{\lambda}} \Bigg((\sqrt{2}+1) \sqrt{\frac{|\Tcal_{S}|\ln(2K/\delta)}{n}}+ \frac{2|\Tcal_{S}|\ln(2K/\delta)}{n}\Bigg)
                \end{align}
        
On the other hand, the bound obtained via uniform stability is:
\begin{align} \label{eq:uniformstability}
\EE_{z}[\ell(\Acal_S ,z)] 
                \le\frac{1}{n} \sum_{i=1}^n \ell(\Acal_S, z_i) +2\beta+ (4n\beta+\sqrt{\frac{B}{\lambda}})\sqrt{\frac{\ln(1/\delta)}{2n}}.
                \end{align}
When $n$ and $B$ are large, the dominating term is $\sqrt{\frac{B}{\lambda}} (\sqrt{2}+1) \sqrt{\frac{|\Tcal_{S}|\ln(2K/\delta)}{n}}$ and $4n\beta\sqrt{\frac{\ln(1/\delta)}{2n}}$, where we take $\beta=2B^2/(\lambda n)$ as in \citet{bousquet2002stability}. We can divide $z^{(x)}$ into equal sized $\nu$-length intervals (again with the possible exception of the last interval):
$$[0,\nu),[\nu,2\nu),\cdots,[i\nu,(i+1)\nu),\cdots$$

If there is no noise, i.e. $z^{(y)}=w^*z^{(x)}$, then all the points will fall on the line segment $(z^{(x)},z^{(x)})$. When we have a Gaussian perturbation over $z^{(x)}$, by suitably choosing the variance parameter $\sigma>0$, and concentration of Gaussian variables, we can let most of the data mass covered in the union of $\{\mathcal{C}_i\}$, where $\mathcal{C}_i=\{(x,y): (x-\frac{i\nu}{2})^2+(y-\frac{i\nu}{2})^2\le \frac{\nu^2}{2}\}$ is a circle with radius $\frac{\sqrt{2}\nu}{2}$ and its center is on the line segment $(z^{(x),z^{(x)}})$. We then have $ |\mathcal T_S|\le \Theta(2/\nu)$. Thus, when $B\gg 2/\nu$, our bound is strictly less than the uniform stability result \eqref{eq:uniformstability}.

\section{Additional Experimental Results and Details } \label{app:exp}

We report the additional experimental results in Figures \ref{fig:app:1}, \ref{fig:app:2}, and \ref{fig:app:3}, where we can observe that our new bounds provide the significant improvements over the previous bounds.   

For the real-world data, we adopted the standard benchmark datasets  --- MNIST \citep{lecun1998gradient}, CIFAR-10 \allowbreak \citep{krizhevsky2009learning}, CIFAR-100 \citep{krizhevsky2009learning}, SVHN \citep{netzer2011reading}, Fashion-MNIST (FMNIST) \citep{xiao2017fashion}, Kuzushiji-MNIST (KMNIST) \citep{clanuwat2019deep}, and Semeion \citep{srl1994semeion}. We used  all the training samples exactly as provided by those datasets. For the synthetic data, we generated them by sampling the input $x \in \Xcal$ from beta distributions and Gaussian mixture distributions with a variety of hyperparameters.  Beta($\alpha$, $\beta$) indicates the Beta distribution with hyper-parameters $\alpha$ and $\beta$. Gauss mix ($\sigma$) means the mixture of five Gaussian distributions with a standard deviation $\sigma$. Beta mix ($\alpha$, $\beta$)-($\sigma$) represents the mixture of beta distributions generated by the following procedure:
$$
x=0.4*v_0 + v_1 + v_2, 
$$
where $v_0$ is drawn from the uniform distribution on $[0,1]$, $v_1 \sim$ Beta($\alpha$, $\beta$), and $v_2 \sim$ Beta($\sigma$, $\sigma$). Similarly, beta-Gauss  ($\alpha$, $\beta$)-($\sigma$) represents the  mixture of  distributions generated by the following procedure:
$$
x=0.4*v_0 + v_1 + v_2, 
$$ 
where $v_0$ is drawn from the uniform distribution on $[0,1]$, $v_1 \sim$ Beta($\alpha$, $\beta$), and $v_2$ is drawn from the Gaussian distribution with a standard deviation $\sigma$. For all the synthetic data, we generated and used 1000
training data points. 

For the partition $\{\Ccal_k\}_{k=1}^K$, we consider the division of the input space $\Xcal$ because we can either (1) assume that there exists a function $\hat y$ such that $y=\hat y(x)$ or (2) notice that the partition $\{\Ccal_k\}_{k=1}^K$ of the input space $\Xcal$ can be dictated by the label $y \in \Ycal$; i.e., $K=|\Ycal| \times K'$ where $K'$ is the size of the partition of the input space $\Xcal$, which is used in the previous paper \citep{xu2012robustness}. Thus,  we can focus on partition of the input space $\Xcal$ for the purpose of comparing $K$ and $|\Tcal_S|$.

The $\epsilon$-covering of $\Xcal \subseteq[0,1]^d$ can be defined by the following. We first define
$$   
\Ccal_{k_1,\dots,k_d}' = \{x \in \Xcal : 0.1(k_j-1)\le x_j < 0.1k_j + \one(k_j=10), j=1,\dots,d\}, \qquad k_1,\dots,k_d \in \{1,\dots,10\},
$$
where $\one(k_j = 10)$ is one if $k_j=10$ and is zero otherwise. Note that without this notation of $\one(k_j = 10)$, equivalently, we can define the condition by $0.1(k_j-1)\le x_j \le 0.1k_j$ if $k_j = 10$ and $0.1(k_j-1)\le x_j < 0.1k_j$ if $k_j < 10$, since $\Xcal \subseteq [0,1]^d$. We then define $\Ccal_k$ to be the flatten version of $\Ccal_{k_1,\dots,k_d}' $; i.e., $\{\Ccal_k\}_{k=1}^K=\{\Ccal_{k_1,\dots,k_d} '\}_{k_1,\dots,k_d\in [10]}$ with $C_1=\Ccal_{1,1,\dots,1}'$, $C_2= \Ccal_{2,1,\dots,1}'$, $C_{10}= \Ccal_{10,1,\dots,1}'$, $C_{10+1}= \Ccal_{1,2,1,\dots,1}'$, $C_{20}= \Ccal_{10,2,1,\dots,1}'$, and so on.
While the $\epsilon$-covering of the original input space $\Xcal$ is the default example from the previous paper \citep{xu2012robustness}, in Figure \ref{fig:1} we see that $K$ grows rapidly as $d$ increases. Therefore, to reduce $K$ significantly, we also propose utilizing the inverse image of the $\epsilon$-covering in a randomly projected space. That is, given a random matrix $A$, we use the $\epsilon$-covering of the space of $u=Ax$ to define the pre-partition $\{\tilde \Ccal_k\}_{k=1}^K$. More concretely, the random matrix $A$ for  the projection was generated by the following procedure: 
\begin{enumerate}
\item 
Each entry of a random matrix $\tilde A$ is generated by the Uniform Distribution on $[0,1]$ independently. 
\item
Each row of the random matrix $\tilde A \in \RR^{3 \times d}$  is then normalized so that $Ax \in [0,1]^3$; i.e., 
$$
A_{ij} = \frac{\tilde A _{ij}}{\sum_{j=1}^d \tilde A_{ij}}.
$$
\end{enumerate}
Then,  we  can define
$$   
\tilde \Ccal_{k_1,k_{2},k_3}' = \{u \in  [0,1]^3: 0.1(k_j-1)\le x_j < 0.1k_j+\one(k_j=10), j=1,2,3\}, \qquad k_1,k_2,k_3 \in \{1,\dots,10\}.
$$

We then define $\tilde \Ccal_k$ to be the flatten version of $\tilde \Ccal_{k_1,\dots,k_d}' $; i.e., $\{\tilde \Ccal_k\}_{k=1}^K=\{\tilde \Ccal_{k_1,\dots,k_d} '\}_{k_1,\dots,k_d\in [10]}$. Finally,  the partition $\{\Ccal_k\}_{k=1}^K$ is defined by $\Ccal_k= \{x \in \Xcal : Ax \in \tilde \Ccal_{k}\}$. In this study, we randomly generated matrix $A \in \RR^{3 \times d}$ in each trial.

For the clustering with unlabeled data, given a set of unlabeled data points $\{\bar x_k\}_{k=1}^K$, the partition $\{\Ccal_k\}_{k=1}^K$ is defined by $\Ccal_k= \{x \in \Xcal : k =\argmin_{k' \in [K]} \|x-\bar x_{k'} \|_2 \}$. Following the literature on semi-supervised learning, we randomly split the training data points into labeled data points (500 for Semeion and 5000 for all other datasets) and unlabeled data points (the remainder of all the training data). 

\begin{figure}[!t]
\center
\begin{subfigure}[b]{0.24\textwidth}
  \includegraphics[width=\textwidth, height=0.8\textwidth]{fig/results_1/base_beta_1000_001_01_01_01_0.pdf}
  \vspace{-15pt}
  \caption{Beta(0.1, 0.1)} 
  \vspace{10pt}
\end{subfigure}
\begin{subfigure}[b]{0.24\textwidth}
  \includegraphics[width=\textwidth, height=0.8\textwidth]{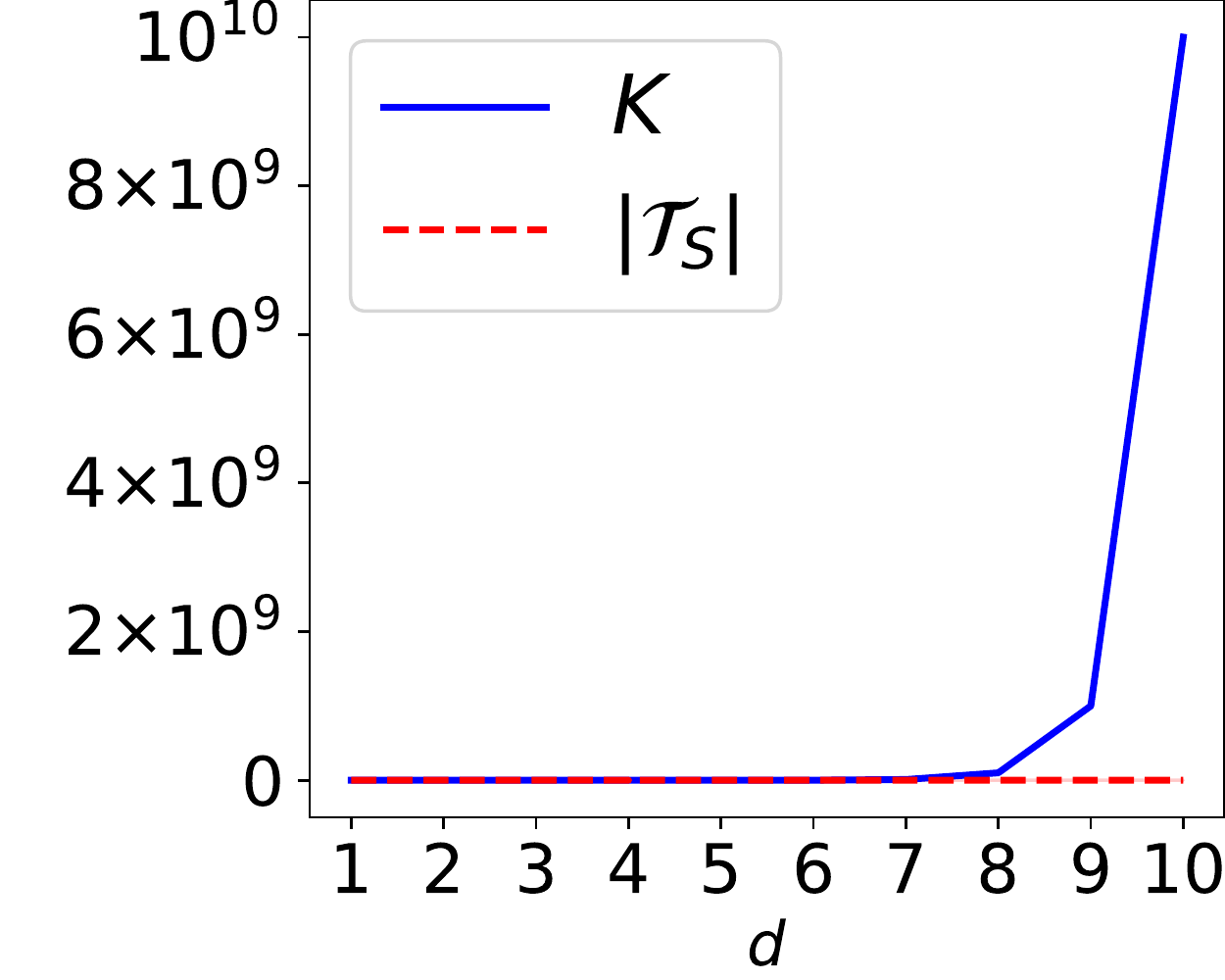}
  \vspace{-15pt}
  \caption{Beta(0.01, 0.01)} 
  \vspace{10pt}
\end{subfigure}
\begin{subfigure}[b]{0.24\textwidth}
  \includegraphics[width=\textwidth, height=0.8\textwidth]{fig/results_1/base_beta_1000_001_01_100_01_0.pdf}
  \vspace{-15pt}
  \caption{Beta(0.1, 10)} 
  \vspace{10pt}
\end{subfigure}
\begin{subfigure}[b]{0.24\textwidth}
  \includegraphics[width=\textwidth, height=0.8\textwidth]{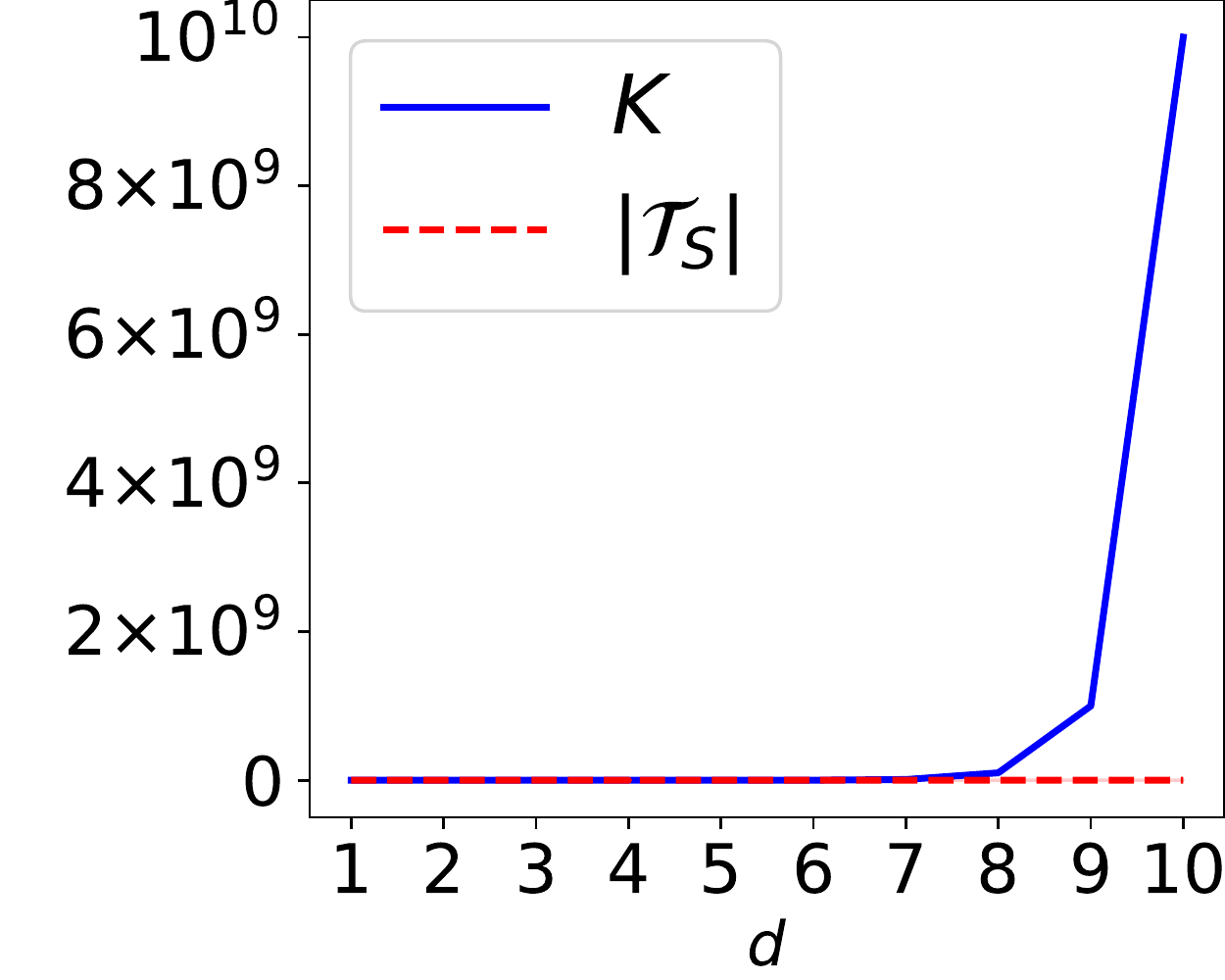}
  \vspace{-15pt}
  \caption{Beta mix (0.1, 0.1)-(0.01)} 
  \vspace{10pt}
\end{subfigure}
\begin{subfigure}[b]{0.24\textwidth}
  \includegraphics[width=\textwidth, height=0.8\textwidth]{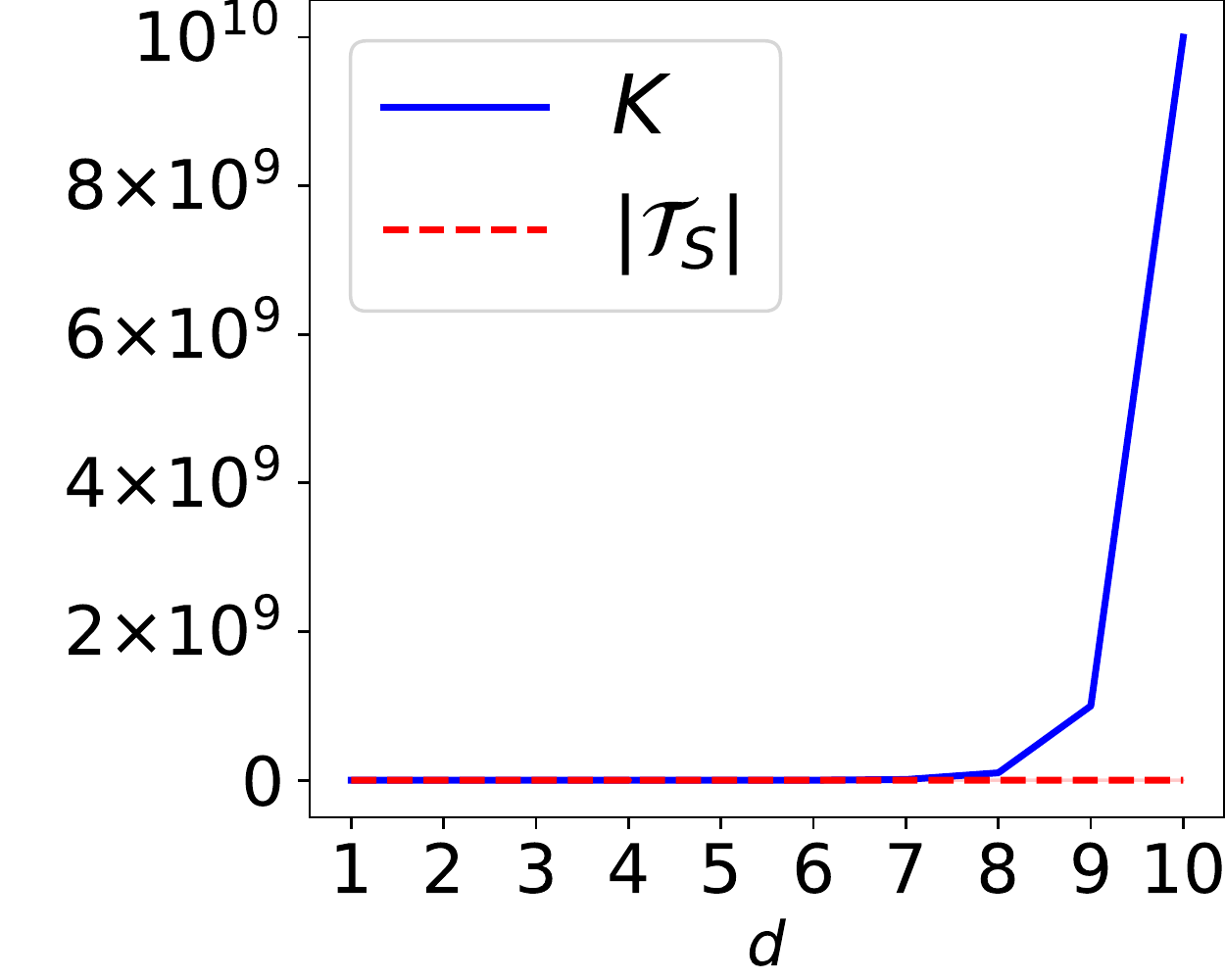}
  \vspace{-15pt}
  \caption{Beta mix (0.1, 0.1)-(0.1)} 
  \vspace{10pt}
\end{subfigure}
\begin{subfigure}[b]{0.24\textwidth}
  \includegraphics[width=\textwidth, height=0.8\textwidth]{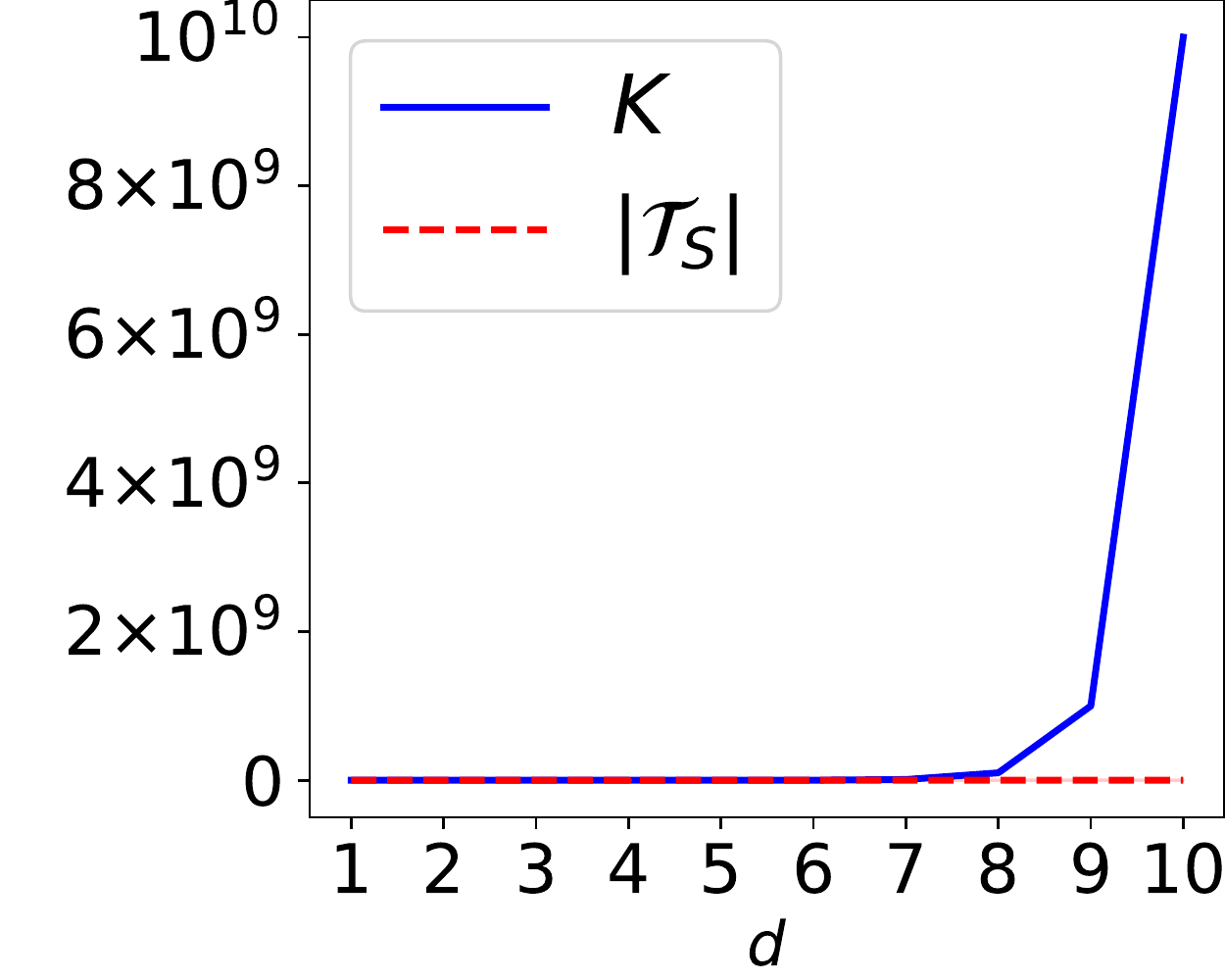}
  \vspace{-15pt}
  \caption{Beta mix (0.1, 0.1)-(10)} 
  \vspace{10pt}
\end{subfigure}
\begin{subfigure}[b]{0.24\textwidth}
  \includegraphics[width=\textwidth, height=0.8\textwidth]{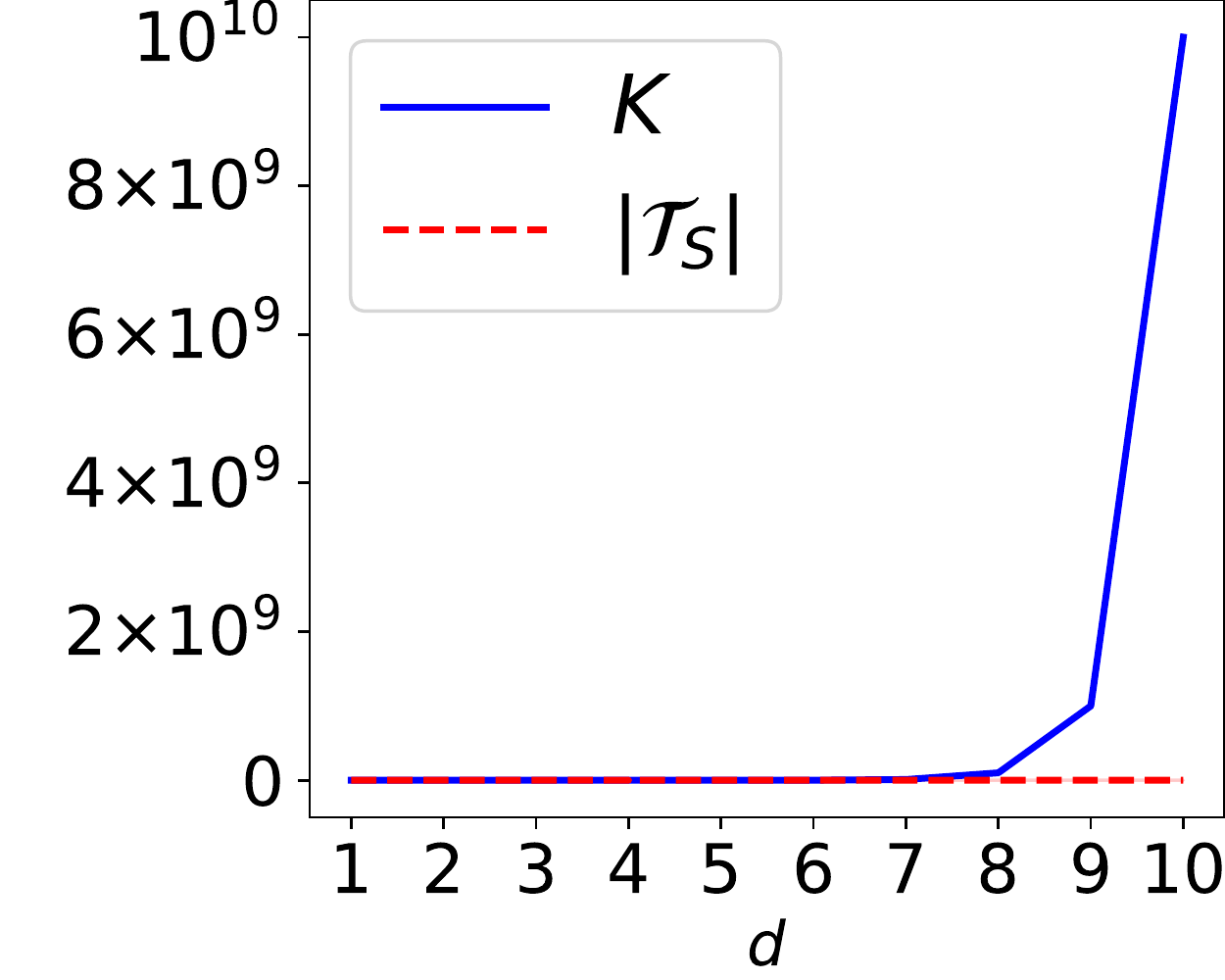}
  \vspace{-15pt}
  \caption{Beta mix (0.1, 10)-(0.01)} 
  \vspace{10pt}
\end{subfigure}
\begin{subfigure}[b]{0.24\textwidth}
  \includegraphics[width=\textwidth, height=0.8\textwidth]{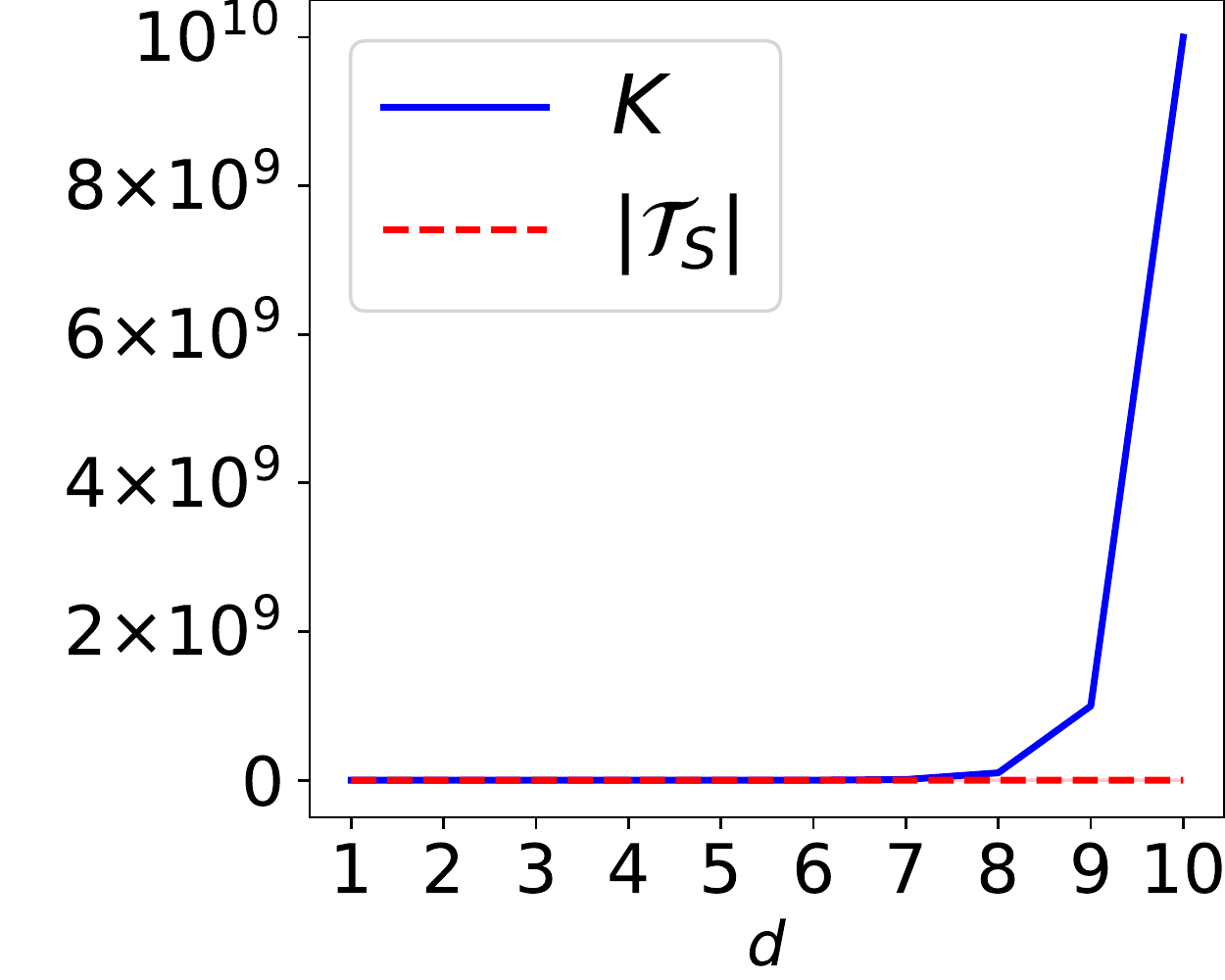}
  \vspace{-15pt}
  \caption{Beta mix (0.1, 10)-(0.1)} 
  \vspace{10pt}
\end{subfigure}
\begin{subfigure}[b]{0.24\textwidth}
  \includegraphics[width=\textwidth, height=0.8\textwidth]{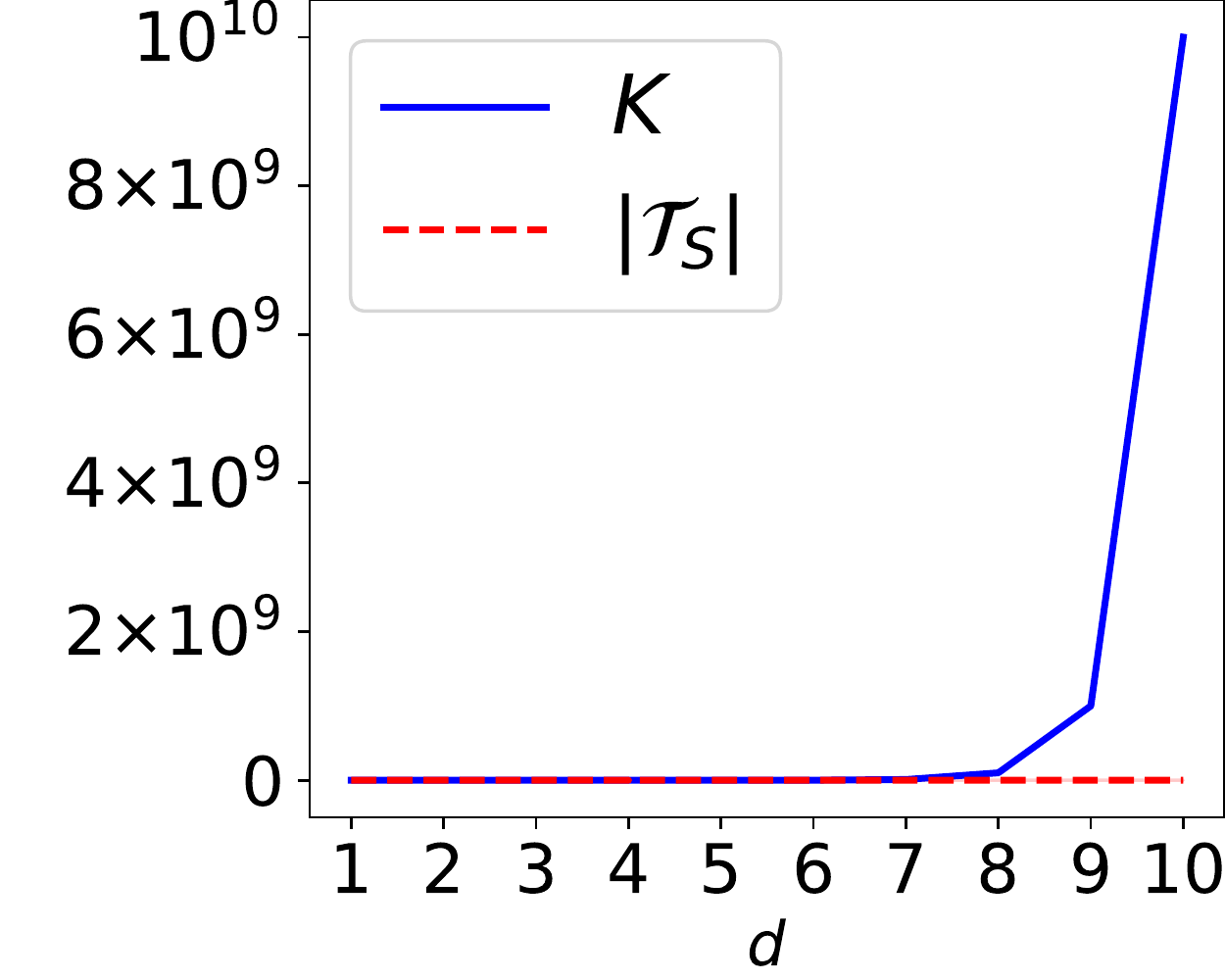}
  \vspace{-15pt}
  \caption{Beta mix (0.1, 10)-(10)} 
  \vspace{10pt}
\end{subfigure}
\begin{subfigure}[b]{0.24\textwidth}
  \includegraphics[width=\textwidth, height=0.8\textwidth]{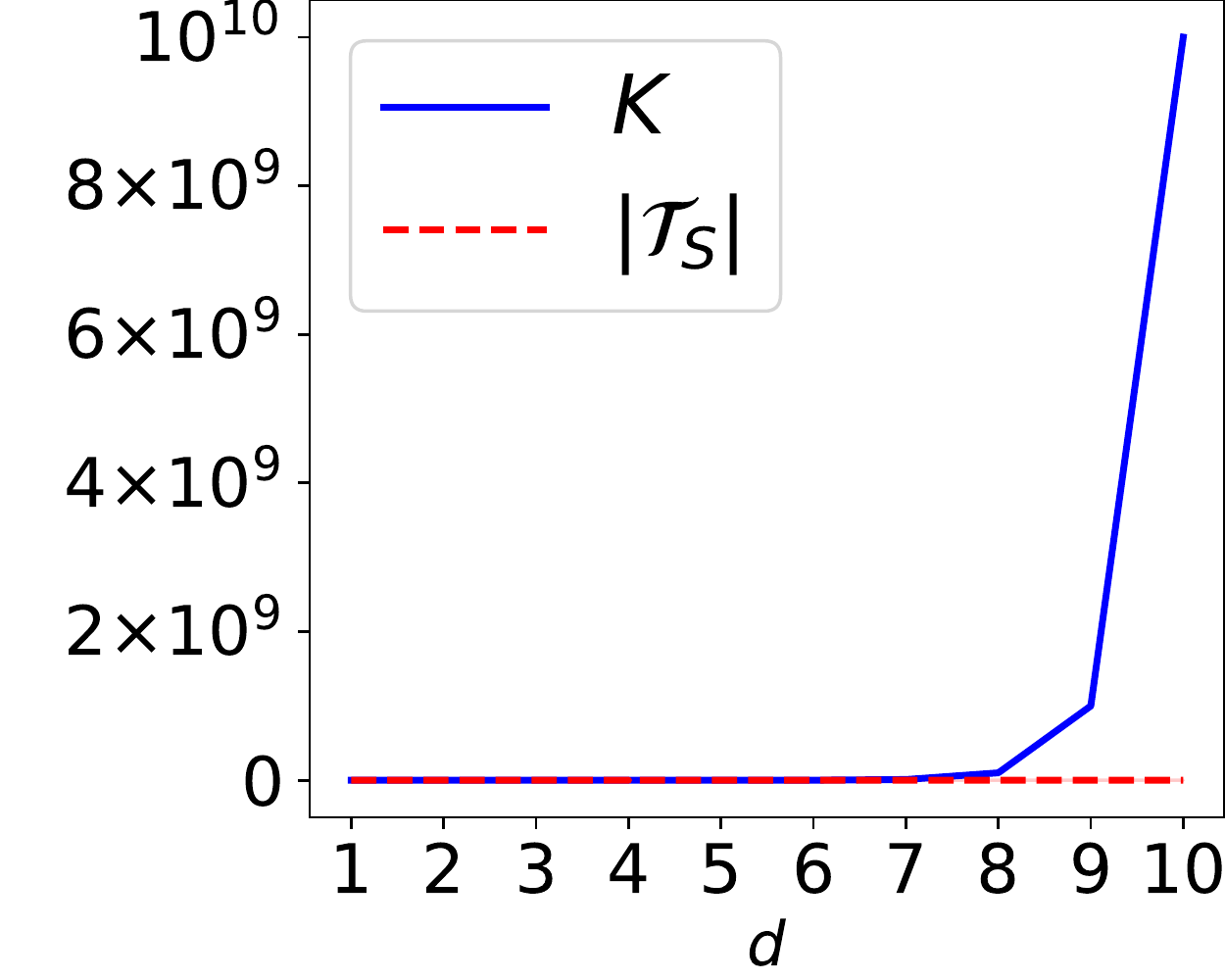}
  \vspace{-15pt}
  \caption{beta-Gauss (0.1,0.1)-(0.01)} 
  \vspace{10pt}
\end{subfigure}
\begin{subfigure}[b]{0.24\textwidth}
  \includegraphics[width=\textwidth, height=0.8\textwidth]{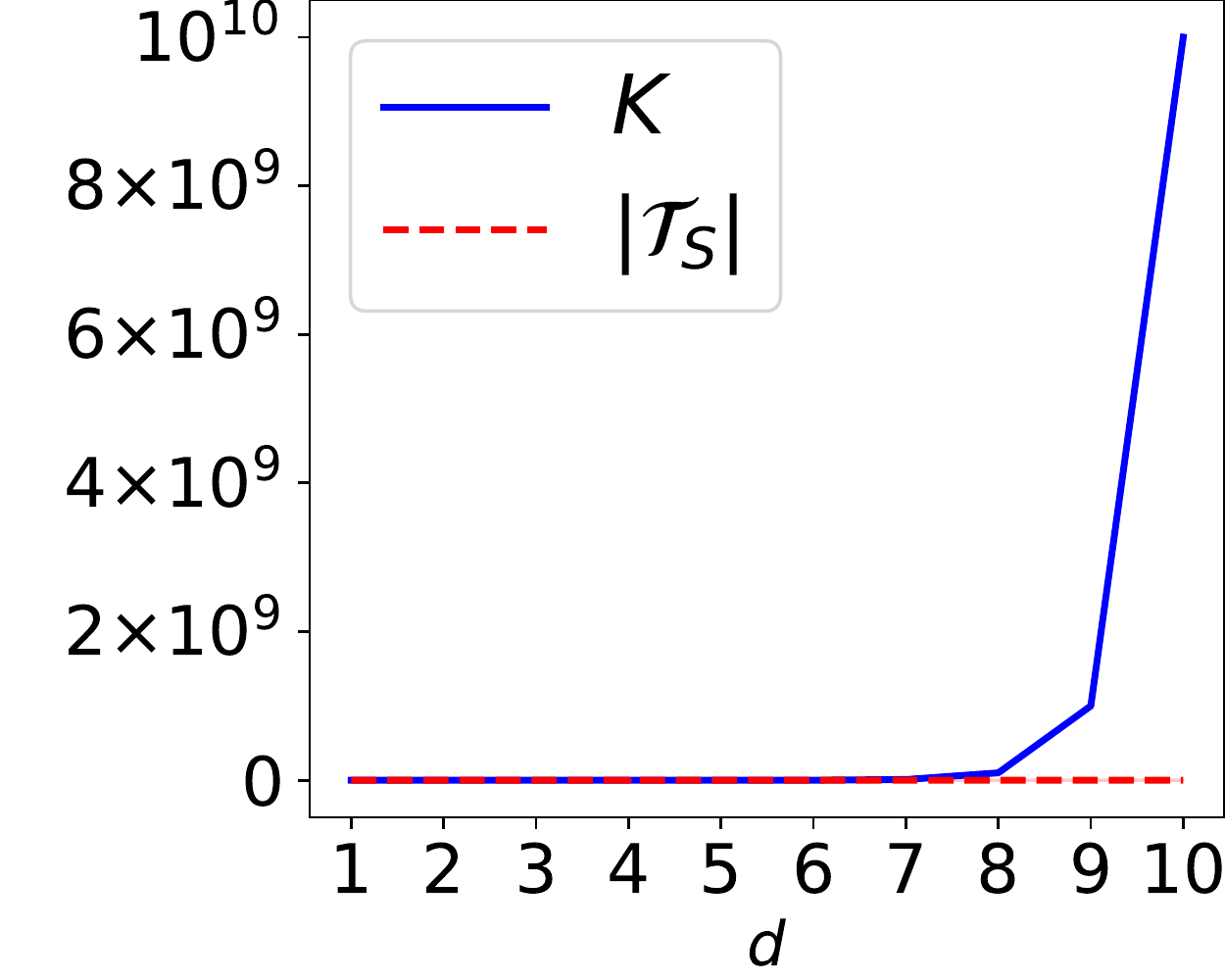}
  \vspace{-15pt}
  \caption{beta-Gauss (0.1,0.1)-(0.1)} 
  \vspace{10pt}
\end{subfigure}
\begin{subfigure}[b]{0.24\textwidth}
  \includegraphics[width=\textwidth, height=0.8\textwidth]{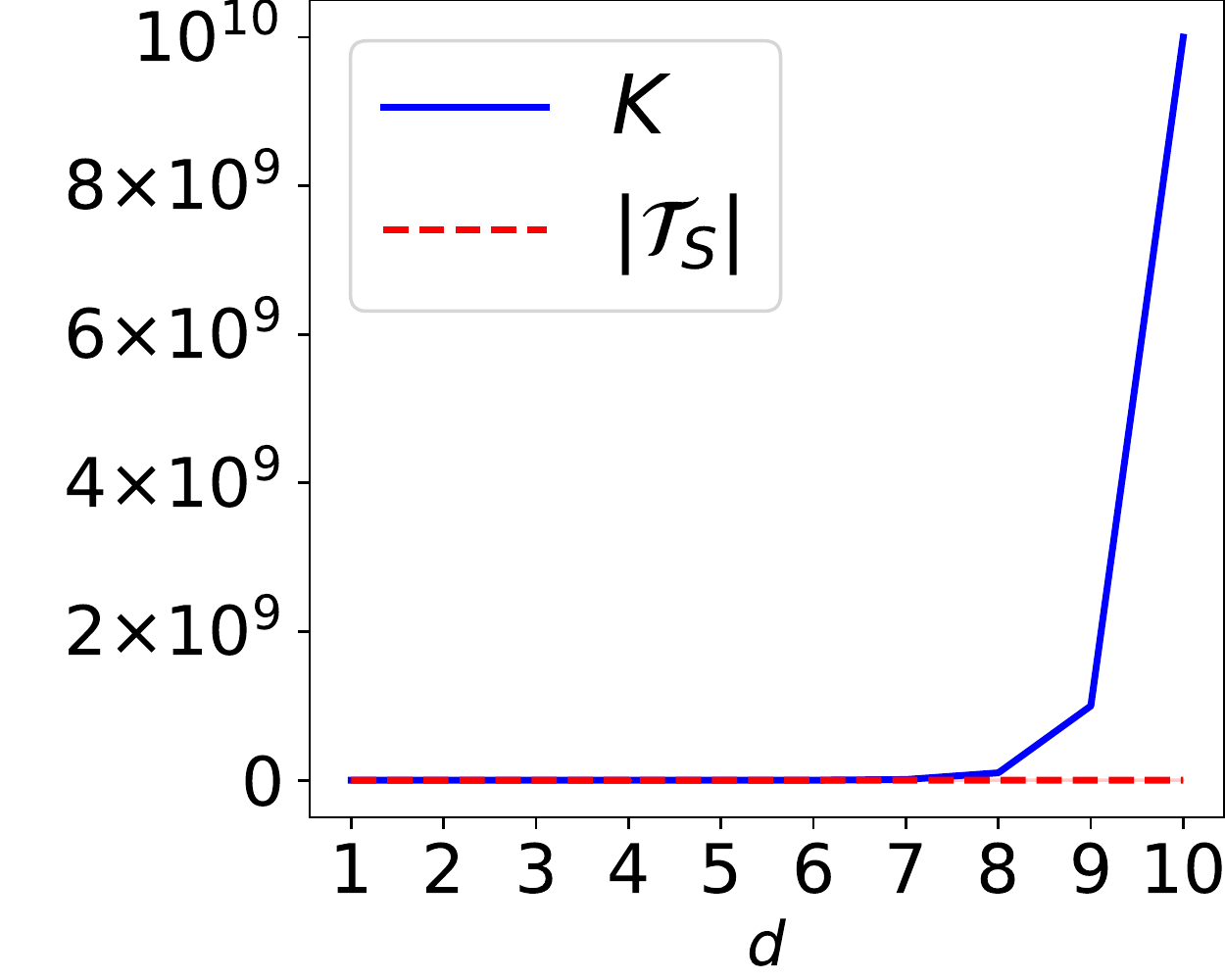}
  \vspace{-15pt}
  \caption{beta-Gauss (0.1,0.1)-(10)} 
  \vspace{10pt}
\end{subfigure}
\begin{subfigure}[b]{0.24\textwidth}
  \includegraphics[width=\textwidth, height=0.8\textwidth]{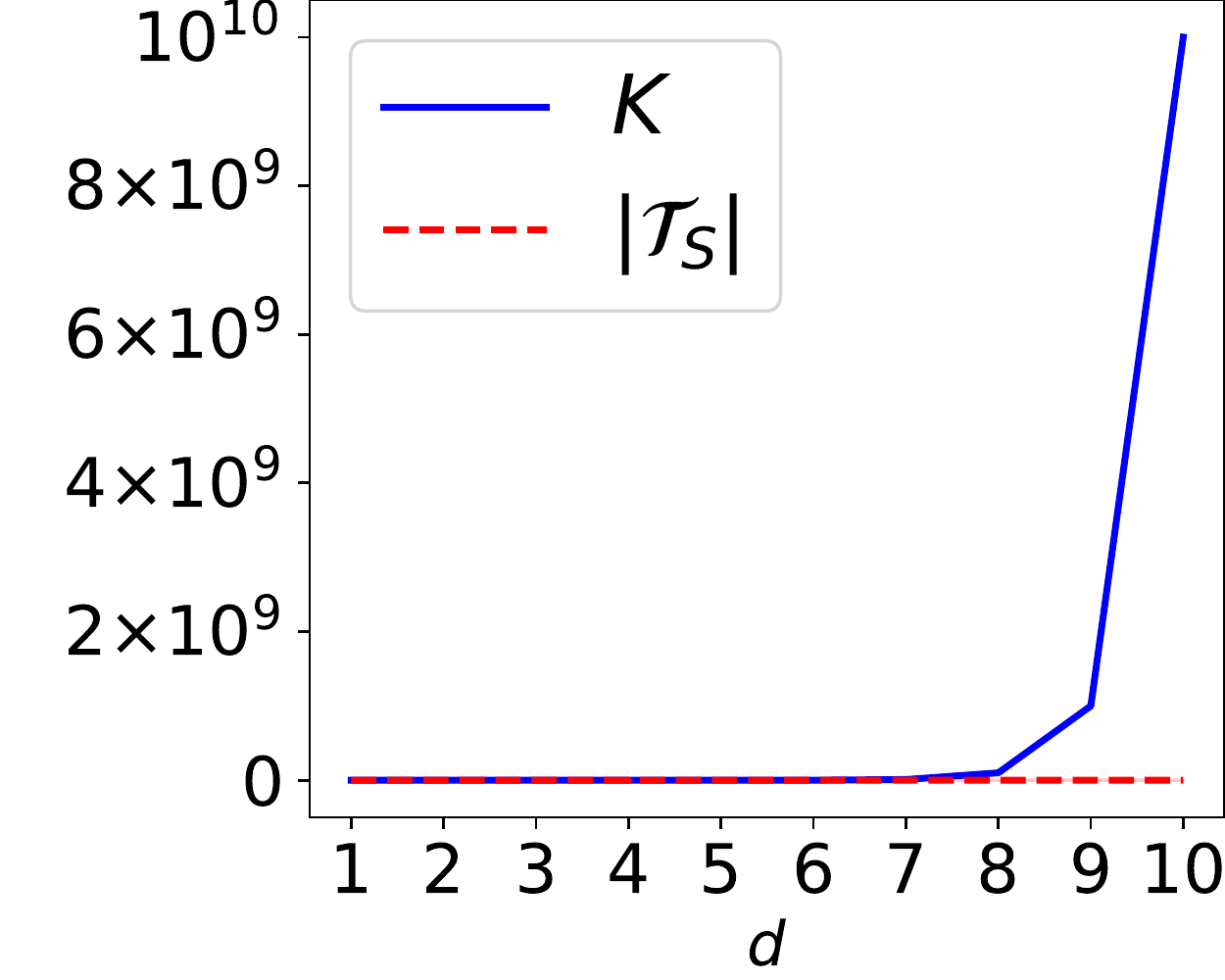}
  \vspace{-15pt}
  \caption{beta-Gauss (0.1,10)-(0.01)} 
  \vspace{10pt}
\end{subfigure}
\begin{subfigure}[b]{0.24\textwidth}
  \includegraphics[width=\textwidth, height=0.8\textwidth]{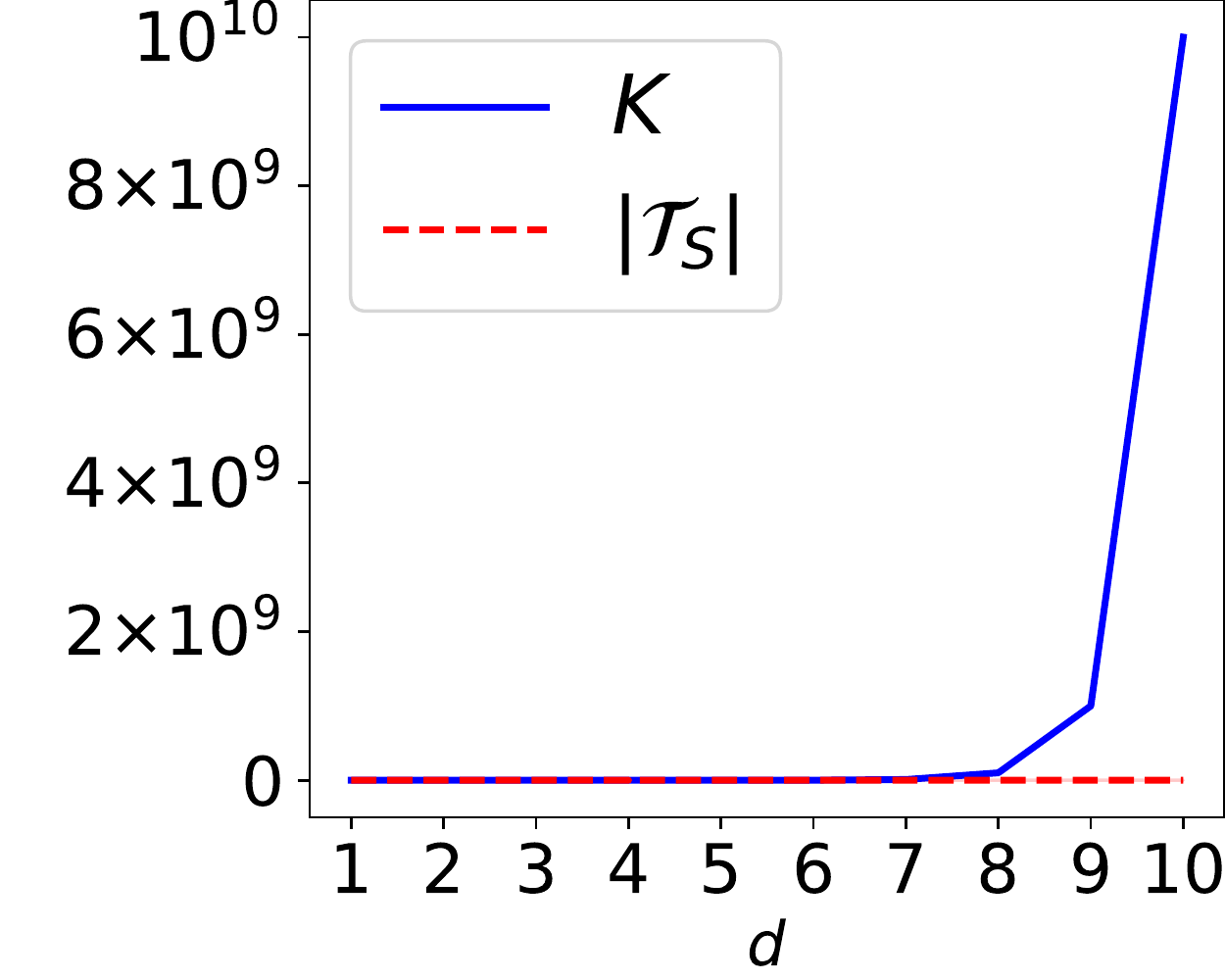}
  \vspace{-15pt}
  \caption{beta-Gauss (0.1,10)-(0.1)} 
  \vspace{10pt}
\end{subfigure}
\begin{subfigure}[b]{0.24\textwidth}
  \includegraphics[width=\textwidth, height=0.8\textwidth]{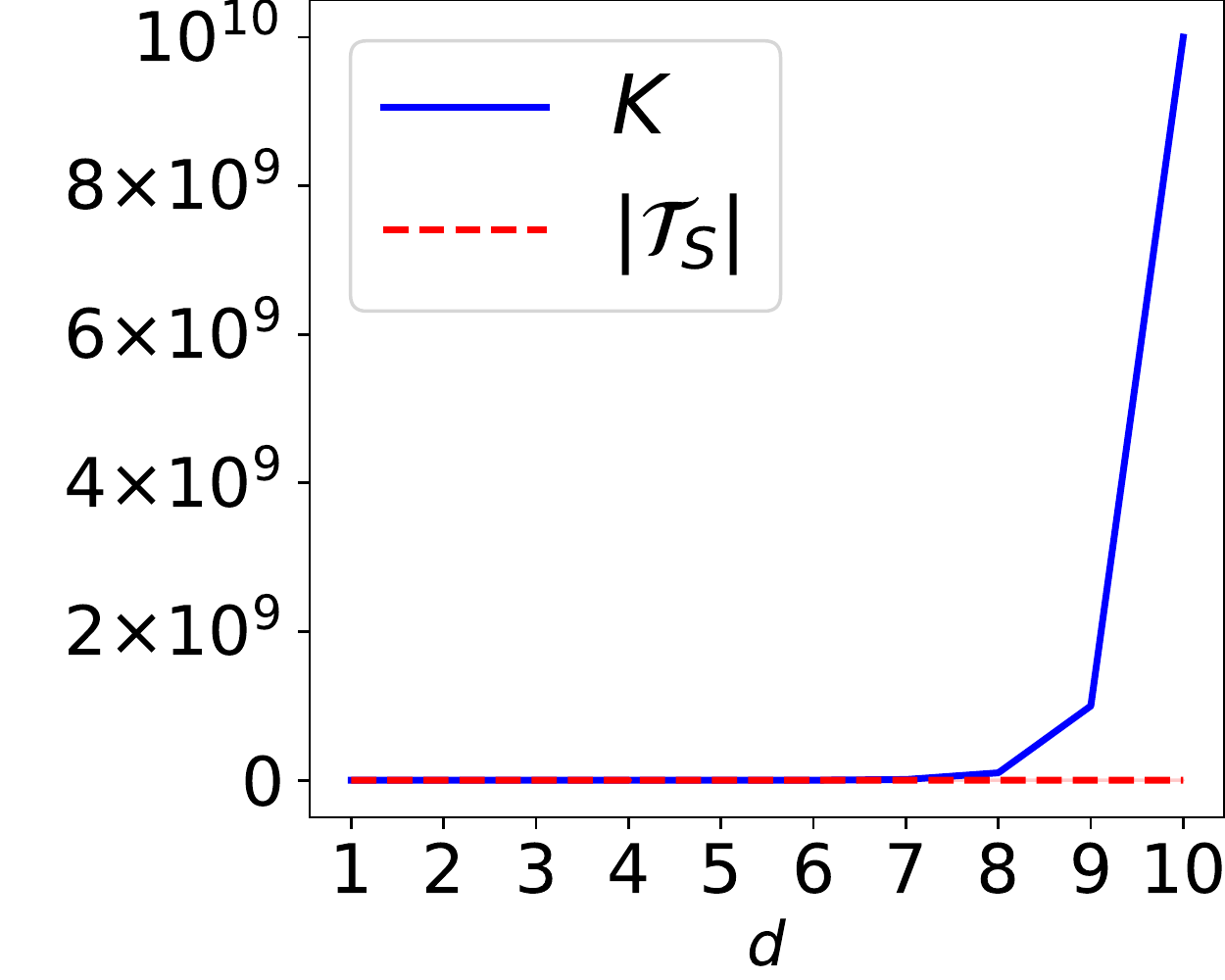}
  \vspace{-15pt}
  \caption{beta-Gauss (0.1,10)-(10)} 
  \vspace{10pt}
\end{subfigure}
\begin{subfigure}[b]{0.24\textwidth}
  \includegraphics[width=\textwidth, height=0.8\textwidth]{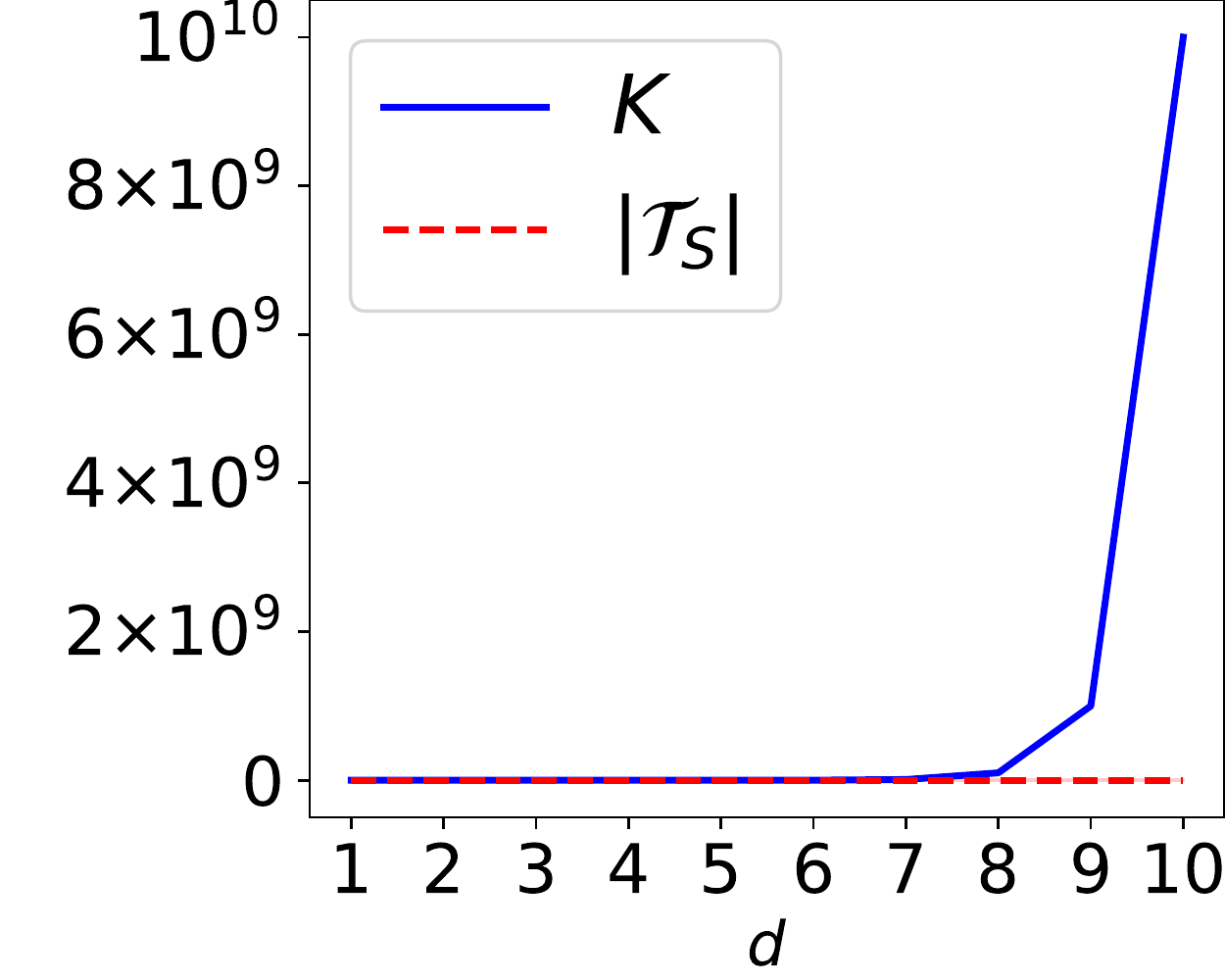}
  \vspace{-15pt}
  \caption{Gauss mix (0.001)} 
  \vspace{10pt}
\end{subfigure}
\begin{subfigure}[b]{0.24\textwidth}
  \includegraphics[width=\textwidth, height=0.8\textwidth]{fig/results_1/base_blobs_1000_001_01_100_01_0.pdf}
  \vspace{-15pt}
  \caption{Gauss mix (0.01)} 
  \vspace{10pt}
\end{subfigure}
\begin{subfigure}[b]{0.24\textwidth}
  \includegraphics[width=\textwidth, height=0.8\textwidth]{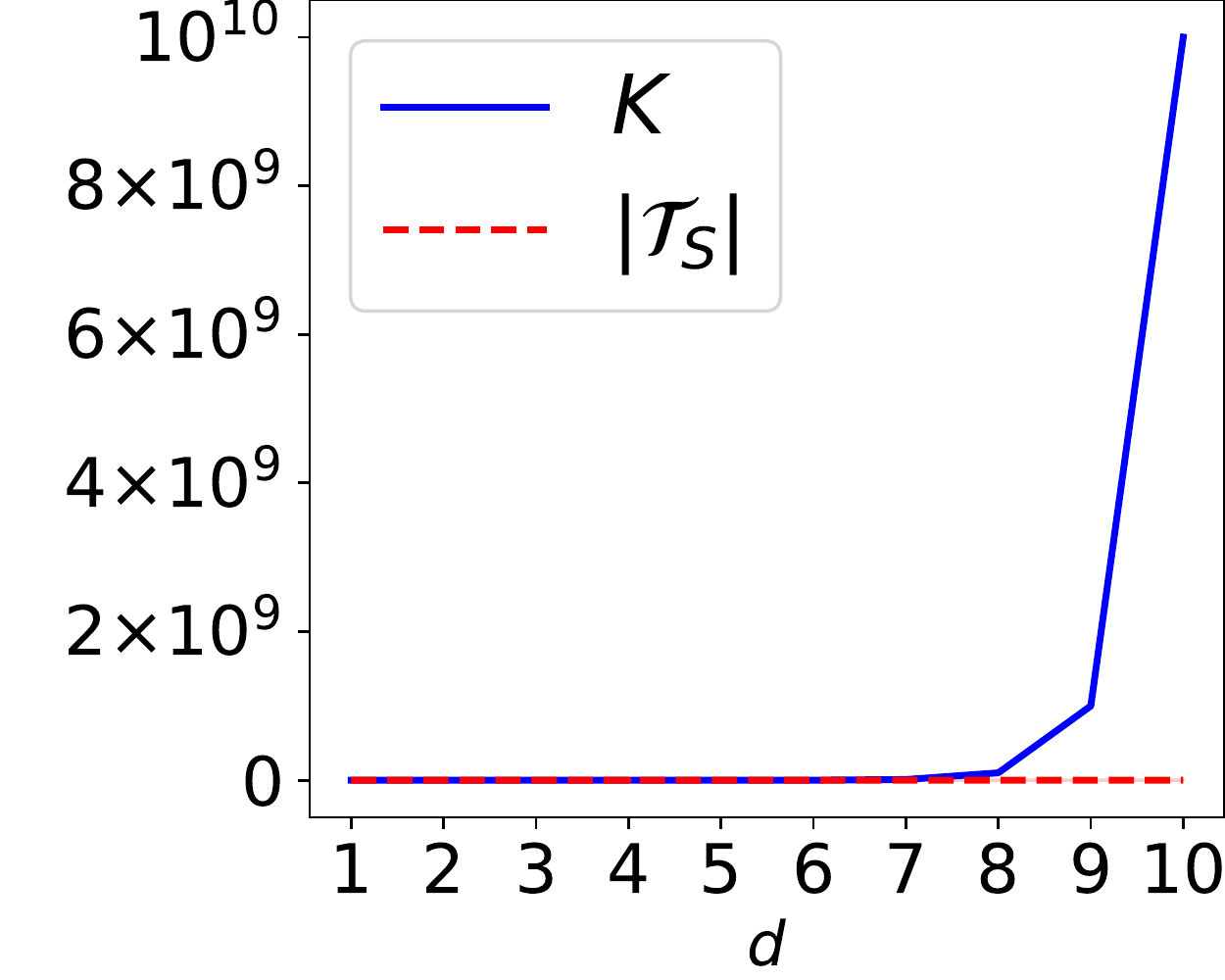}
  \vspace{-15pt}
  \caption{Gauss mix (0.1)} 
  \vspace{10pt}
\end{subfigure}
\begin{subfigure}[b]{0.24\textwidth}
  \includegraphics[width=\textwidth, height=0.8\textwidth]{fig/results_1/base_blobs_1000_10_01_100_01_0.pdf}
  \vspace{-15pt}
  \caption{Gauss mix (1.0)} 
  \vspace{10pt}
\end{subfigure}
\begin{subfigure}[b]{0.24\textwidth}
  \includegraphics[width=\textwidth, height=0.8\textwidth]{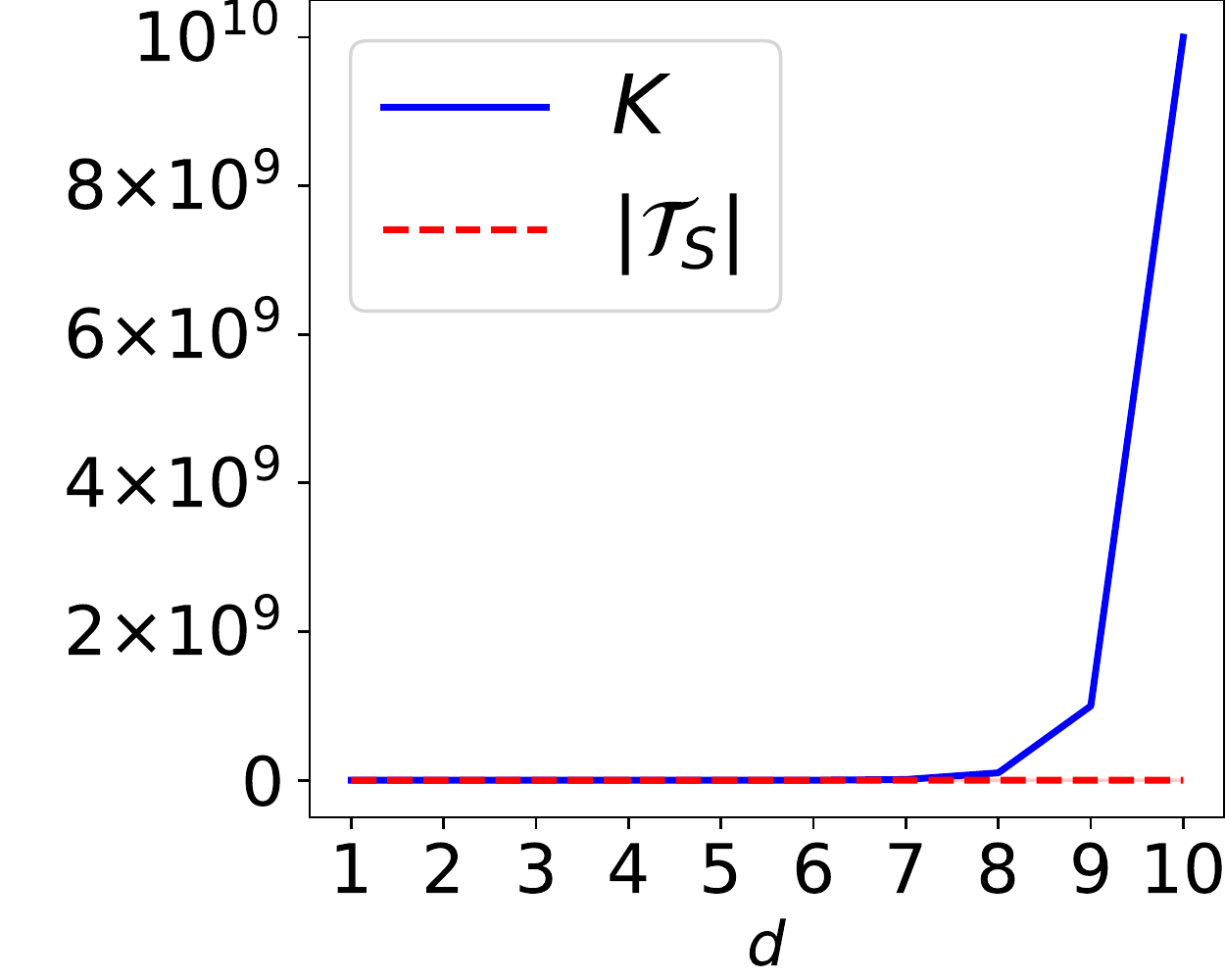}
  \vspace{-15pt}
  \caption{Gauss mix (10.0)} 
  \vspace{10pt}
\end{subfigure}
\caption{The values of $K$ versus $|\Tcal_S|$ with  the $\epsilon$-covering of the original space. The figures display the mean of 10 random trials along with one standard deviation.} \label{fig:app:1}
\end{figure}

\begin{figure}[!t]
\center
\begin{subfigure}[b]{0.24\textwidth}
  \includegraphics[width=\textwidth, height=0.8\textwidth]{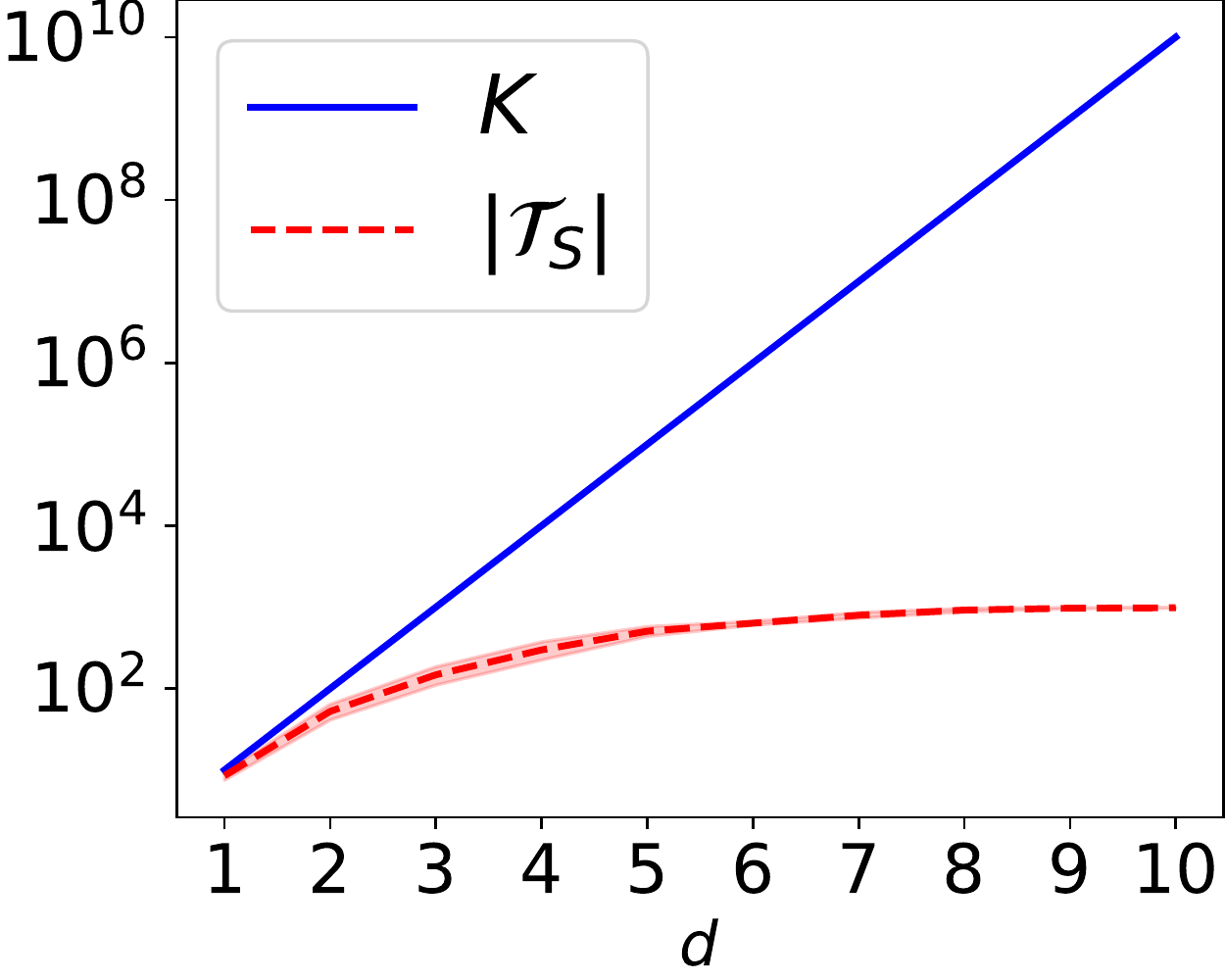}
  \vspace{-15pt}
  \caption{Beta(0.1, 0.1)} 
  \vspace{10pt}
\end{subfigure}
\begin{subfigure}[b]{0.24\textwidth}
  \includegraphics[width=\textwidth, height=0.8\textwidth]{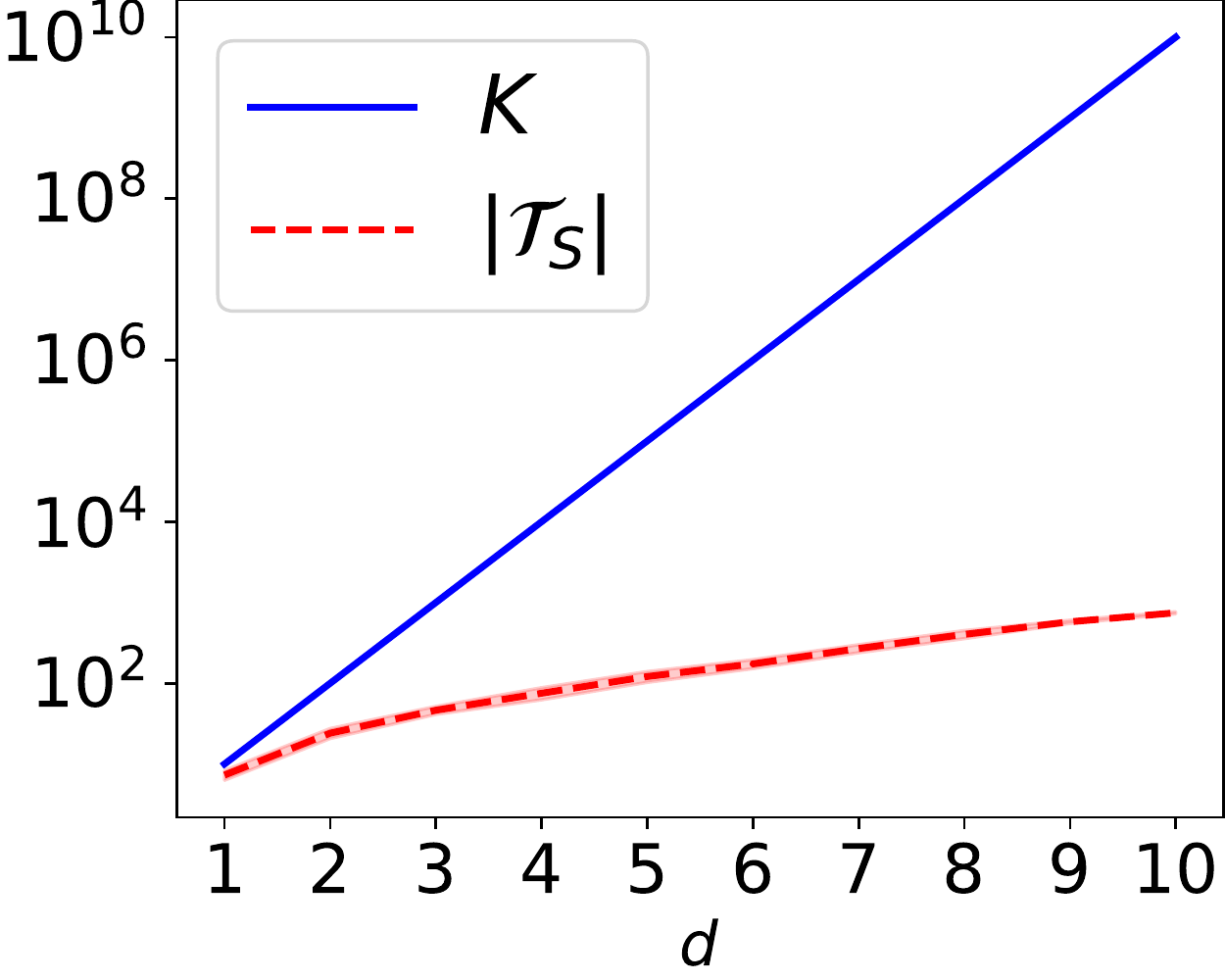}
  \vspace{-15pt}
  \caption{Beta(0.01, 0.01)} 
  \vspace{10pt}
\end{subfigure}
\begin{subfigure}[b]{0.24\textwidth}
  \includegraphics[width=\textwidth, height=0.8\textwidth]{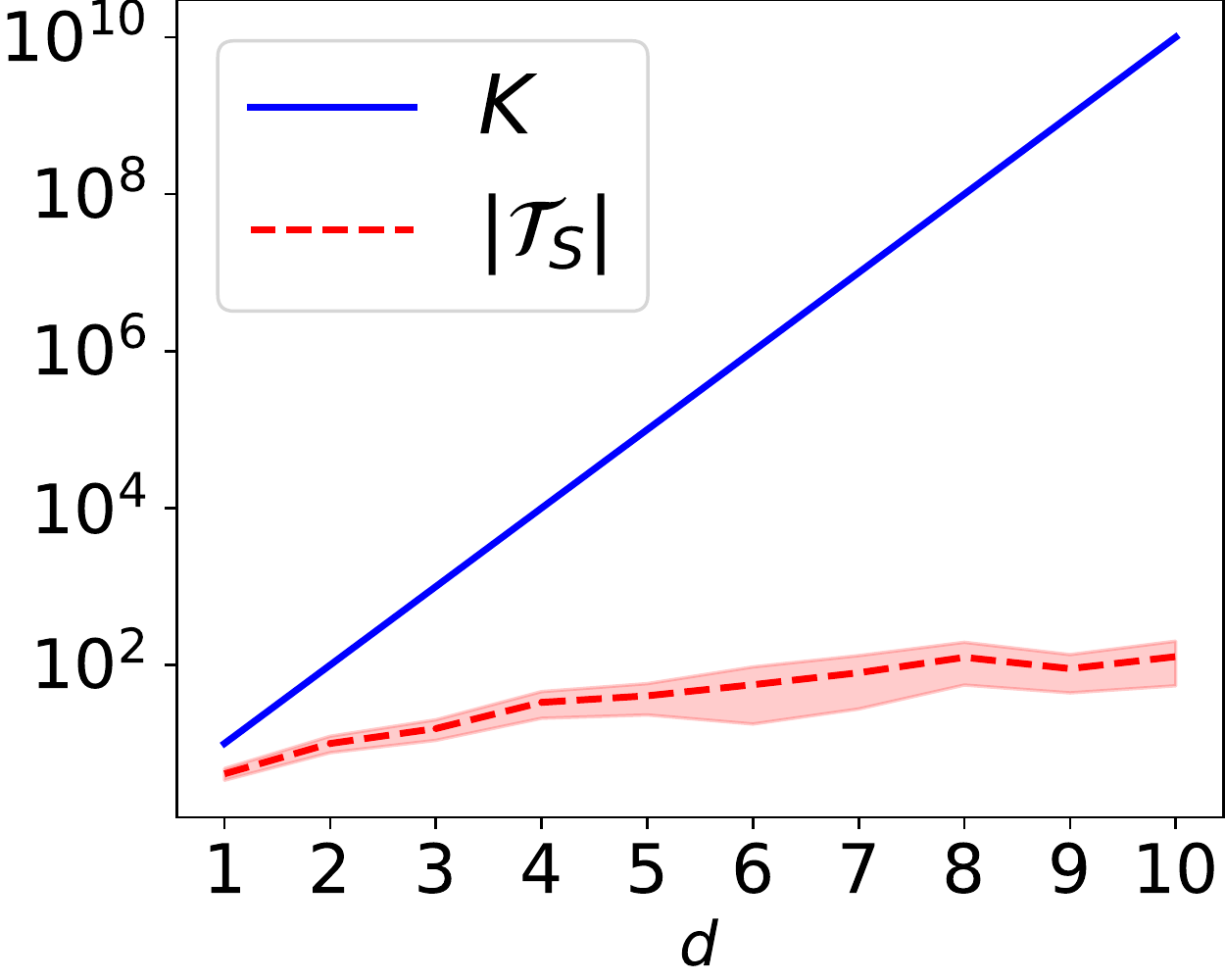}
  \vspace{-15pt}
  \caption{Beta(0.1, 10)} 
  \vspace{10pt}
\end{subfigure}
\begin{subfigure}[b]{0.24\textwidth}
  \includegraphics[width=\textwidth, height=0.8\textwidth]{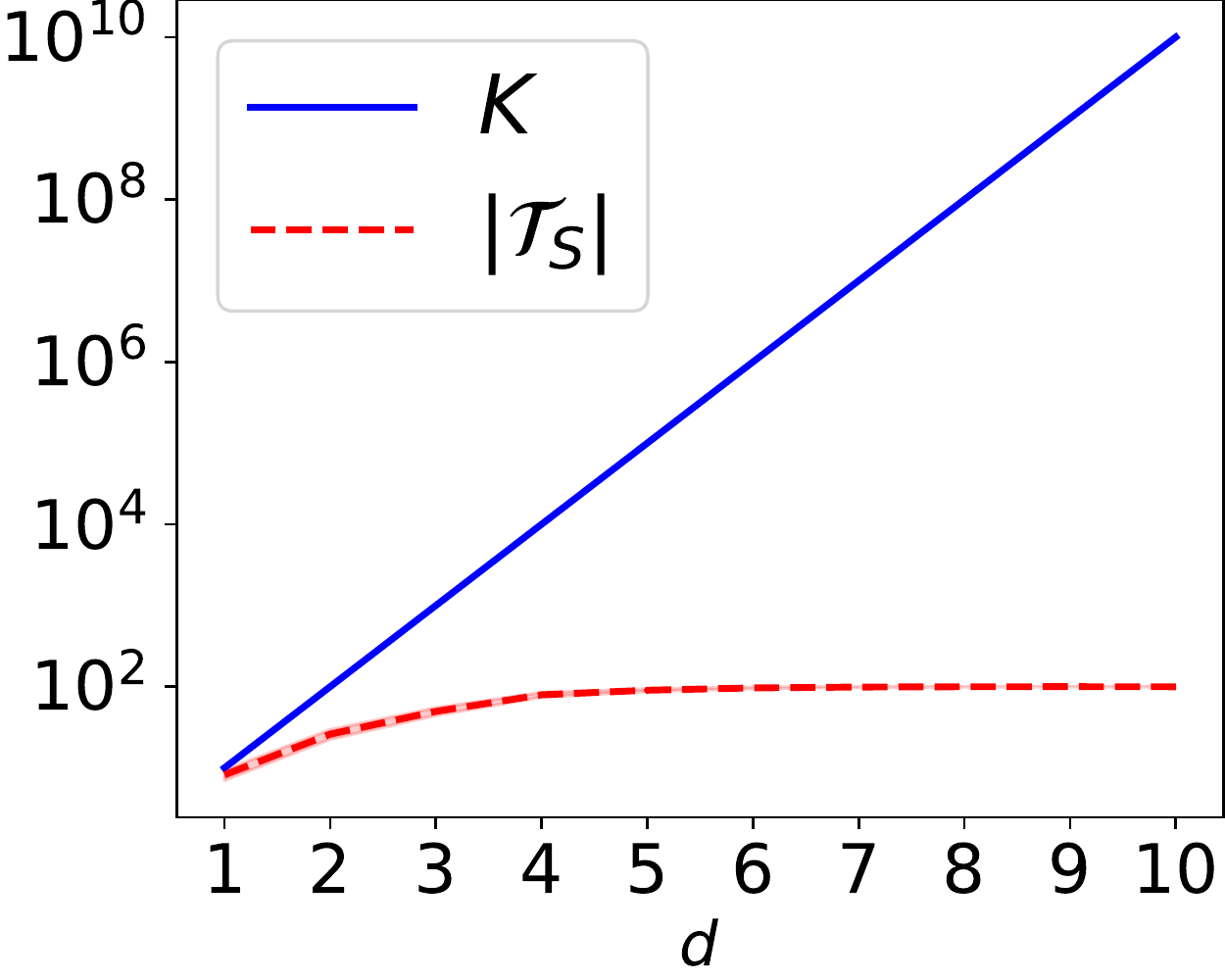}
  \vspace{-15pt}
  \caption{Beta mix (0.1, 0.1)-(0.01)} 
  \vspace{10pt}
\end{subfigure}
\begin{subfigure}[b]{0.24\textwidth}
  \includegraphics[width=\textwidth, height=0.8\textwidth]{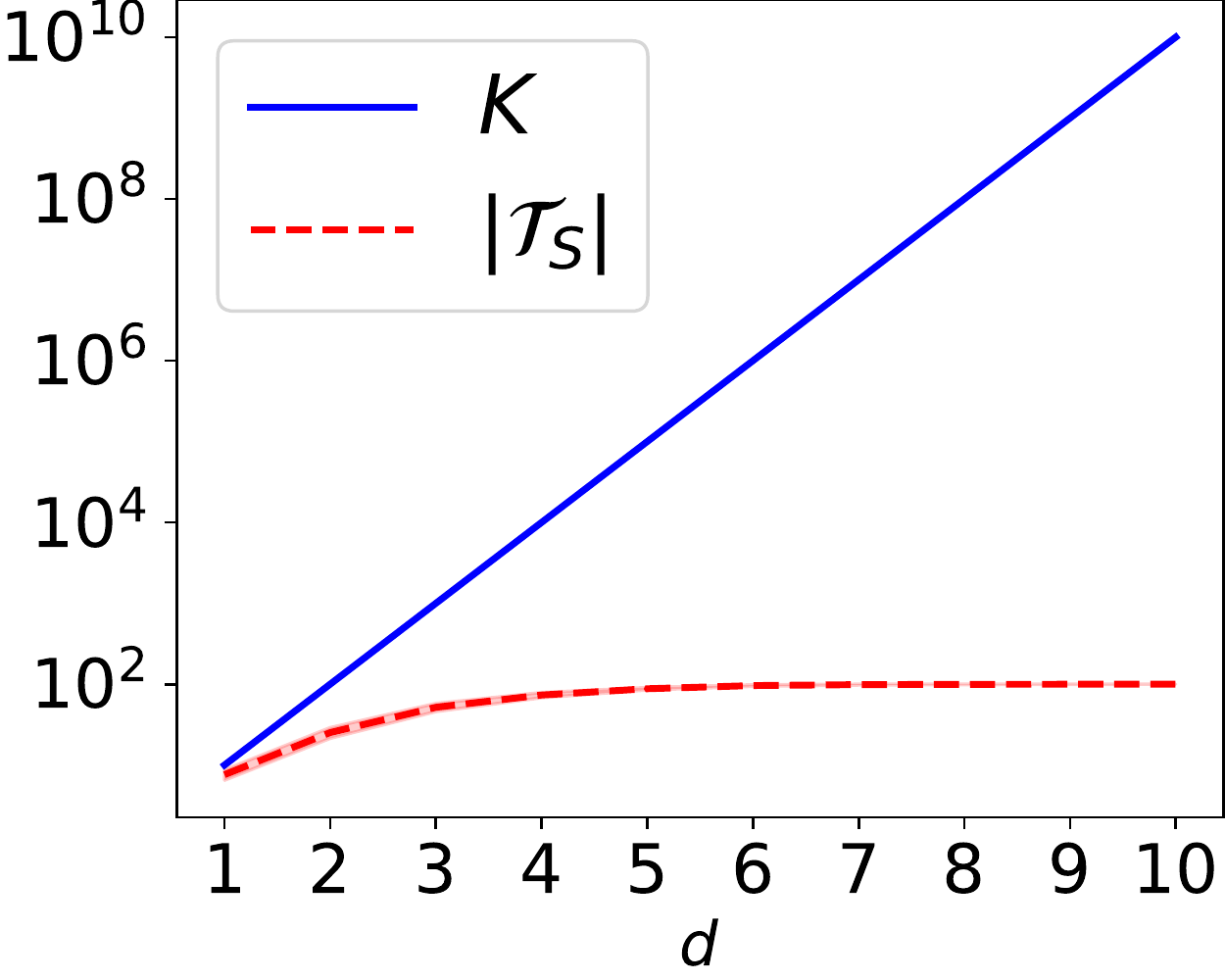}
  \vspace{-15pt}
  \caption{Beta mix (0.1, 0.1)-(0.1)} 
  \vspace{10pt}
\end{subfigure}
\begin{subfigure}[b]{0.24\textwidth}
  \includegraphics[width=\textwidth, height=0.8\textwidth]{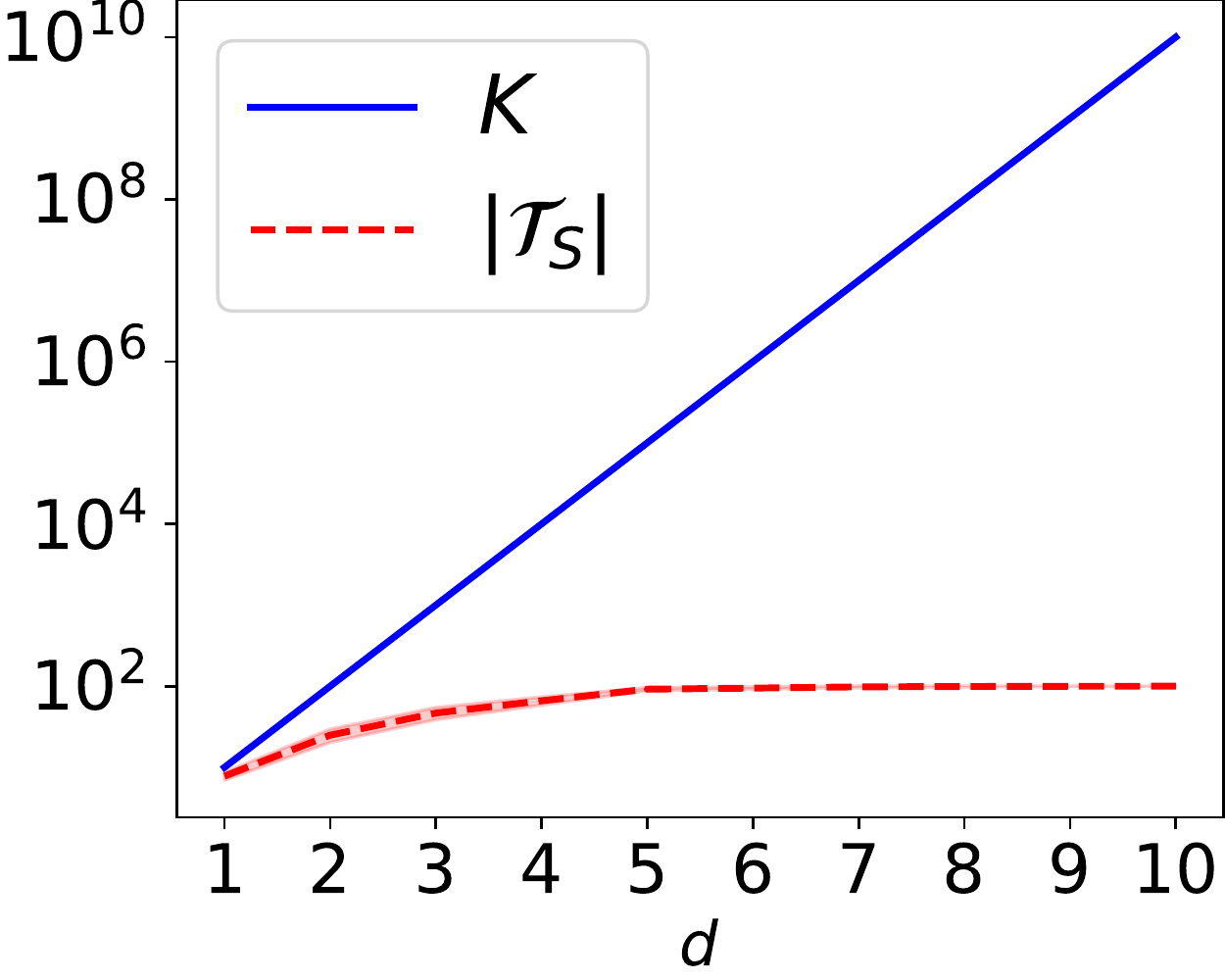}
  \vspace{-15pt}
  \caption{Beta mix (0.1, 0.1)-(10)} 
  \vspace{10pt}
\end{subfigure}
\begin{subfigure}[b]{0.24\textwidth}
  \includegraphics[width=\textwidth, height=0.8\textwidth]{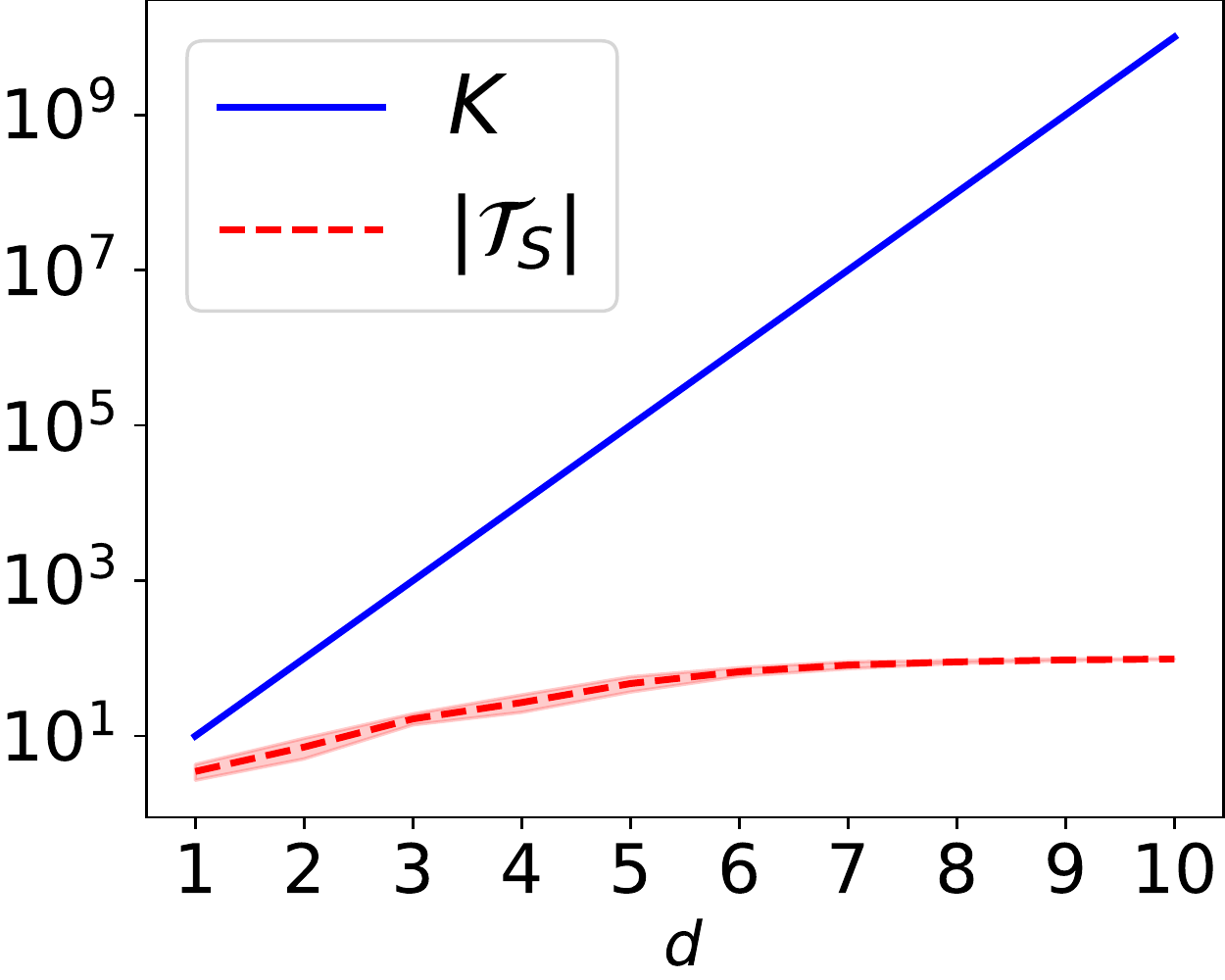}
  \vspace{-15pt}
  \caption{Beta mix (0.1, 10)-(0.01)} 
  \vspace{10pt}
\end{subfigure}
\begin{subfigure}[b]{0.24\textwidth}
  \includegraphics[width=\textwidth, height=0.8\textwidth]{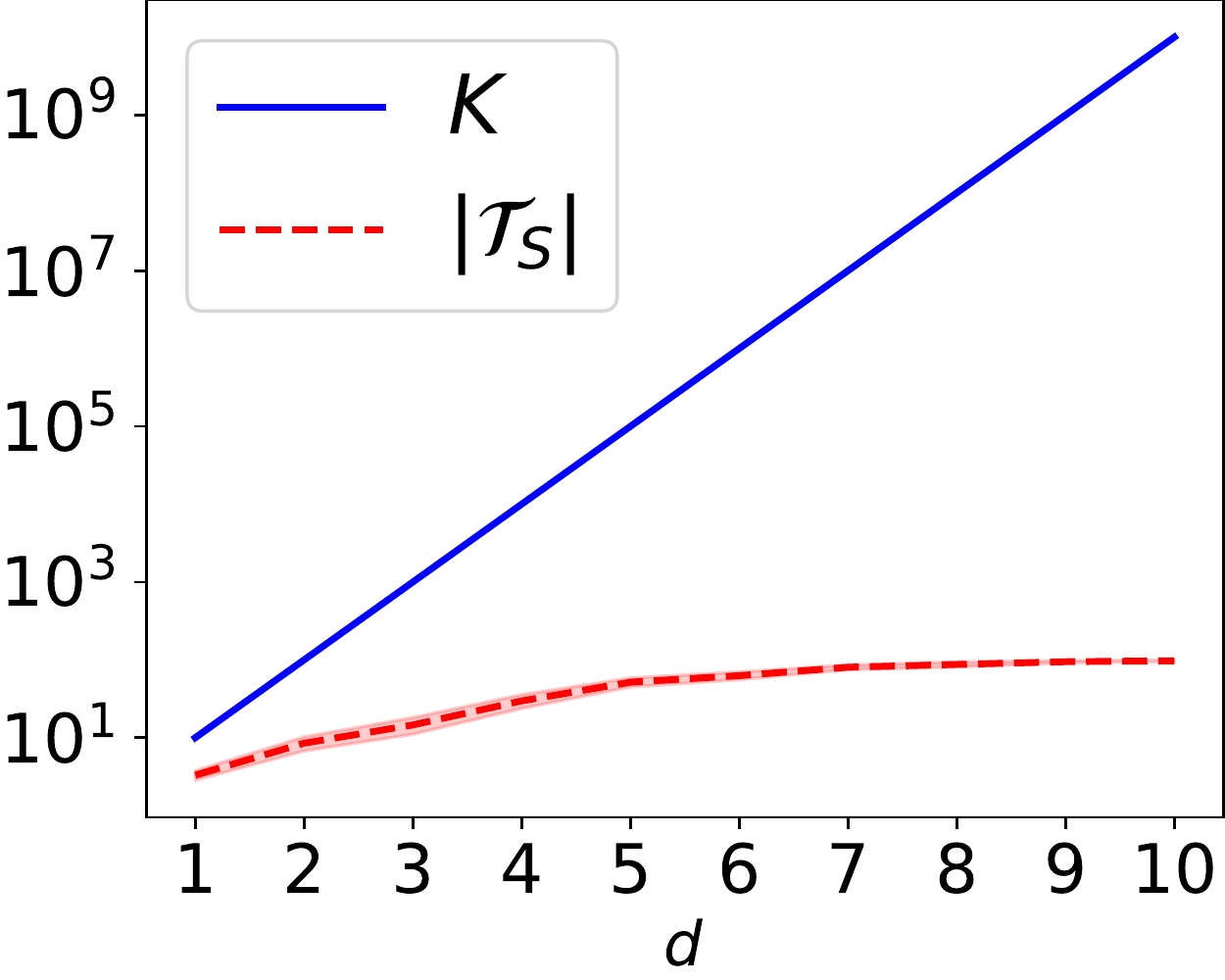}
  \vspace{-15pt}
  \caption{Beta mix (0.1, 10)-(0.1)} 
  \vspace{10pt}
\end{subfigure}
\begin{subfigure}[b]{0.24\textwidth}
  \includegraphics[width=\textwidth, height=0.8\textwidth]{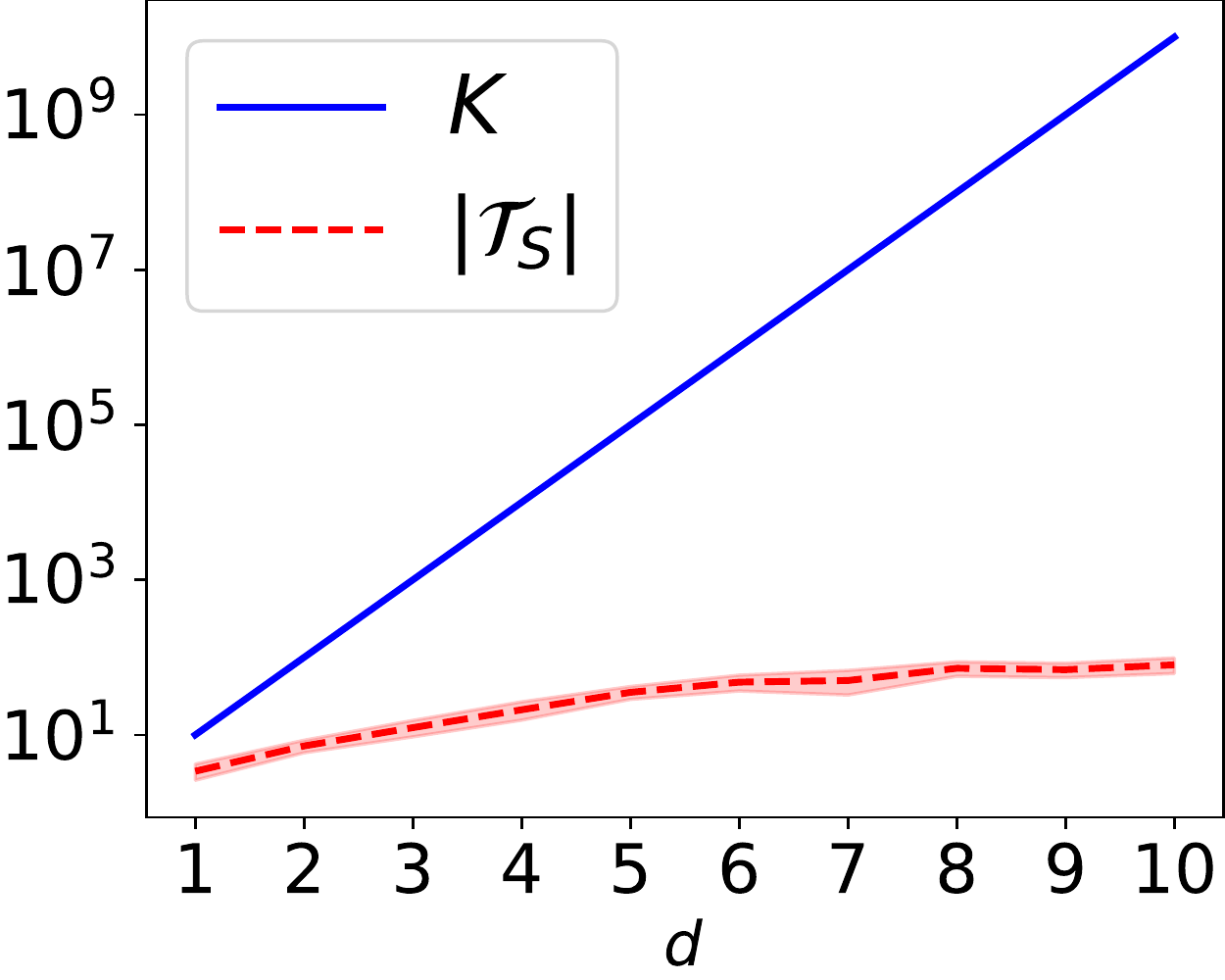}
  \vspace{-15pt}
  \caption{Beta mix (0.1, 10)-(10)} 
  \vspace{10pt}
\end{subfigure}
\begin{subfigure}[b]{0.24\textwidth}
  \includegraphics[width=\textwidth, height=0.8\textwidth]{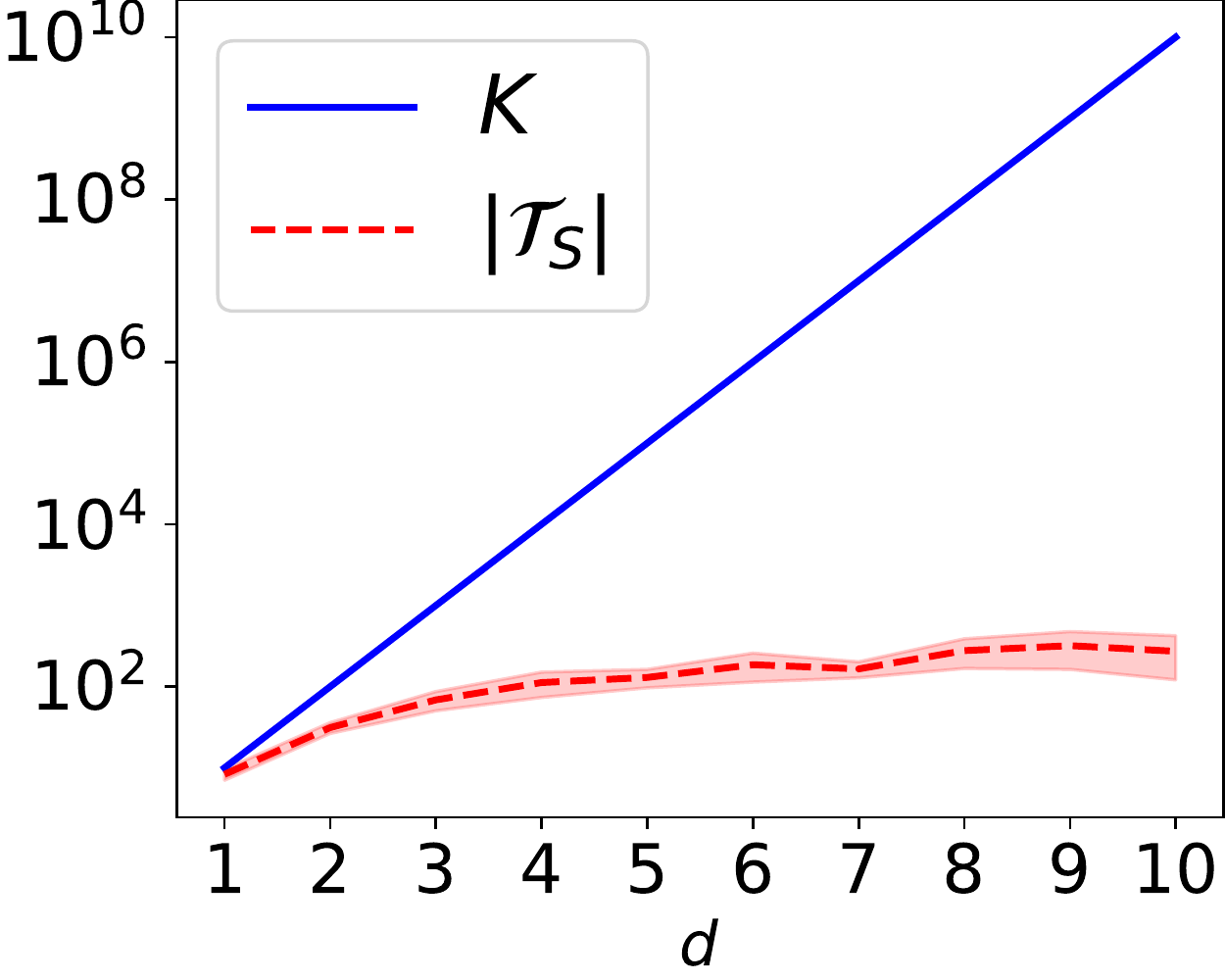}
  \vspace{-15pt}
  \caption{beta-Gauss (0.1,0.1)-(0.01)} 
  \vspace{10pt}
\end{subfigure}
\begin{subfigure}[b]{0.24\textwidth}
  \includegraphics[width=\textwidth, height=0.8\textwidth]{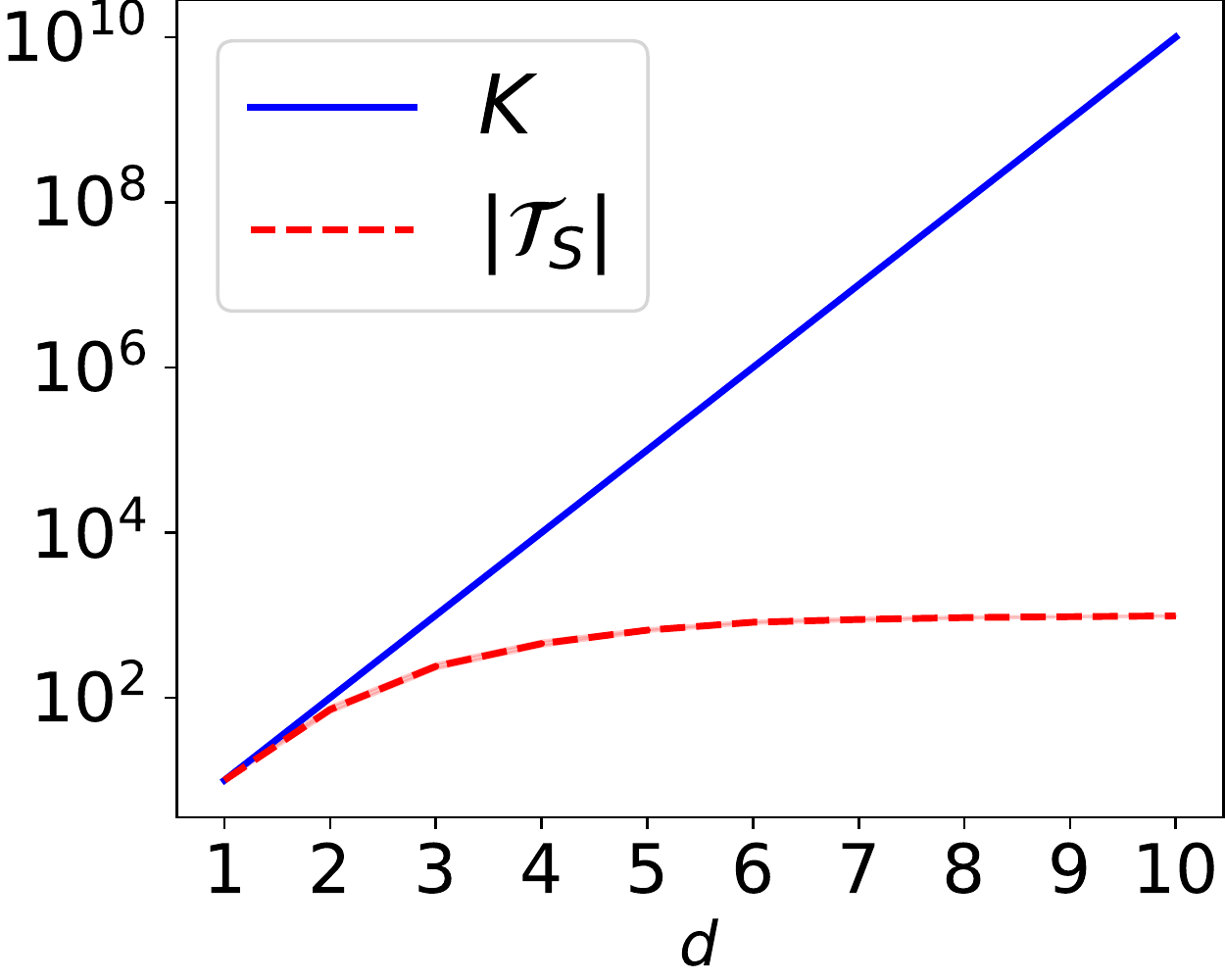}
  \vspace{-15pt}
  \caption{beta-Gauss (0.1,0.1)-(0.1)} 
  \vspace{10pt}
\end{subfigure}
\begin{subfigure}[b]{0.24\textwidth}
  \includegraphics[width=\textwidth, height=0.8\textwidth]{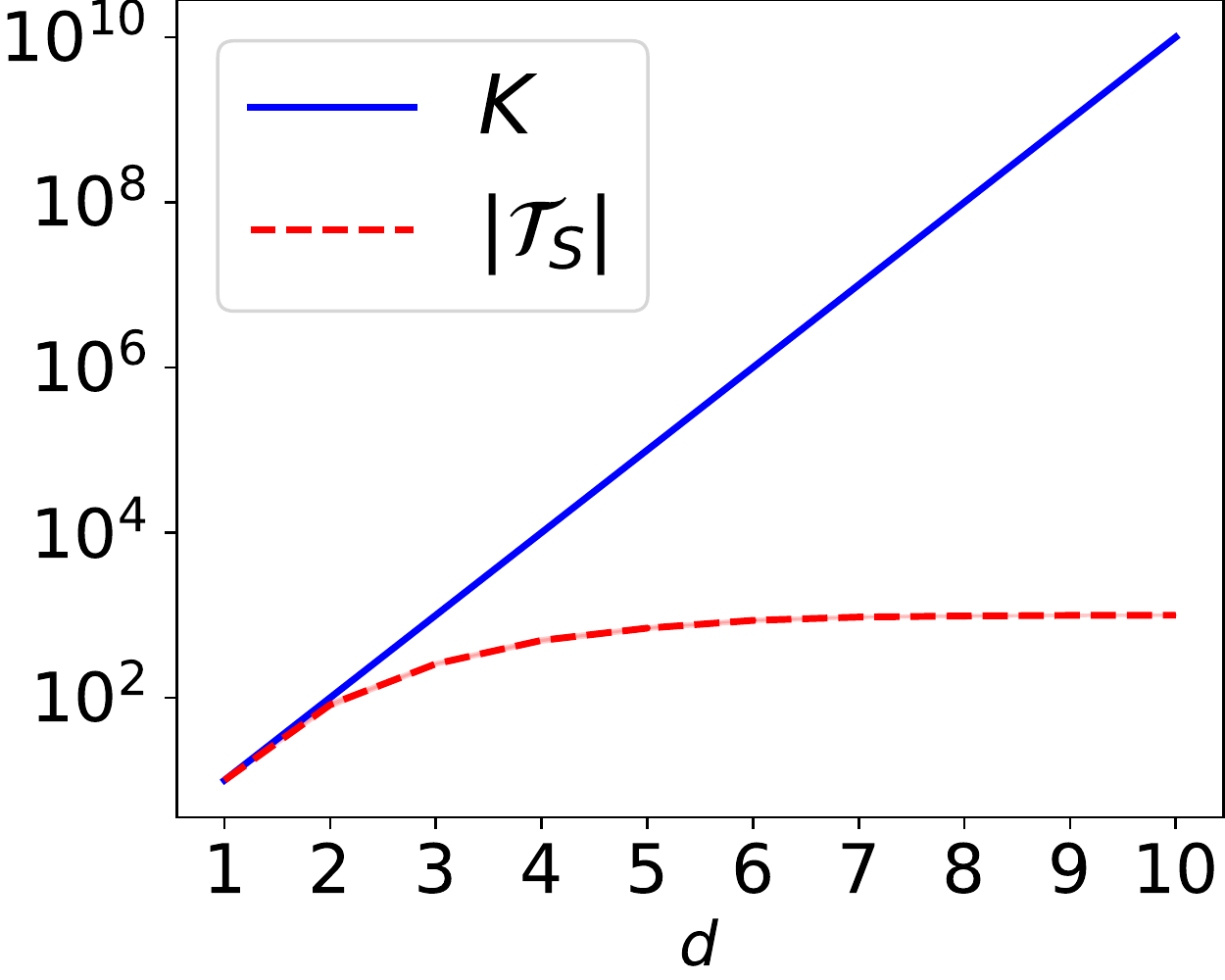}
  \vspace{-15pt}
  \caption{beta-Gauss (0.1,0.1)-(10)} 
  \vspace{10pt}
\end{subfigure}
\begin{subfigure}[b]{0.24\textwidth}
  \includegraphics[width=\textwidth, height=0.8\textwidth]{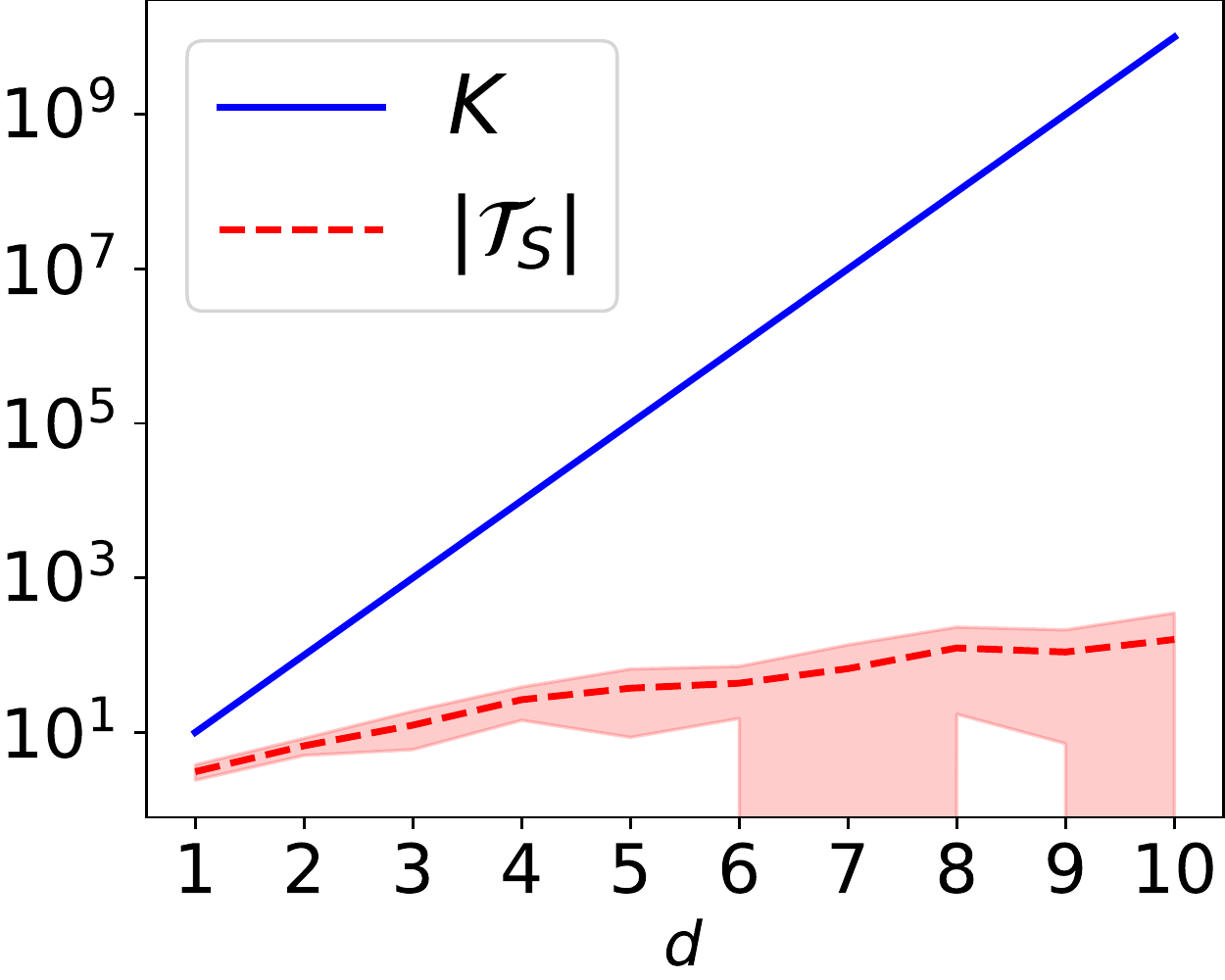}
  \vspace{-15pt}
  \caption{beta-Gauss (0.1,10)-(0.01)} 
  \vspace{10pt}
\end{subfigure}
\begin{subfigure}[b]{0.24\textwidth}
  \includegraphics[width=\textwidth, height=0.8\textwidth]{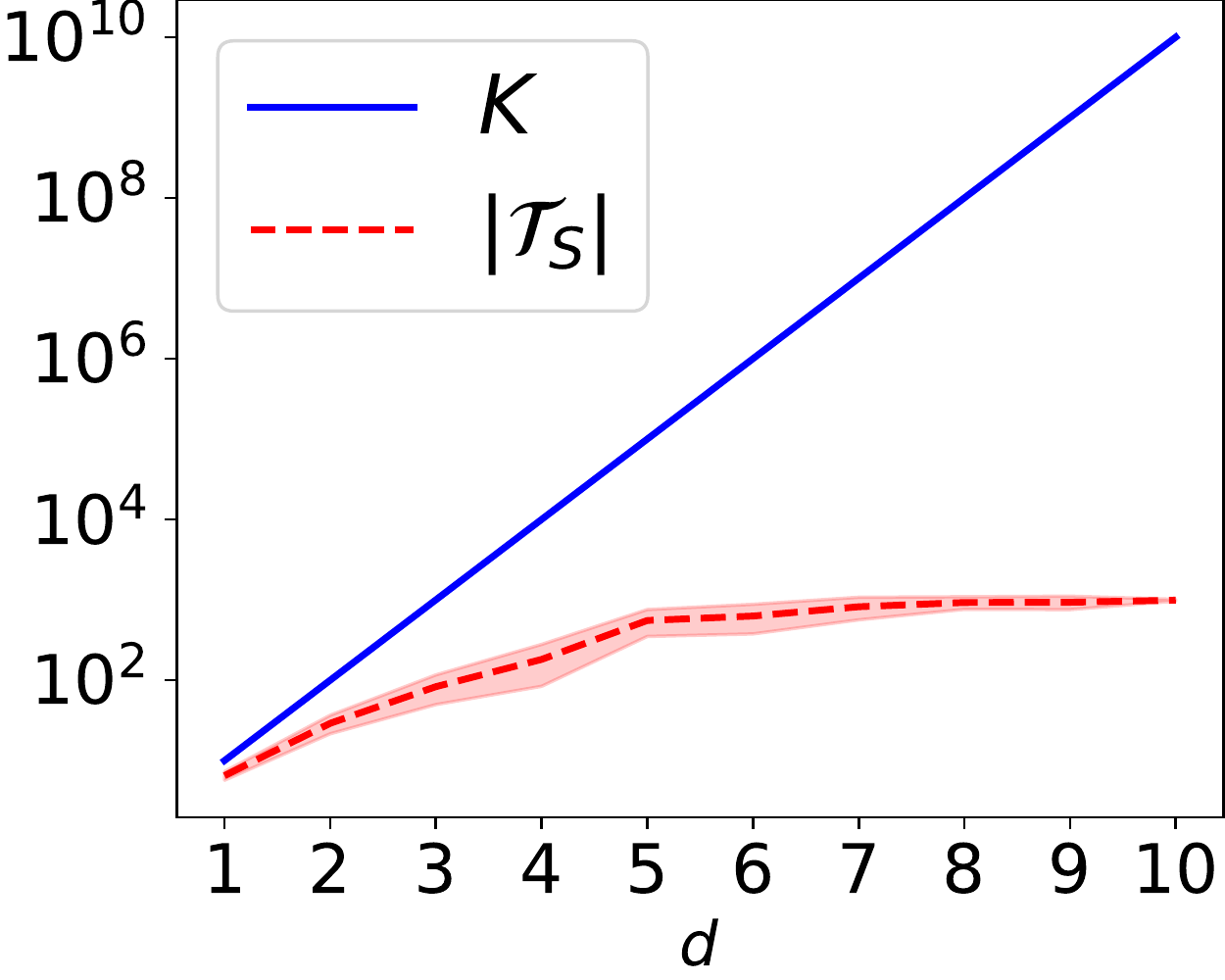}
  \vspace{-15pt}
  \caption{beta-Gauss (0.1,10)-(0.1)} 
  \vspace{10pt}
\end{subfigure}
\begin{subfigure}[b]{0.24\textwidth}
  \includegraphics[width=\textwidth, height=0.8\textwidth]{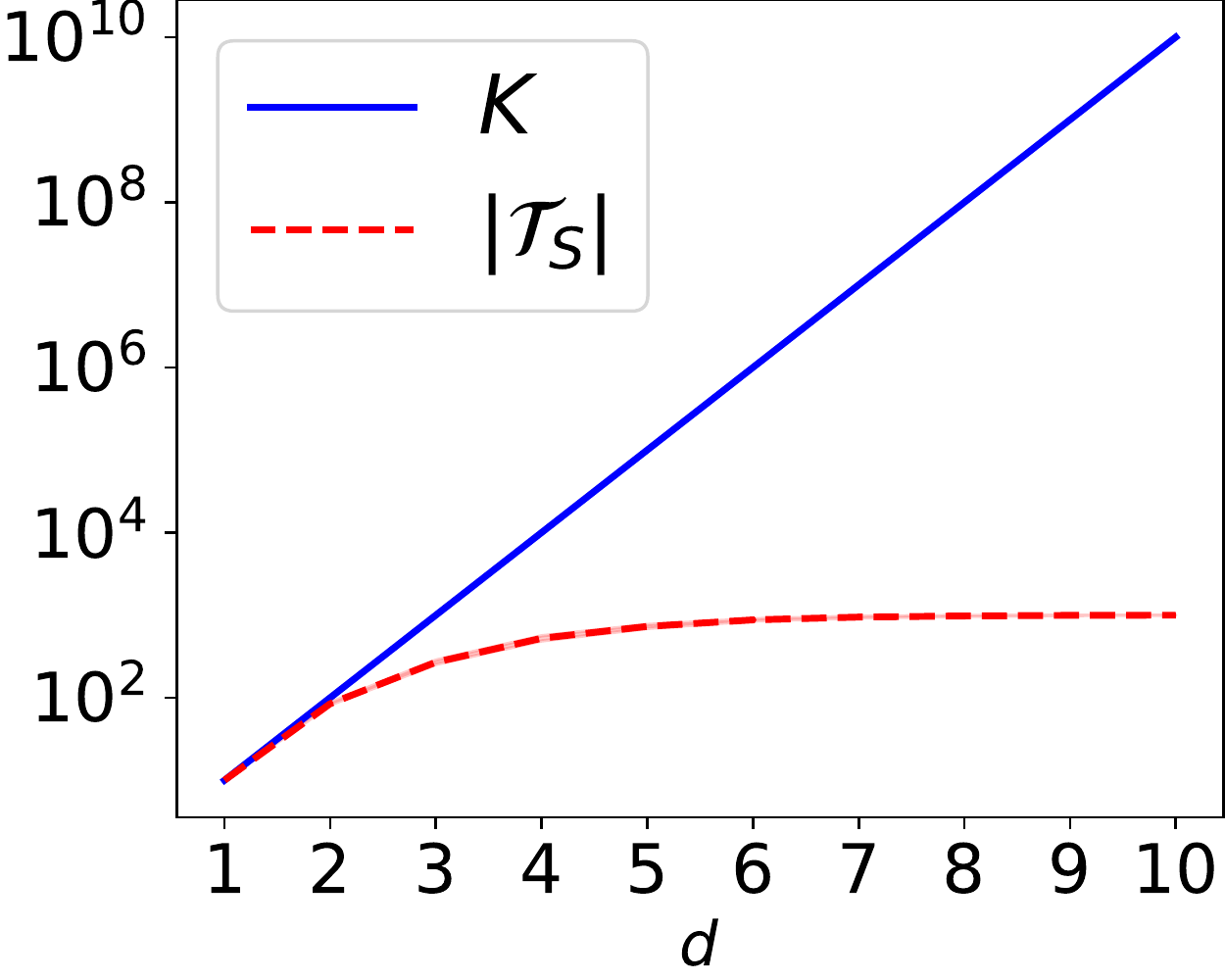}
  \vspace{-15pt}
  \caption{beta-Gauss (0.1,10)-(10)} 
  \vspace{10pt}
\end{subfigure}
\begin{subfigure}[b]{0.24\textwidth}
  \includegraphics[width=\textwidth, height=0.8\textwidth]{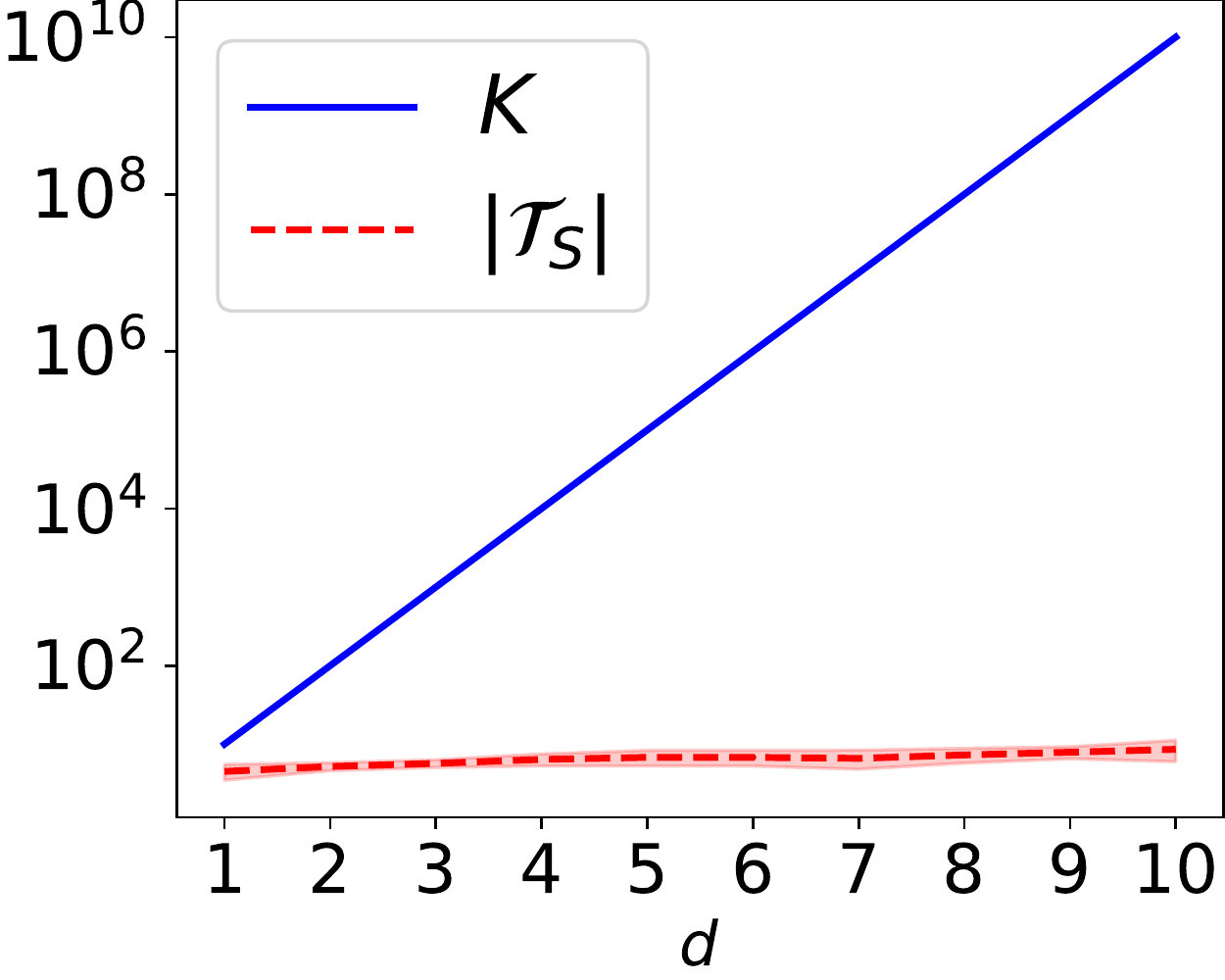}
  \vspace{-15pt}
  \caption{Gauss mix (0.001)} 
  \vspace{10pt}
\end{subfigure}
\begin{subfigure}[b]{0.24\textwidth}
  \includegraphics[width=\textwidth, height=0.8\textwidth]{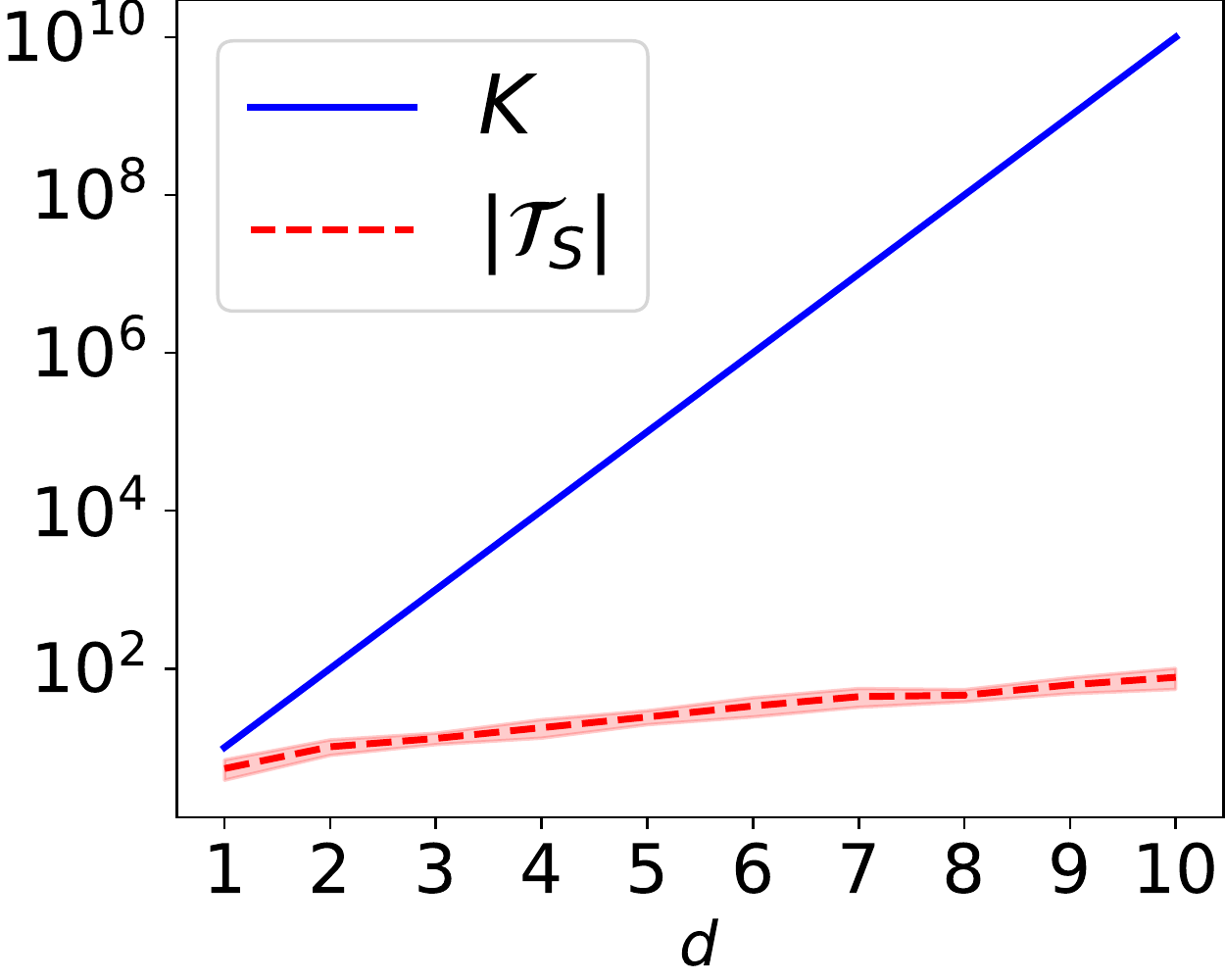}
  \vspace{-15pt}
  \caption{Gauss mix (0.01)} 
  \vspace{10pt}
\end{subfigure}
\begin{subfigure}[b]{0.24\textwidth}
  \includegraphics[width=\textwidth, height=0.8\textwidth]{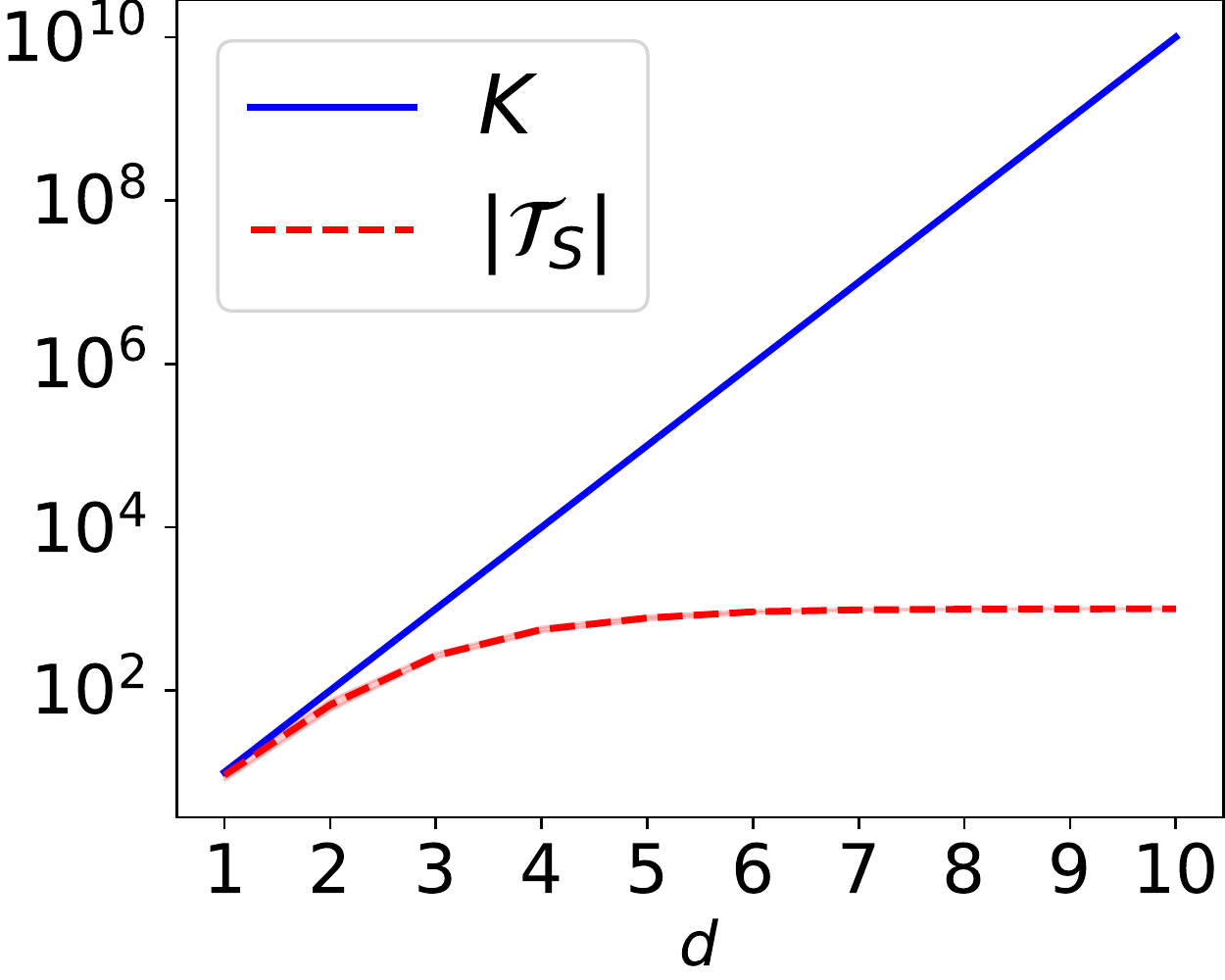}
  \vspace{-15pt}
  \caption{Gauss mix (0.1)} 
  \vspace{10pt}
\end{subfigure}
\begin{subfigure}[b]{0.24\textwidth}
  \includegraphics[width=\textwidth, height=0.8\textwidth]{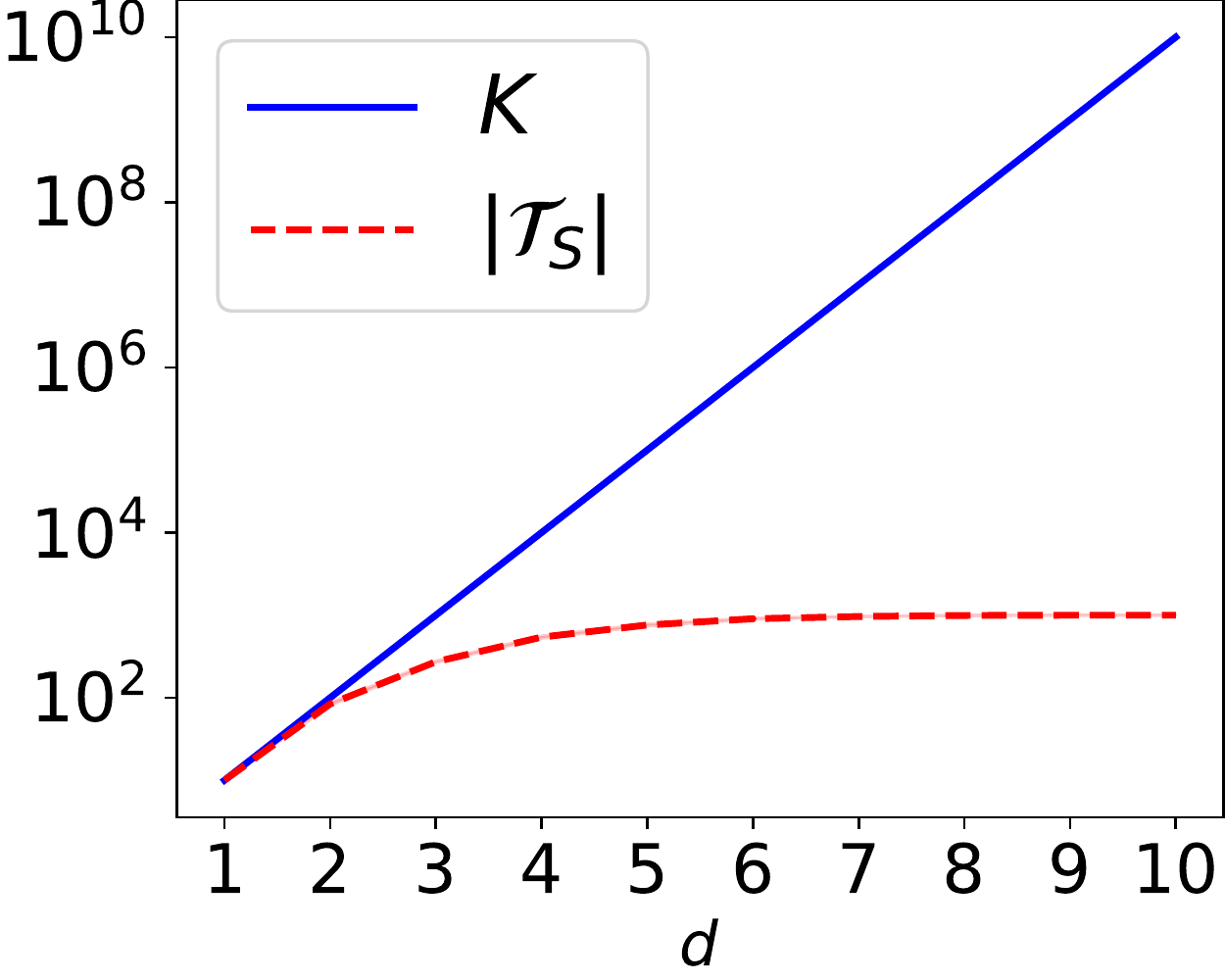}
  \vspace{-15pt}
  \caption{Gauss mix (1.0)} 
  \vspace{10pt}
\end{subfigure}
\begin{subfigure}[b]{0.24\textwidth}
  \includegraphics[width=\textwidth, height=0.8\textwidth]{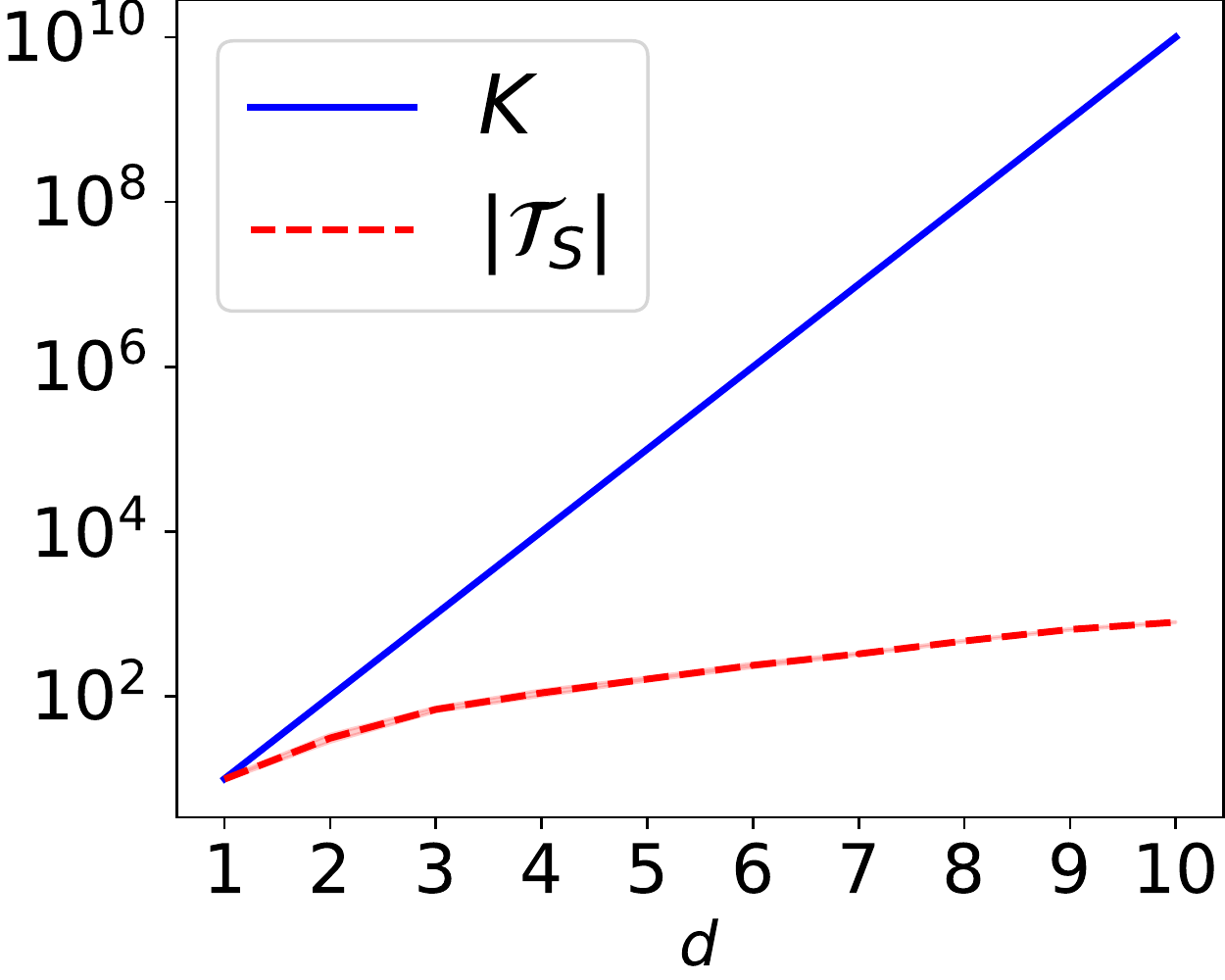}
  \vspace{-15pt}
  \caption{Gauss mix (10.0)} 
  \vspace{10pt}
\end{subfigure}
\caption{The values of $K$ versus $|\Tcal_S|$ with the $\epsilon$-covering of the original space. These figures show the mean of 10 random trials and one standard deviation.  They are plotted on a \emph{logarithmic} scale to show the discrepancy in the rates of growth.}  \label{fig:app:2}
\end{figure}

\begin{figure}[!t]
\center
\begin{subfigure}[b]{0.24\textwidth}
  \includegraphics[width=\textwidth, height=0.8\textwidth]{fig/results_2/log/projection_beta_1000_001_01_01_01_0.pdf}
  \vspace{-15pt}
  \caption{Beta(0.1, 0.1)} 
  \vspace{10pt}
\end{subfigure}
\begin{subfigure}[b]{0.24\textwidth}
  \includegraphics[width=\textwidth, height=0.8\textwidth]{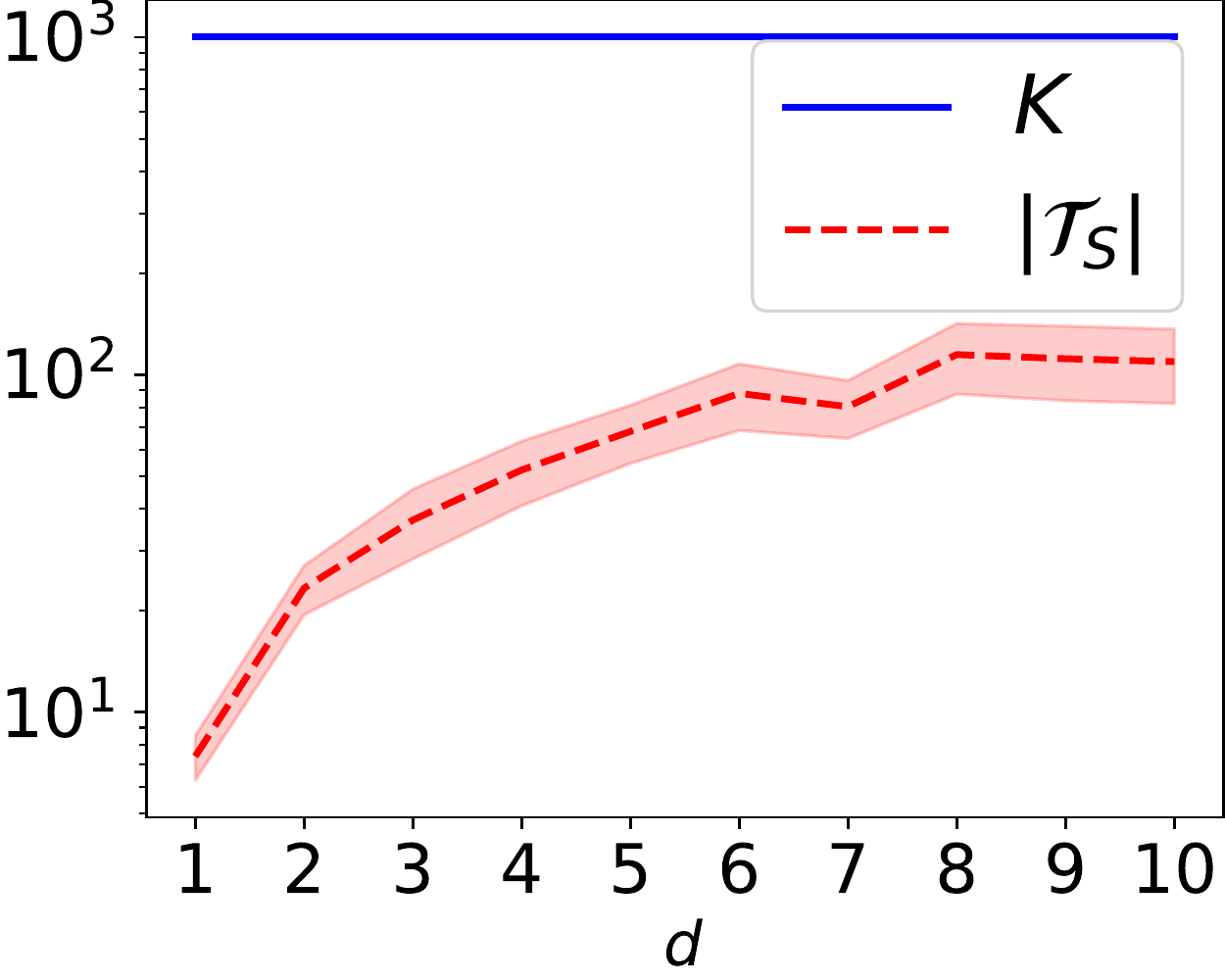}
  \vspace{-15pt}
  \caption{Beta(0.01, 0.01)} 
  \vspace{10pt}
\end{subfigure}
\begin{subfigure}[b]{0.24\textwidth}
  \includegraphics[width=\textwidth, height=0.8\textwidth]{fig/results_2/log/projection_beta_1000_001_01_100_01_0.pdf}
  \vspace{-15pt}
  \caption{Beta(0.1, 10)} 
  \vspace{10pt}
\end{subfigure}
\begin{subfigure}[b]{0.24\textwidth}
  \includegraphics[width=\textwidth, height=0.8\textwidth]{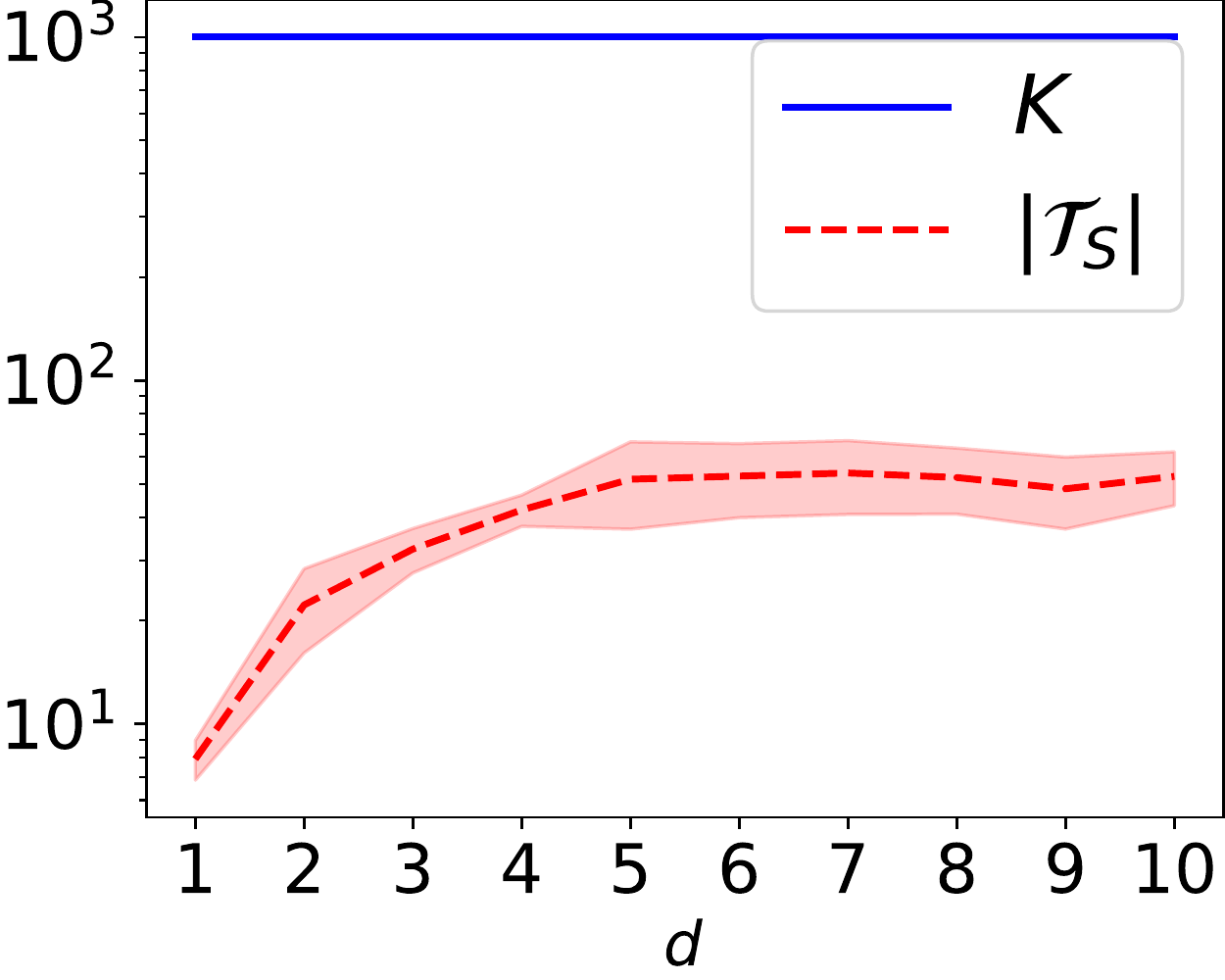}
  \vspace{-15pt}
  \caption{Beta mix (0.1, 0.1)-(0.01)} 
  \vspace{10pt}
\end{subfigure}
\begin{subfigure}[b]{0.24\textwidth}
  \includegraphics[width=\textwidth, height=0.8\textwidth]{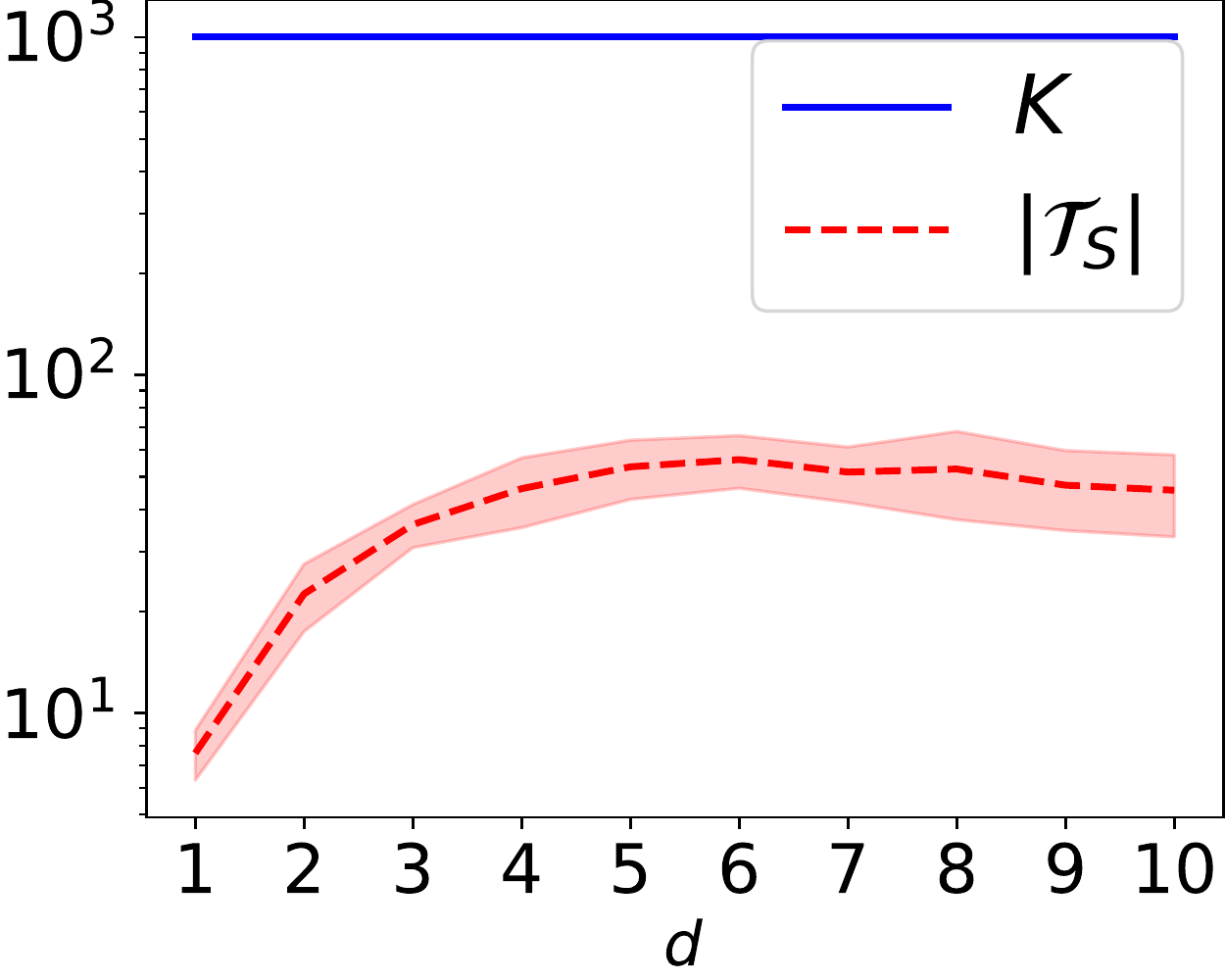}
  \vspace{-15pt}
  \caption{Beta mix (0.1, 0.1)-(0.1)} 
  \vspace{10pt}
\end{subfigure}
\begin{subfigure}[b]{0.24\textwidth}
  \includegraphics[width=\textwidth, height=0.8\textwidth]{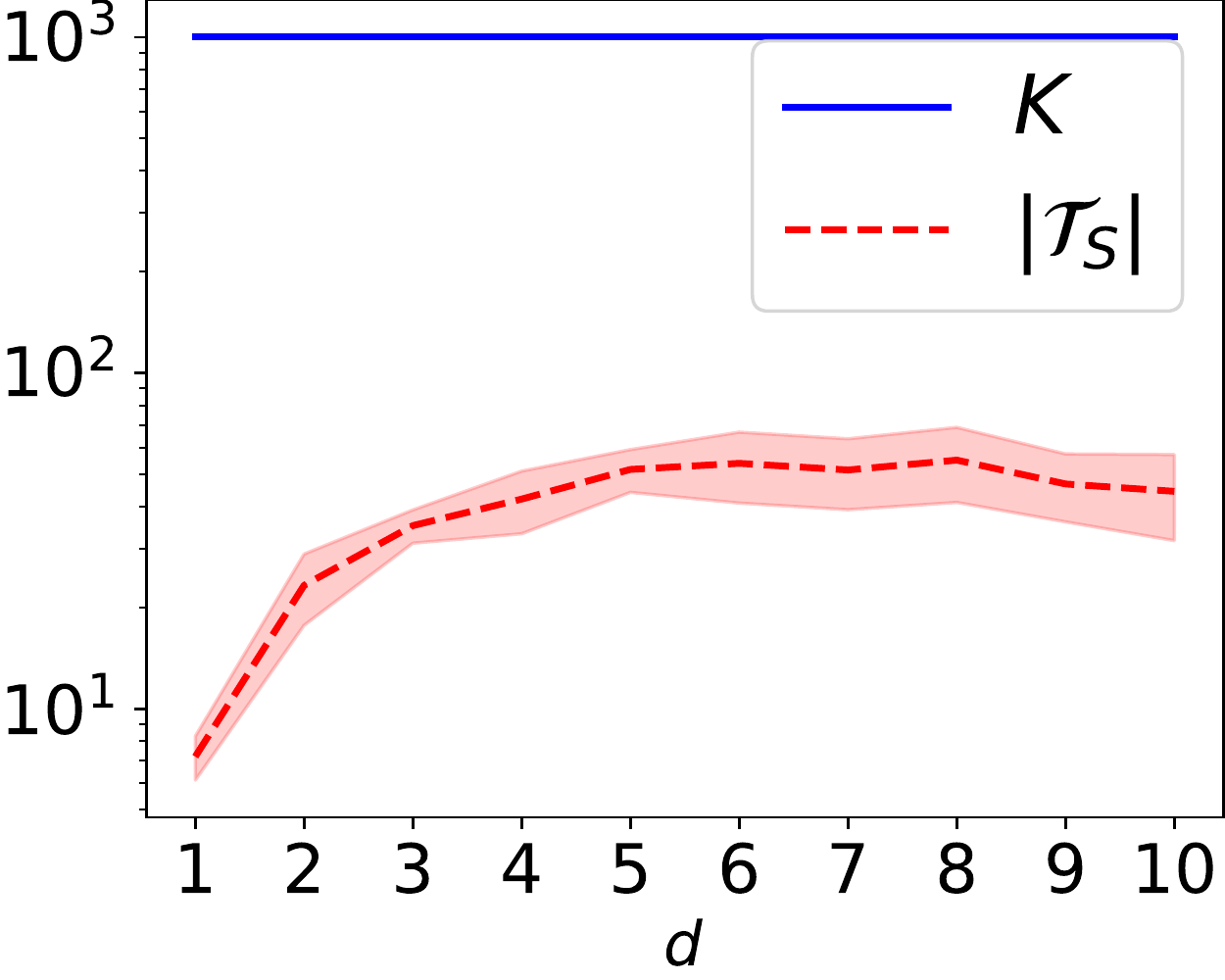}
  \vspace{-15pt}
  \caption{Beta mix (0.1, 0.1)-(10)} 
  \vspace{10pt}
\end{subfigure}
\begin{subfigure}[b]{0.24\textwidth}
  \includegraphics[width=\textwidth, height=0.8\textwidth]{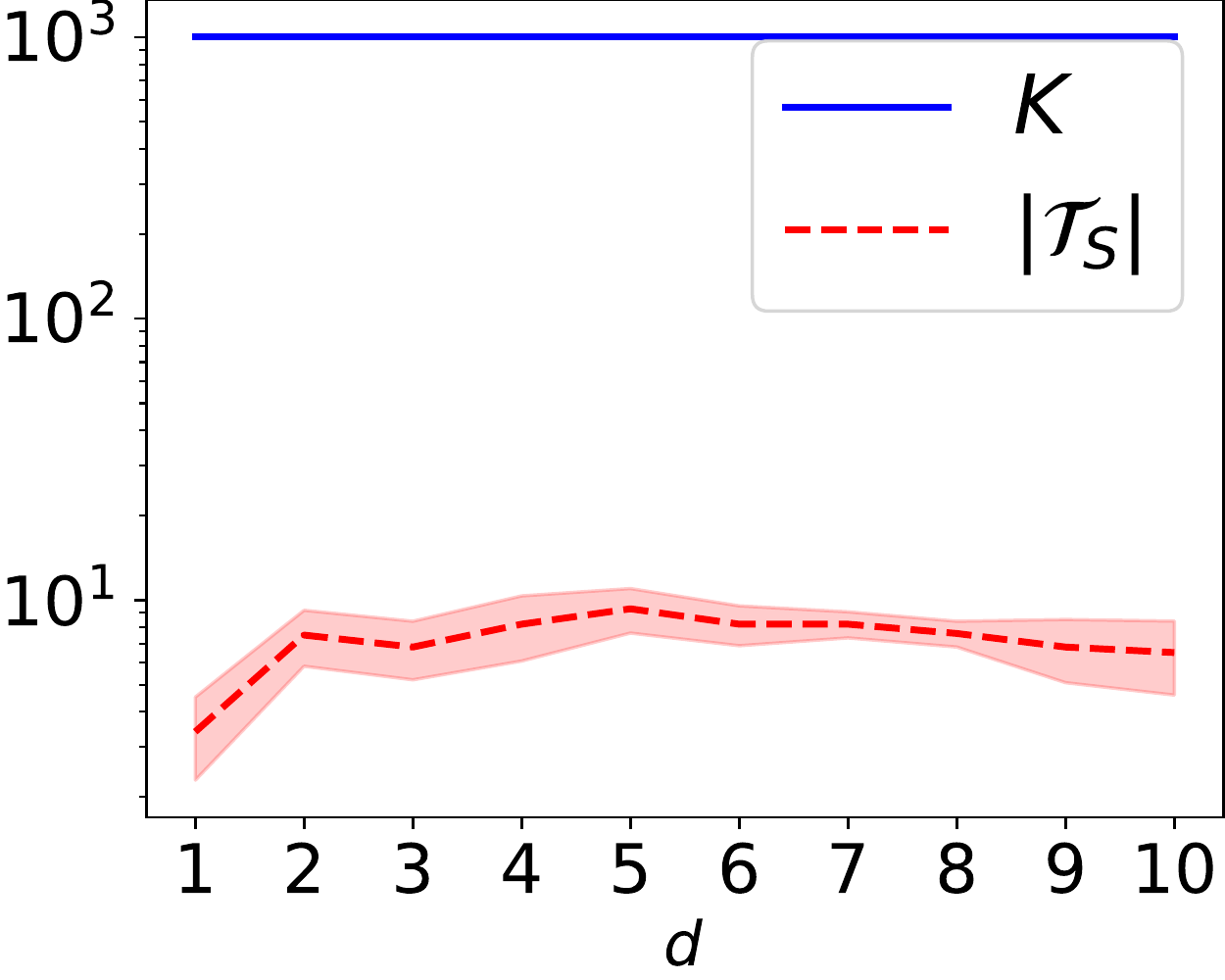}
  \vspace{-15pt}
  \caption{Beta mix (0.1, 10)-(0.01)} 
  \vspace{10pt}
\end{subfigure}
\begin{subfigure}[b]{0.24\textwidth}
  \includegraphics[width=\textwidth, height=0.8\textwidth]{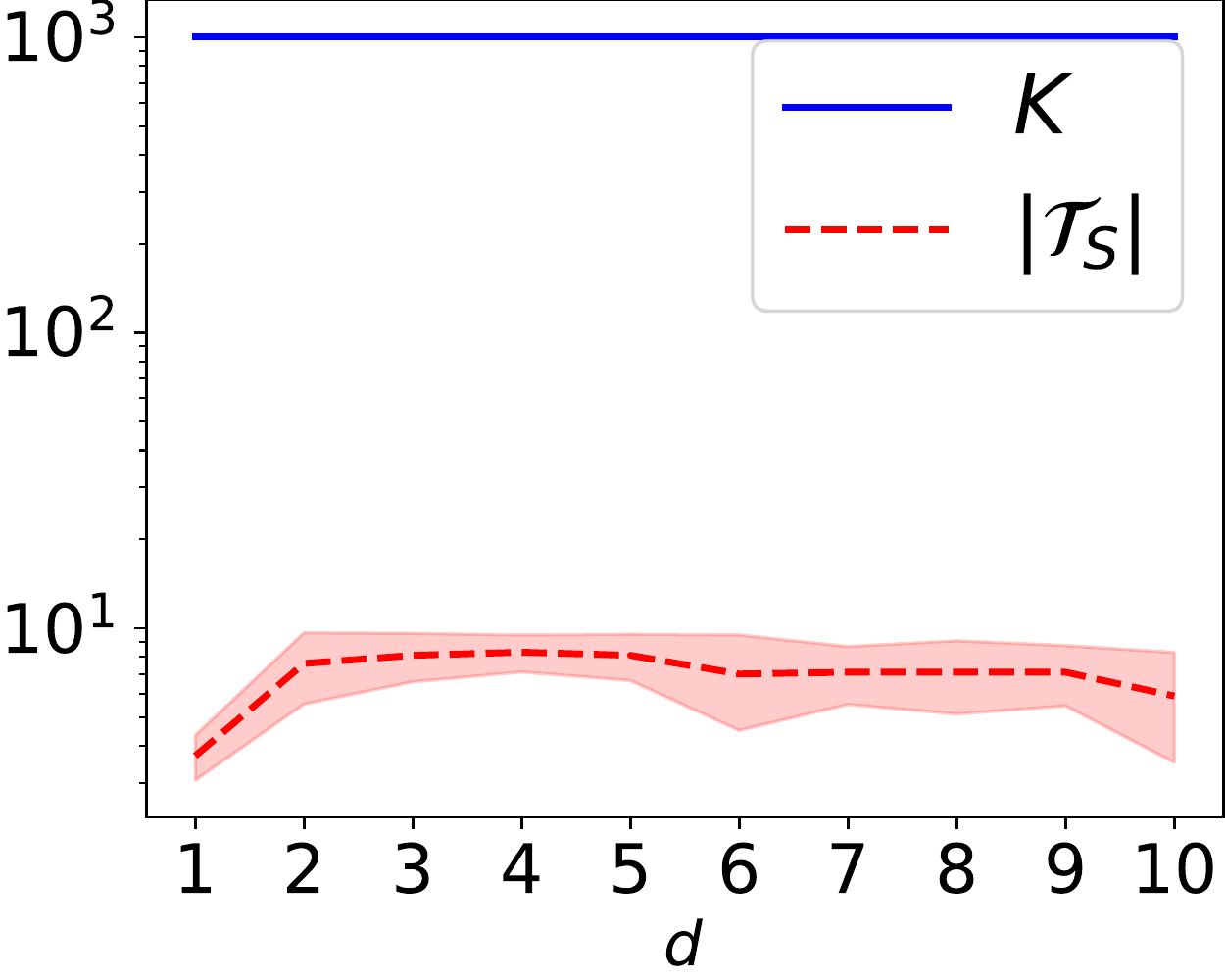}
  \vspace{-15pt}
  \caption{Beta mix (0.1, 10)-(0.1)} 
  \vspace{10pt}
\end{subfigure}
\begin{subfigure}[b]{0.24\textwidth}
  \includegraphics[width=\textwidth, height=0.8\textwidth]{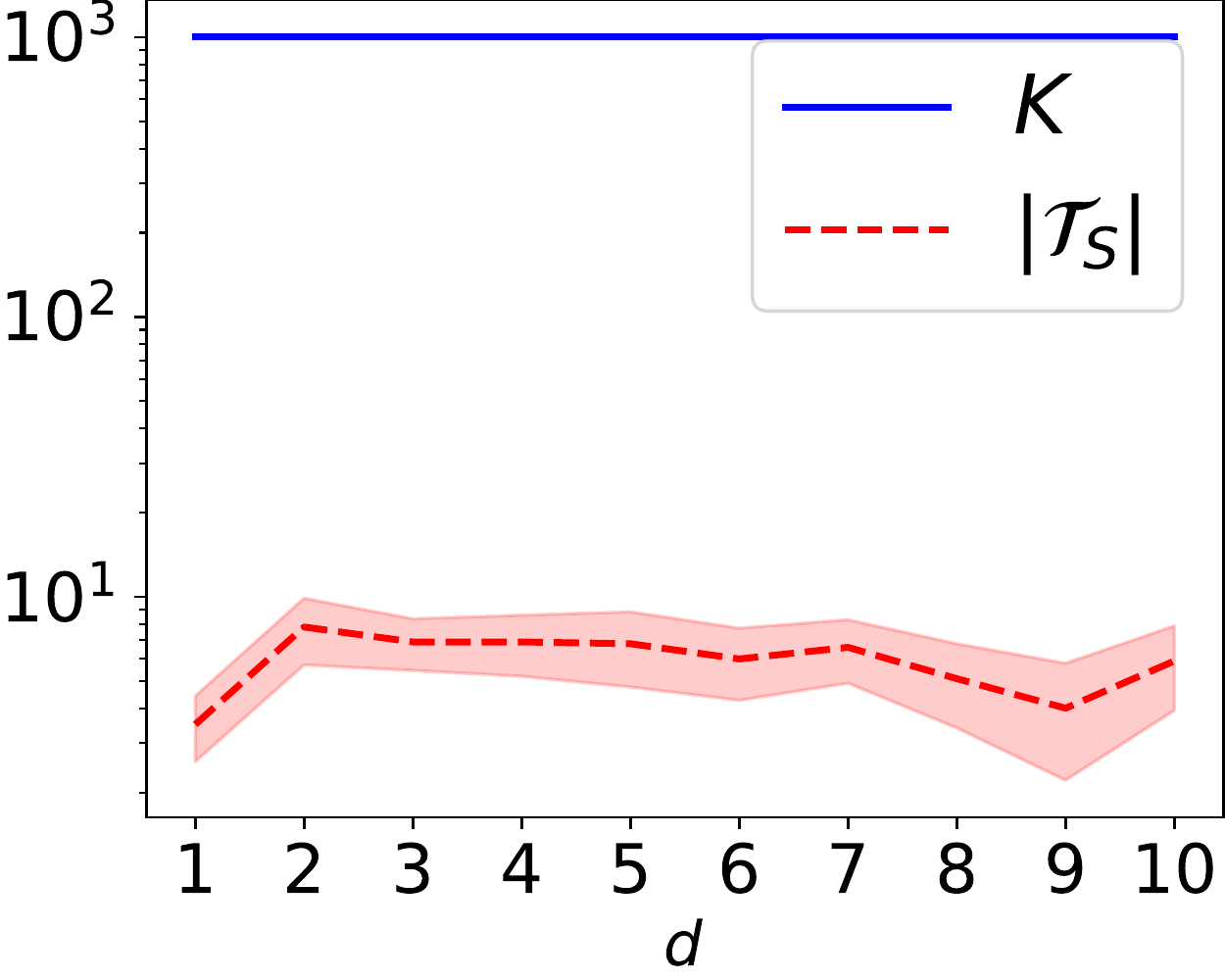}
  \vspace{-15pt}
  \caption{Beta mix (0.1, 10)-(10)} 
  \vspace{10pt}
\end{subfigure}
\begin{subfigure}[b]{0.24\textwidth}
  \includegraphics[width=\textwidth, height=0.8\textwidth]{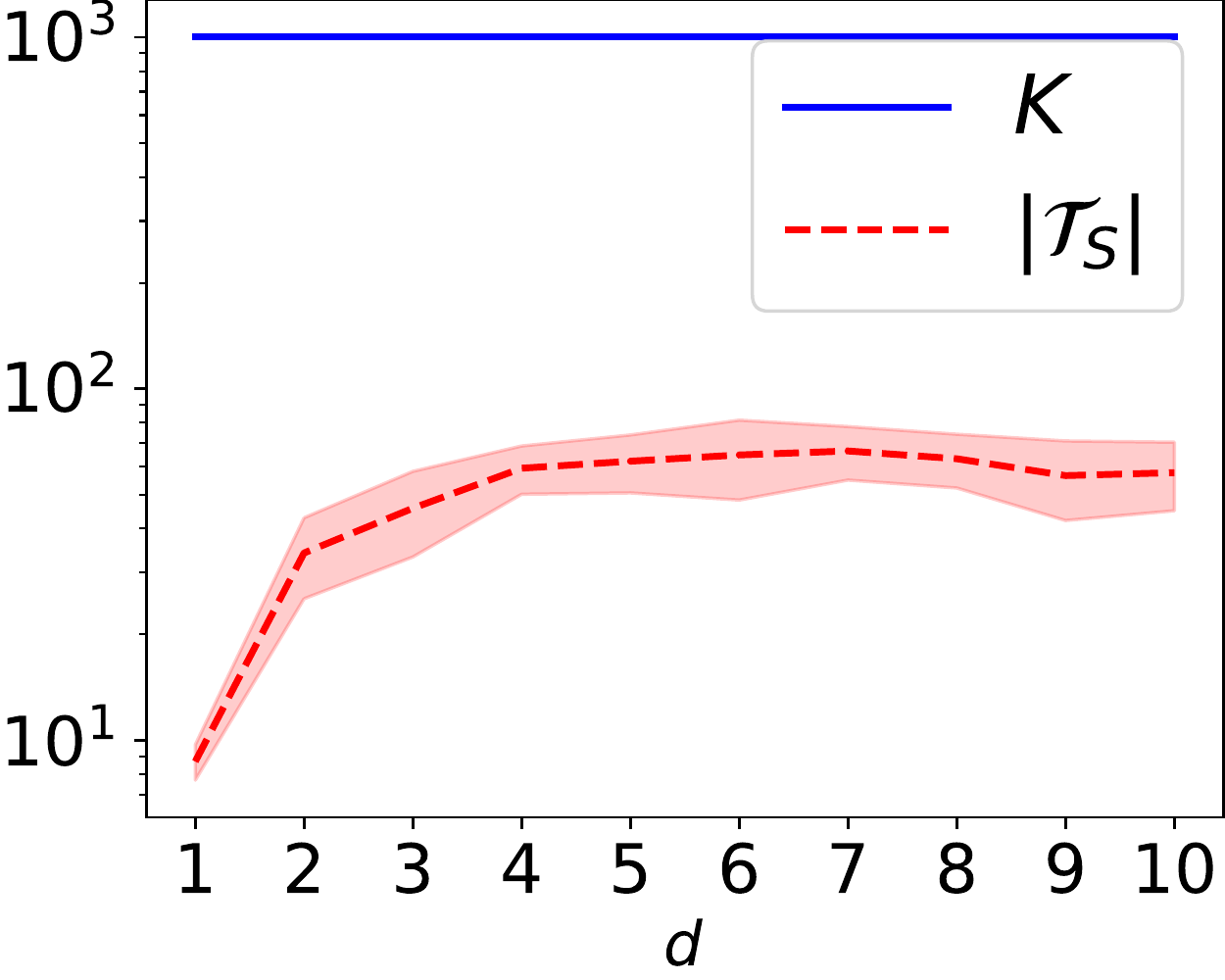}
  \vspace{-15pt}
  \caption{beta-Gauss (0.1,0.1)-(0.01)} 
  \vspace{10pt}
\end{subfigure}
\begin{subfigure}[b]{0.24\textwidth}
  \includegraphics[width=\textwidth, height=0.8\textwidth]{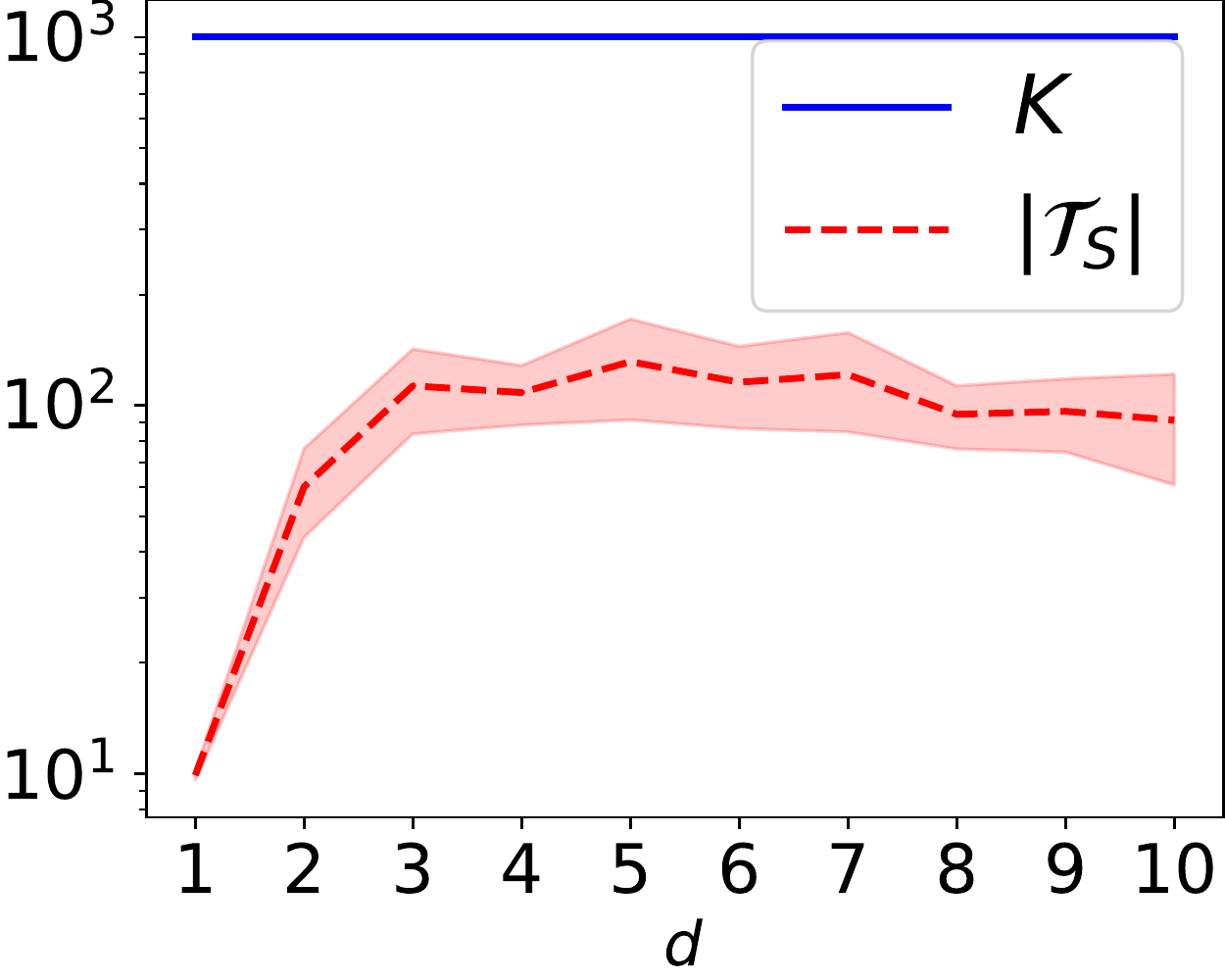}
  \vspace{-15pt}
  \caption{beta-Gauss (0.1,0.1)-(0.1)} 
  \vspace{10pt}
\end{subfigure}
\begin{subfigure}[b]{0.24\textwidth}
  \includegraphics[width=\textwidth, height=0.8\textwidth]{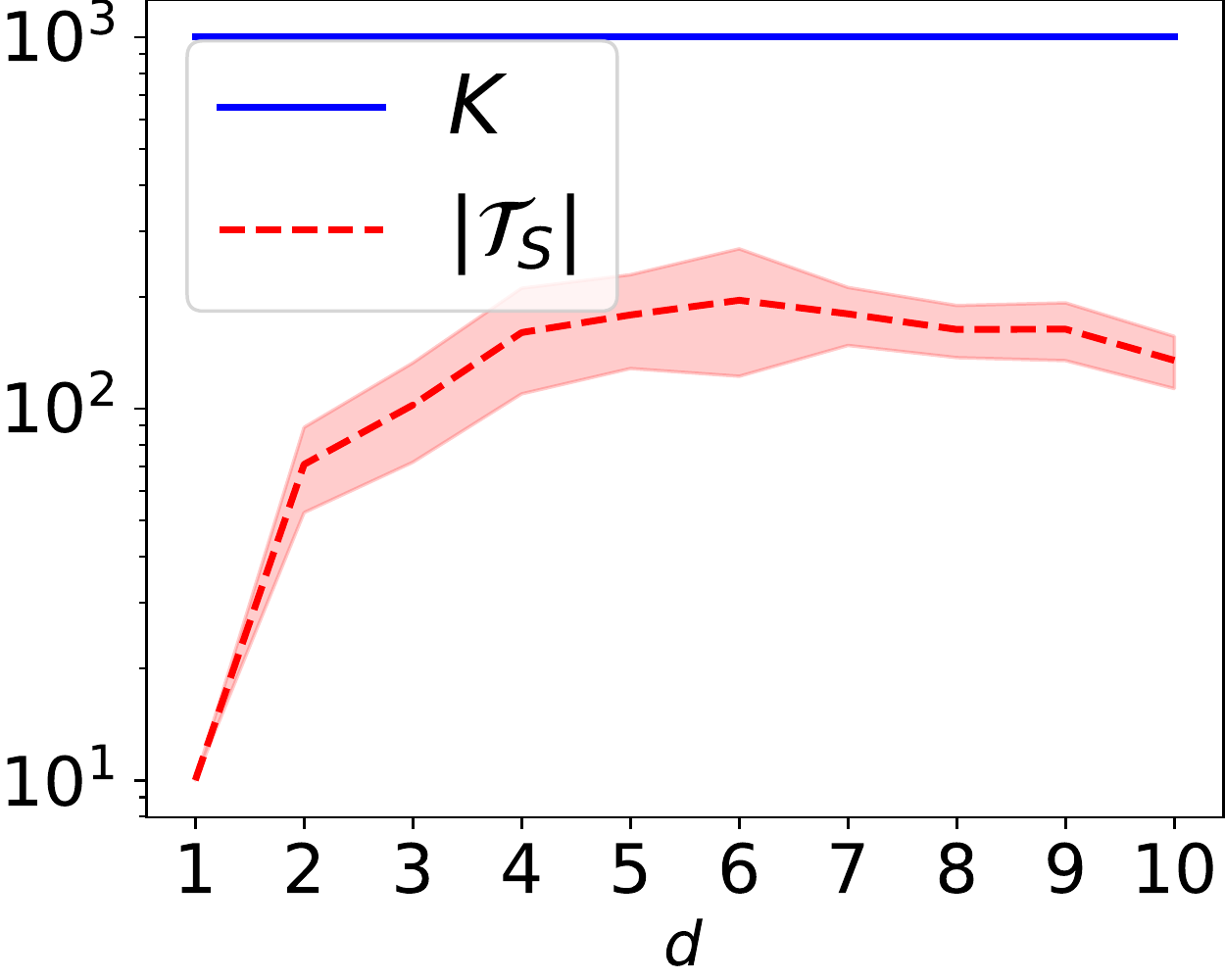}
  \vspace{-15pt}
  \caption{beta-Gauss (0.1,0.1)-(10)} 
  \vspace{10pt}
\end{subfigure}
\begin{subfigure}[b]{0.24\textwidth}
  \includegraphics[width=\textwidth, height=0.8\textwidth]{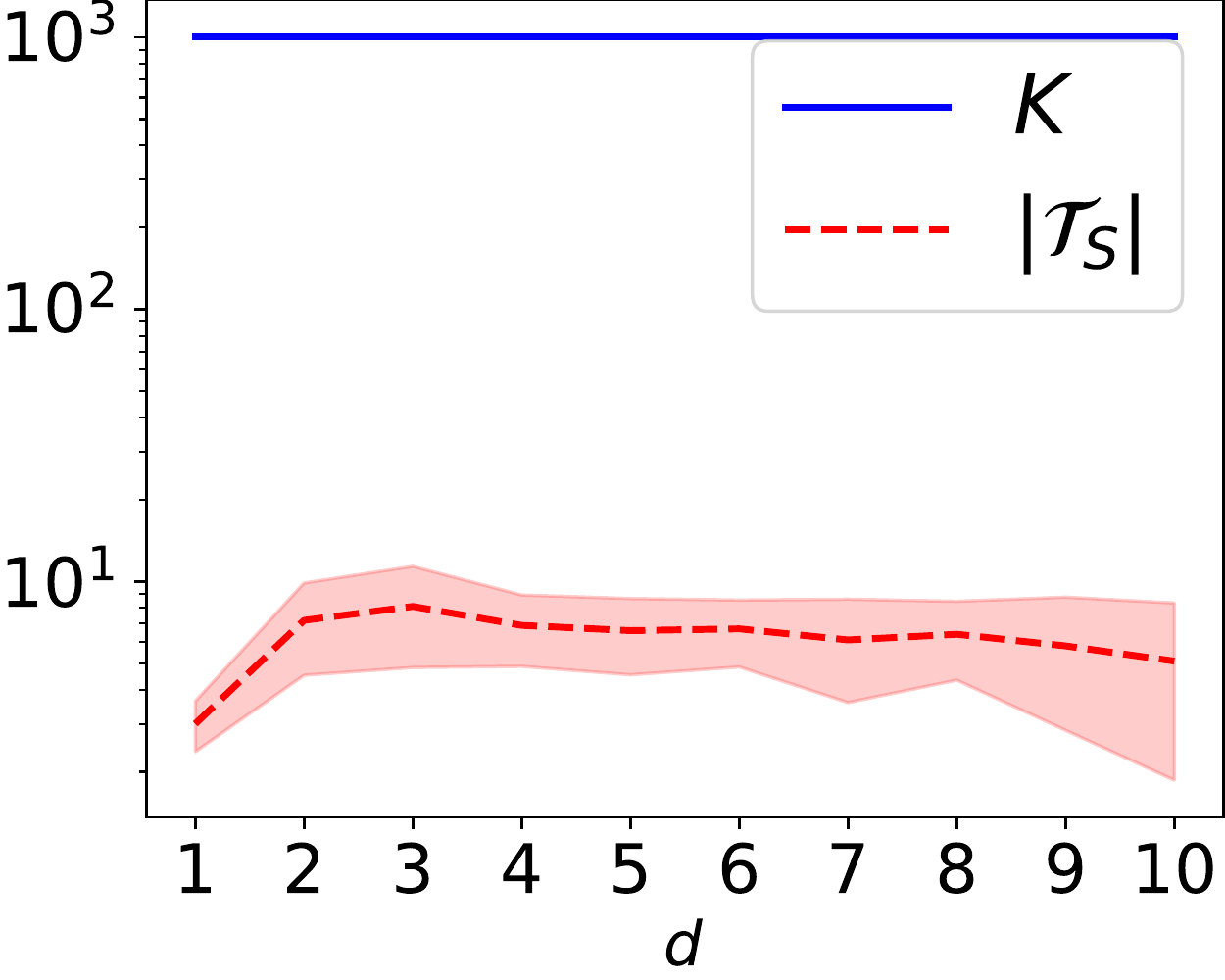}
  \vspace{-15pt}
  \caption{beta-Gauss (0.1,10)-(0.01)} 
  \vspace{10pt}
\end{subfigure}
\begin{subfigure}[b]{0.24\textwidth}
  \includegraphics[width=\textwidth, height=0.8\textwidth]{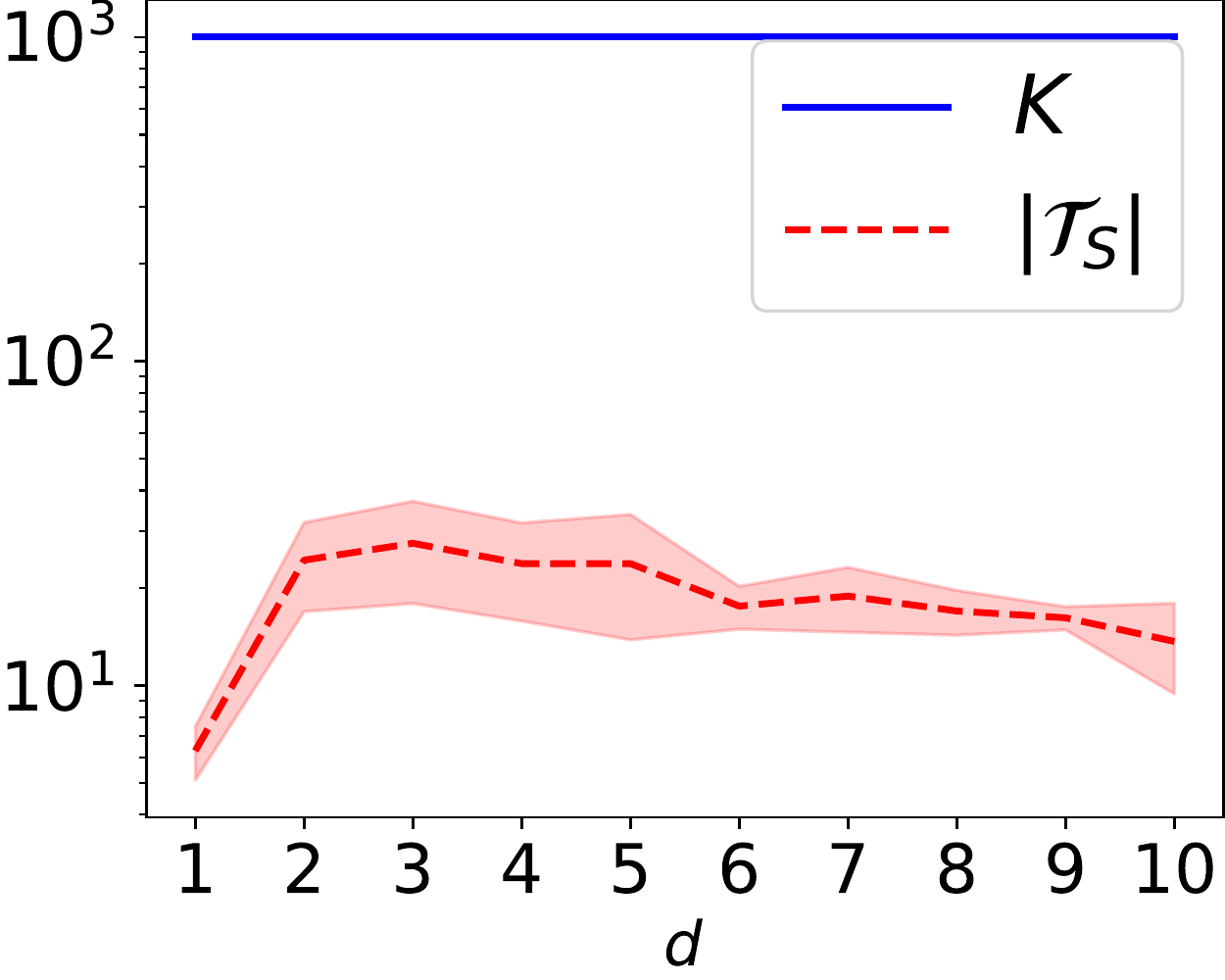}
  \vspace{-15pt}
  \caption{beta-Gauss (0.1,10)-(0.1)} 
  \vspace{10pt}
\end{subfigure}
\begin{subfigure}[b]{0.24\textwidth}
  \includegraphics[width=\textwidth, height=0.8\textwidth]{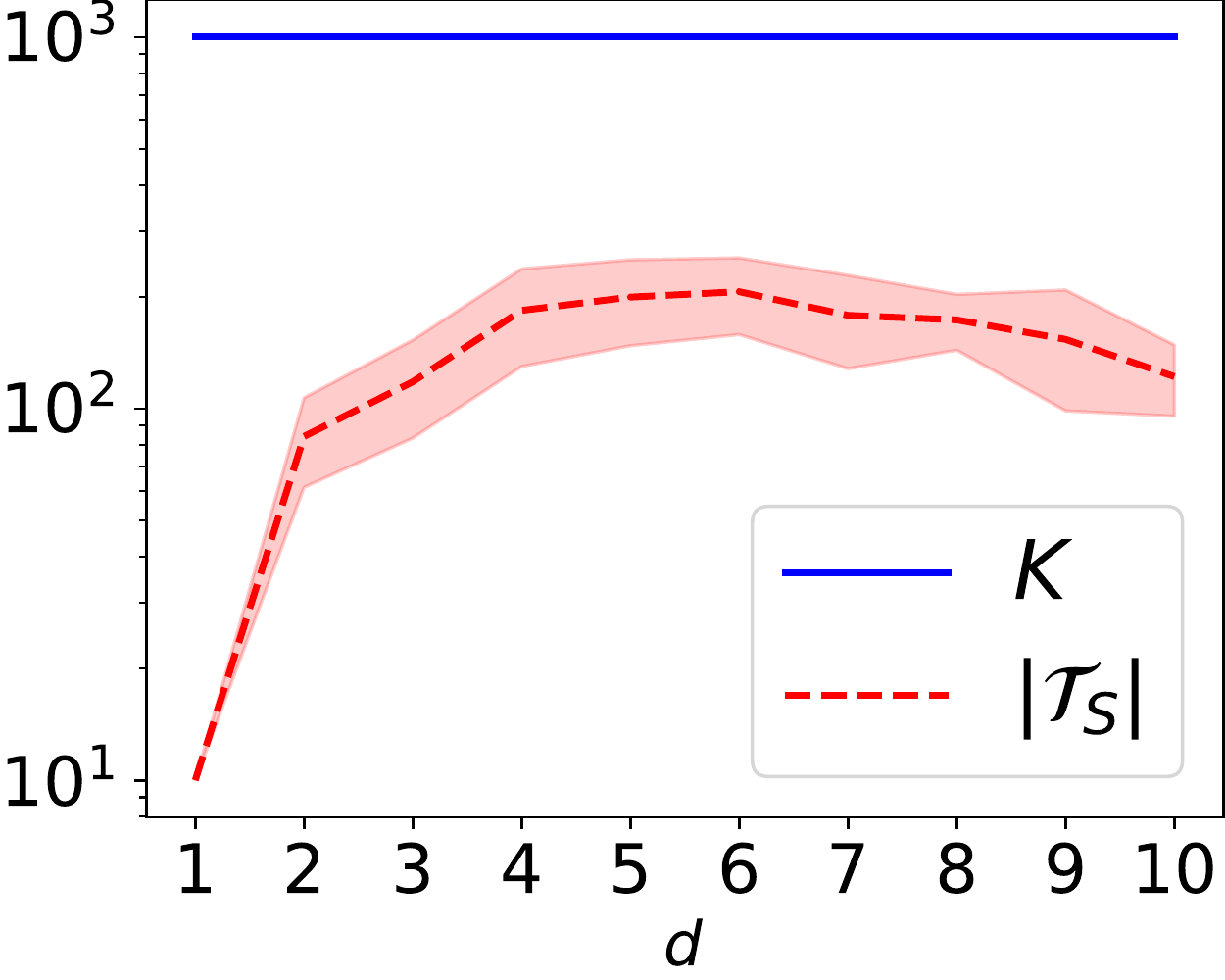}
  \vspace{-15pt}
  \caption{beta-Gauss (0.1,10)-(10)} 
  \vspace{10pt}
\end{subfigure}
\begin{subfigure}[b]{0.24\textwidth}
  \includegraphics[width=\textwidth, height=0.8\textwidth]{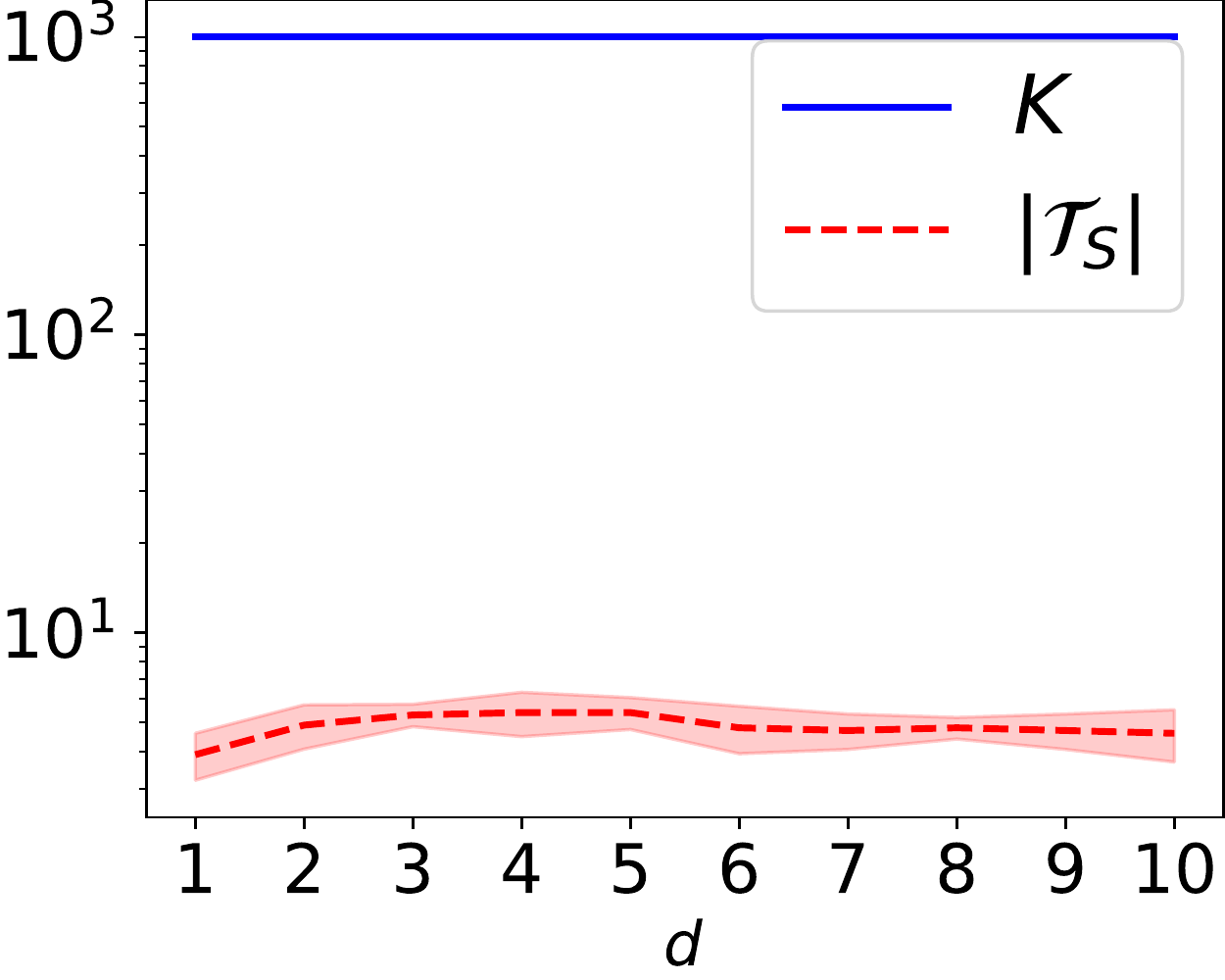}
  \vspace{-15pt}
  \caption{Gauss mix (0.001)} 
  \vspace{10pt}
\end{subfigure}
\begin{subfigure}[b]{0.24\textwidth}
  \includegraphics[width=\textwidth, height=0.8\textwidth]{fig/results_2/log/projection_blobs_1000_001_01_100_01_0.pdf}
  \vspace{-15pt}
  \caption{Gauss mix (0.01)} 
  \vspace{10pt}
\end{subfigure}
\begin{subfigure}[b]{0.24\textwidth}
  \includegraphics[width=\textwidth, height=0.8\textwidth]{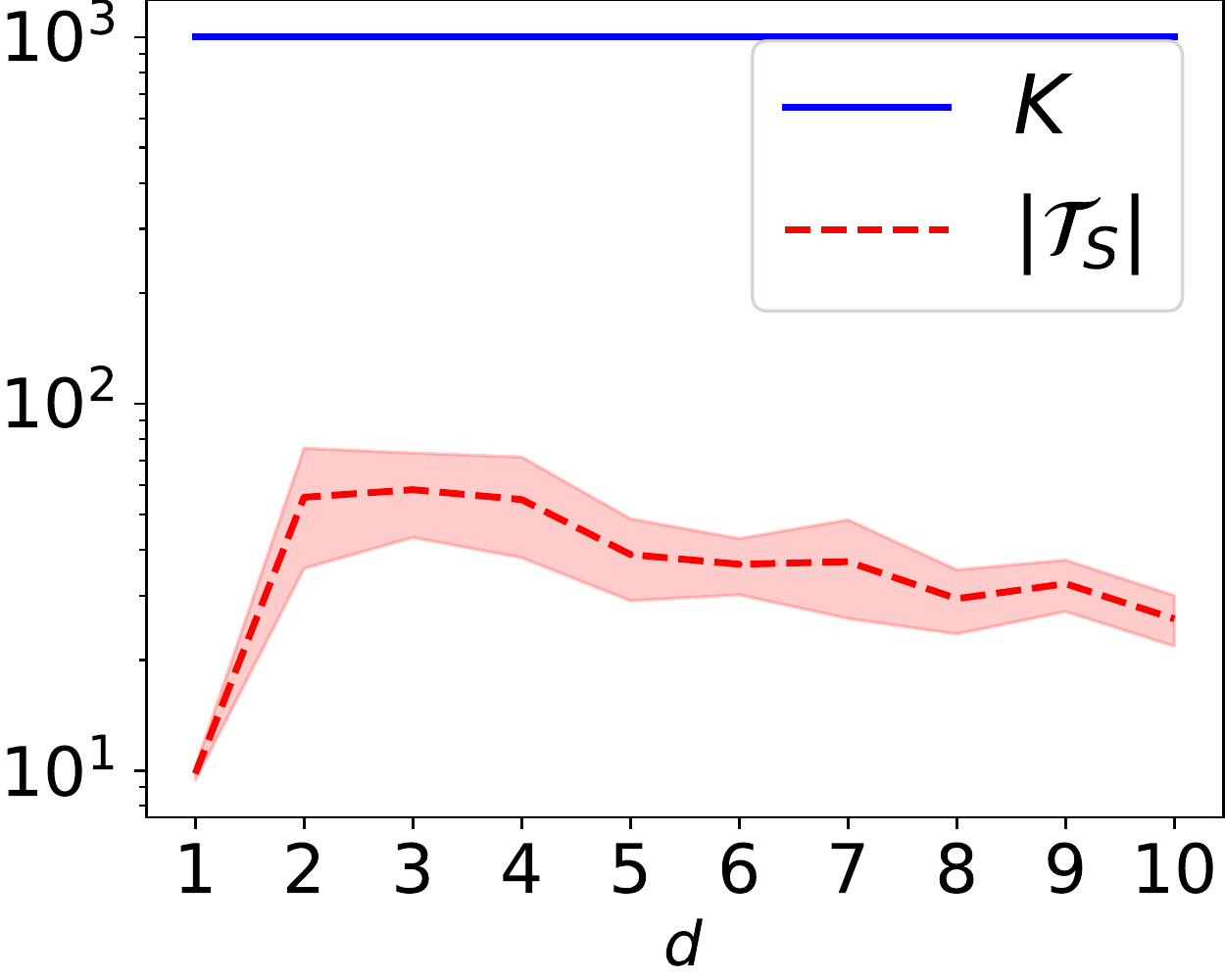}
  \vspace{-15pt}
  \caption{Gauss mix (0.1)} 
  \vspace{10pt}
\end{subfigure}
\begin{subfigure}[b]{0.24\textwidth}
  \includegraphics[width=\textwidth, height=0.8\textwidth]{fig/results_2/log/projection_blobs_1000_10_01_100_01_0.pdf}
  \vspace{-15pt}
  \caption{Gauss mix (1.0)} 
  \vspace{10pt}
\end{subfigure}
\begin{subfigure}[b]{0.24\textwidth}
  \includegraphics[width=\textwidth, height=0.8\textwidth]{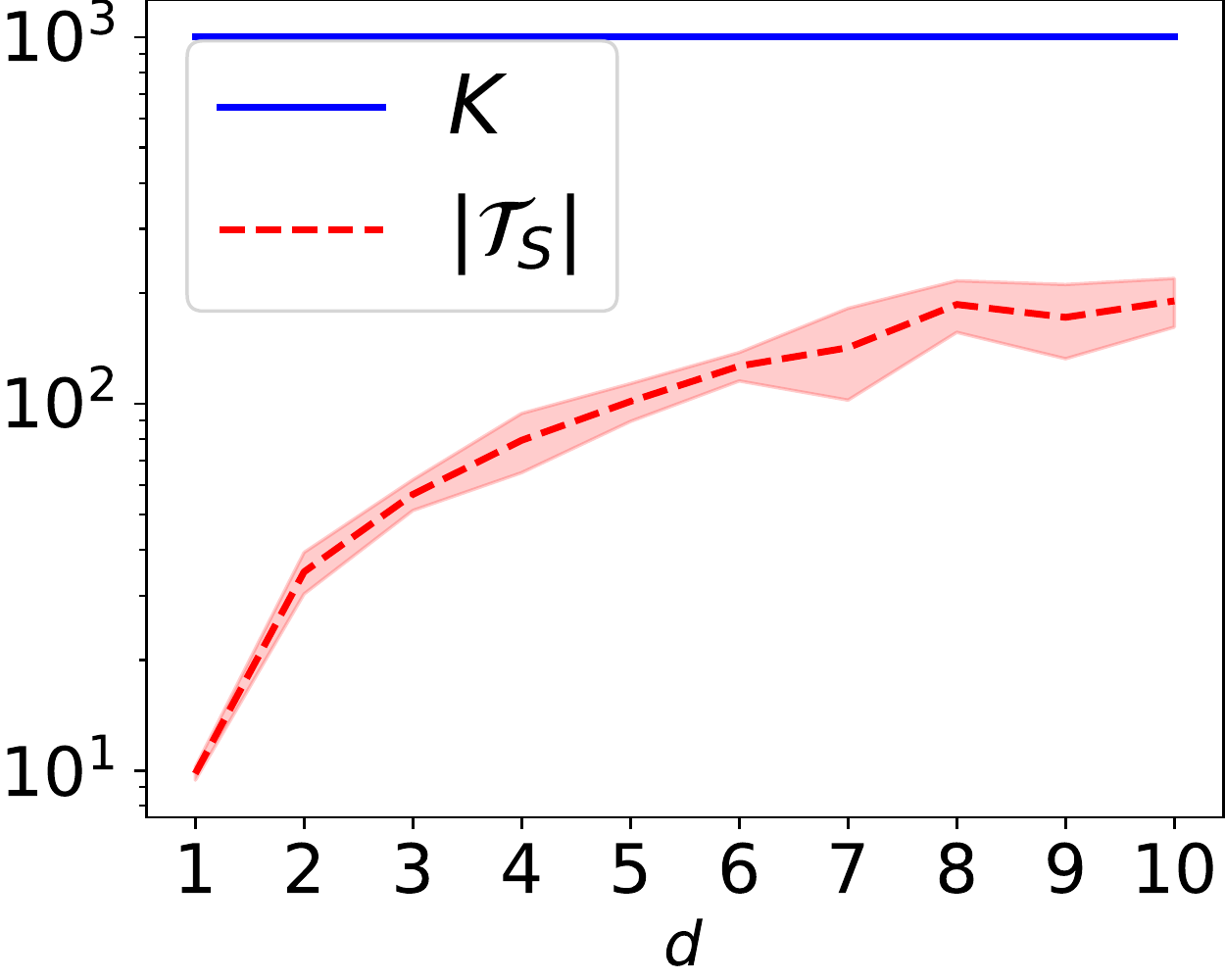}
  \vspace{-15pt}
  \caption{Gauss mix (10.0)} 
  \vspace{10pt}
\end{subfigure}
\caption{The values of $K$ versus $|\Tcal_S|$ with the inverse image of the $\epsilon$-covering in randomly projected spaces. These figures show the mean of 10 random trials and one standard deviation.  Here, the input space was first randomly projected onto a space of dimension 3.  This data is plotted on a logarithmic scale.} \label{fig:app:3}
\end{figure}    
          
\end{document}